\documentclass{article}
\PassOptionsToPackage{numbers,sort&compress}{natbib}

\usepackage{neurips_20261}

\makeatletter
\@preprinttrue
\nolinenumbers
\makeatother

\usepackage{amsmath, amssymb, amsthm, mathtools}
\usepackage{bm}
\usepackage{float}
\usepackage{afterpage}

\usepackage{cancel}                 
\usepackage{placeins} 
\usepackage{graphicx}
\usepackage{multirow}
\usepackage{array}
\usepackage{xcolor}
\usepackage{booktabs}

\AtBeginDocument{
\hypersetup{
    colorlinks=true,
    linkcolor=red,
    citecolor=blue,
    urlcolor=blue
}}

\usepackage{enumitem}
\usepackage{pifont}
\newcommand{\cmark}{\ding{51}}
\newcommand{\xmark}{\ding{55}}

\usepackage{fancyhdr}
\newtheorem{theorem}{Theorem}
\newtheorem{lemma}{Lemma}
\newtheorem{proposition}{Proposition}
\newtheorem{corollary}{Corollary}[section]

\newtheorem{definition}{Definition}
\newtheorem{assumption}{Assumption}

\newtheorem{remark}{Remark}

\usepackage[colorlinks=true, linkcolor=blue!70!black,
            citecolor=blue!70!black, urlcolor=blue!70!black]{hyperref}

\usepackage{algorithm}
\usepackage{algpseudocode}

\usepackage{tcolorbox}
\tcbuselibrary{most}

\newtcolorbox{keybox}{
  colback=blue!4!white, colframe=blue!55!black,
  boxrule=0.6pt, arc=3pt,
  left=6pt, right=6pt, top=4pt, bottom=4pt
}

\definecolor{boxbg}{RGB}{245, 245, 245}      
\definecolor{boxborder}{RGB}{80, 80, 80}     
\newtcbtheorem[number within=section]{boxtheorem}{Theorem}{
    enhanced, breakable,
    colback=boxbg, colframe=boxborder,
    fonttitle=\bfseries,
    separator sign none, description delimiters none,
    attach boxed title to top left={yshift=-2mm, xshift=4mm},
    boxed title style={colback=boxborder, colframe=boxborder,
                       sharp corners, size=small}
}{thm}

\usepackage{titlesec}
\titleformat{\section}{\large\bfseries}{\thesection.}{0.5em}{}
\titleformat{\subsection}{\normalsize\bfseries}{\thesubsection.}{0.5em}{}

\usepackage{color}

\DeclareMathOperator{\softmax}{softmax}

\newcommand{\R}{\mathbb{R}}
\newcommand{\E}{\mathbb{E}}

\newcommand{\Kop}{\mathcal{K}_{\theta}}     
\newcommand{\kap}{\kappa_{\theta}}           
\newcommand{\Id}{\mathbf{I}_d}              
\newcommand{\Idn}{\mathbf{I}_n}             
\newcommand{\T}{\mathbb{T}}                 
\newcommand{\Z}{\mathbb{Z}}                 
\newcommand{\norm}[1]{\left\|#1\right\|}
\newcommand{\abs}[1]{\left|#1\right|}

\newcommand{\had}{\odot}                    
\newcommand{\diag}{\mathrm{diag}}

\newcommand{\sigmoid}{\sigma}
\newcommand{\Attn}{\mathrm{Attn}}
\newcommand{\MHA}{\mathrm{MHA}}
\newcommand{\RNN}{\mathrm{RNN}}
\newcommand{\LSTM}{\mathrm{LSTM}}
\newcommand{\GRU}{\mathrm{GRU}}
\newcommand{\SSM}{\mathrm{SSM}}
\newcommand{\Mamba}{\mathrm{Mamba}}

\title{ITNet: A Learnable Integral Transform That Subsumes Convolution, Attention, and Recurrence}

\author{%
  Ashim Dhor$^{1}$ \quad Rasel Mondal$^{1}$ \quad Pin-Yu Chen$^{2}$ \\[2pt]
  $^{1}$Indian Institute of Science Education and Research Bhopal \\[2pt]
  $^{2}$IBM Research \\[2pt]
  \texttt{ashimdhor2003@gmail.com, raselmondal61@gmail.com, pin-yu.chen@ibm.com} 
}

\makeatletter
\@submissionfalse
\makeatother
\date{}   

\begin{document}
\maketitle

\thispagestyle{fancy}
\fancyhead{}  
\fancyfoot{}
\fancyfoot[L]{\footnotesize \textcolor{black}{Preprint.}}
\renewcommand{\headrulewidth}{0pt}
\renewcommand{\footrulewidth}{0pt}

\begin{abstract}
Convolutional networks, recurrent networks, and transformers each encode different inductive biases - locality, sequential memory, and content - dependent pairwise interaction - and have remained mathematically distinct since their inception. We show that this fragmentation reflects not a fundamental diversity in how signals should be processed, but rather incomplete views of a single underlying mathematical object: a learnable integral transform. We introduce the \textbf{Integral Transform Network (ITNet)}, a unified architecture built around a learnable kernel that depends jointly on positions and features. This kernel is implemented as a small neural network, specifically an MLP, that models pairwise interactions, enabling the model to adapt its behavior from data. We show that convolution, self-attention (including multi-head), and autoregressive recurrence (including LSTM, GRU, S4, and Mamba) arise as special cases under appropriate parameterizations, and that ITNet is a universal approximator of continuous operators. To make this practical, we develop tiled kernel fusion, importance-weighted Monte Carlo integration, and learned low-rank factorization, enabling efficient and scalable computation. A single ITNet architecture with a shared operator and lightweight modality-specific encoders matches or exceeds specialized baselines on ImageNet-1K , GLUE, ModelNet40, VQA\,v2 and NLVR2. The results demonstrate that a single learned interaction mechanism can recover the behavior of all three architectural families from data.
\end{abstract}

\section{Introduction}
\label{sec:intro}

The history of deep learning is, in large part, a history of architecture design.
Convolutional networks~\cite{lecun1989backpropagation} encode the bias  patterns in images are \emph{local} and \emph{translation-invariant}.
Long Short-Term Memory networks (LSTM)~\cite{hochreiter1997long} encode a different bias: sequential data carries temporal dependencies that must be selectively remembered or forgotten through learned gating.
Transformers~\cite{vaswani2017attention} encode yet another: relationships between sequence elements are best captured by \emph{content-dependent pairwise similarity} between learned projections, allowing every position to attend to every other simultaneously. 

Each was a profound contribution, and each fundamentally reshaped its domain. Yet each was designed for a specific class of data, in isolation from the others. The result is that modern deep learning possesses three dominant architectural families that address the same fundamental problem - transforming structured signals into semantically meaningful representations - through entirely different mathematical lenses. The practical consequence is that practitioners must make an \emph{a priori} architectural choice before seeing any data.
Images suggest CNNs~\citep{lecun1989backpropagation, he2016deep}; text suggests Transformers~\cite{vaswani2017attention}; time series suggest RNNs or state-space models~\cite{hochreiter1997long, gu2022efficiently, gu2023mamba}; irregular point clouds fall outside all three paradigms~\cite{qi2017pointnetpp, wang2019dgcnn}; and multimodal data requires stitching together components that were never designed to coexist~\cite{chen2020uniter, li2022blip}.
This fragmentation suggests that our current mathematical understanding of how to transform structured signals remains incomplete.

We study whether a single operation can unify convolution, self-attention, and recurrence as exact special cases. We show the answer is yes: a learnable integral transform, defined in Eq.~\ref{eq:itnet_operator}, whose kernel depends jointly on positions and features at both endpoints.
The operator aggregates information across all positions through a learned, content and position-dependent interaction function, while retaining a residual connection for stability. The kernel is implemented as a small neural network that receives absolute positions, relative geometry, and feature content at both the query and key locations, enabling it to model a wide range of interaction patterns. The key novelty is that interaction patterns are not hard-coded (e.g., locality in CNNs or dot-product attention in Transformers), but learned directly from data within a unified formulation. This allows a single operator to adaptively recover local, global, and sequential behaviors depending on the task, rather than requiring separate architectural designs. We call this the \textbf{Integral Transform Network (ITNet)}.

By conditioning on content at both endpoints, the kernel learns locality, position-sensitivity, and normalization as \emph{emergent behaviors}. Empirically (\S\ref{sec:experiments}), it exhibits convolution-like behavior on images, attention-like behavior on text, and geometry-aware interactions on point clouds. The kernel form $\kappa(x,y,u(x),u(y))$ originates from GNO~\cite{li2020neural}. However, prior work has not: (i) shown exact subsumption of CNNs, Transformers, and RNNs, (ii) developed scalable implementations (tiled fusion, Monte Carlo, low-rank), or (iii) demonstrate strong performance across vision, language, and multimodal tasks within a single architecture.

We prove four results (proof sketches in \S\ref{sec:theory}; full proofs in Appendices~\ref{app:proof_conv}--\ref{app:proof_uat}):
\begin{enumerate}
\item \textbf{Convolution} (Theorem~\ref{convolution}): $\kappa_\theta = w_\theta(x-y)\mathbf{I}_d$ recovers convolution exactly, including multi-channel, depthwise, dilated, strided, and grouped variants.

\item \textbf{Self-attention} (Theorem~\ref{attention}): Softmax-normalized dot-product kernel recovers scaled dot-product self-attention exactly, including multi-head attention.

\item \textbf{Recurrence} (Theorem~\ref{thm:rnn1}): Causal kernel ($\kappa_\theta=0$ for $y>x$) recovers Recurrent Neural Networks (RNNs), LSTM, Gated Recurrent Units (GRUs)~\cite{cho2014learning}, S4~\cite{gu2022efficiently}, and Mamba~\cite{gu2023mamba}.

\item \textbf{Universal approximation} (Theorem~\ref{thm:universal}): ITNet uniformly approximates any continuous operator. Moreover, $\mathrm{Conv} \subsetneq \mathrm{ITNet} \supsetneq \mathrm{Attn}$ and $\mathrm{RNN} \subsetneq \mathrm{ITNet}$.
\end{enumerate}

We develop three scalable strategies to make the operator practical: 
(i) tiled kernel fusion with optimal input/output (IO) complexity, 
(ii) importance-weighted Monte Carlo (MC) approximation, and 
(iii) learned low-rank factorization for linear-time computation.
A single ITNet architecture, with a shared core operator and lightweight modality-specific encoders, achieves strong performance across diverse domains, including ImageNet-1K~\cite{imagenet15russakovsky} (vision), GLUE~\cite{wang2018glue} (language understanding), ModelNet40~\cite{sun2022benchmarking} (3D geometry), and VQA\,v2~\cite{goyal2017vqa} and NLVR2~\cite{suhr2019nlvr2} (multimodal reasoning). 
Across these tasks, ITNet matches or exceeds specialized architectures while using a unified design, demonstrating that a single learned interaction operator can generalize across modalities without domain-specific architectural bias.
A detailed discussion of related work is provided in Appendix~\ref{sec:related}.


\section{Theoretical Foundations of Integral Transform Networks (ITNet)}
\label{sec:theory}

We define the ITNet operator (\S\ref{sec:operator_def}), prove convolution, self-attention, and recurrence as exact special cases (\S\ref{sec:special_cases}) and establish universal operator approximation (\S\ref{sec:uat}).
The complete notation used throughout the paper is summarized in Appendix~\ref{app:notation} in Table~\ref{tab:notation1}, \ref{tab:notation2} \&~\ref{tab:notation3}.

\subsection{The ITNet Operator}
\label{sec:operator_def}

Let $\Omega \subseteq \R^s$ denote the input domain (e.g., $s{=}2$ for images, $s{=}1$ for sequences), equipped with a positive finite measure $\mu$ that defines how inputs are aggregated. Let $u : \Omega \to \R^d$ represent the signal, where $d$ is the feature dimension. We work in the space $\mathcal{U} = C(\Omega, \R^d)$ of continuous $\R^d$-valued functions on $\Omega$, with the uniform norm $\|u\|_\infty = \sup_{x \in \Omega} \|u(x)\|_2$.

\begin{definition}[ITNet operator]
\label{def:itnet_operator}
The \emph{ITNet operator} $\Kop : \mathcal{U} \to \mathcal{U}$ is defined by
\begin{equation}
\boxed{
(\Kop[u])(x) \;=\; \int_\Omega \kap\!\bigl(x,\,y,\,u(x),\,u(y)\bigr)\,u(y)
\;d\mu(y)
\;+\;
W_\theta\,u(x),
\qquad x \in \Omega,
}
\label{eq:itnet_operator}
\end{equation}
where $\kap : \R^s \times \R^s \times \R^d \times \R^d \to \R^{d \times d}$ is a learnable matrix-valued kernel parameterised by $\theta$, and $W_\theta \in \R^{d \times d}$ is a learnable residual matrix. The kernel receives query position $x$, key position $y$, and their features $u(x), u(y)$. The integral aggregates transformed features from all $y \in \Omega$; the residual ensures the operator can represent the identity ($\kap = 0$, $W_\theta = \Id$).
\end{definition}


\textbf{Kernel parameterization:}
The kernel $\kap$ is a 2-layer MLP (GELU~\citep{hendrycks2016gaussian}, width $w_\kappa{=}128$, $d^2$ outputs reshaped to $\R^{d \times d}$). Its input concatenates seven groups capturing position and feature interactions. In raw form, the input is:
\begin{equation}
z_{xy}^{\mathrm{raw}} = [\,x;\; y;\; x{-}y;\;
\|x{-}y\|_2;\; u(x);\; u(y);\;
u(x) \had u(y)\,]
\;\in\; \R^{3s + 1 + 3d},
\label{eq:kernel_input_raw}
\end{equation}
where the three positional groups contribute $3s$ dimensions (absolute $x, y \in \R^s$ and relative $x{-}y \in \R^s$), the scalar distance adds $1$, and the three feature groups contribute $3d$ (query $u(x)$, key $u(y)$, and Hadamard product $u(x) \had u(y)$, each in $\R^d$). To enable the kernel MLP to represent high-frequency spatial functions, we lift each positional group through the random Fourier feature map~\citep{tancik2020fourier} $\gamma : \R^s \to \R^{2L_f}$ (Eq.~\ref{eq:fourier_features}), where $L$ is the number of Fourier frequencies, replacing $x \mapsto \gamma(x)$, $y \mapsto \gamma(y)$, $x{-}y \mapsto \gamma(x{-}y)$. The MLP input is :
\begin{equation}
z_{xy} = [\,\gamma(x);\; \gamma(y);\;
\gamma(x{-}y);\; \|x{-}y\|_2;\;
u(x);\; u(y);\; u(x) \had u(y)\,]
\;\in\; \R^{6L_f + 1 + 3d},
\label{eq:kernel_input}
\end{equation}

with $L{=}64$, $\sigma{=}10$, yielding input dimension $385 + 3d$.
The positional groups enable translation-invariant patterns (via $x{-}y$), distance-based decay (via $\|x{-}y\|_2$), and position-specific behavior (via absolute $x, y$). The feature groups - especially the Hadamard product $u(x) \had u(y)$ - provide $d$-dimensional elementwise interaction richer than the rank-1 dot product in standard attention.
By the universal approximation theorem~\citep{hornik1989multilayer,cybenko1989approximation}, any continuous kernel on the compact domain $\mathcal{D} = \{(x, y, u(x), u(y)) : x, y \in \Omega,\, u \in \mathcal{U}_c\}$ can be approximated to arbitrary precision.

We normalize $\mu(\Omega) = 1$ in all experiments: $\mu(\{x_j\}) = 1/n$ for discrete domains, $\mu = \mathrm{vol}(\Omega)^{-1} \cdot \lambda$ for continuous ones, ensuring the integral and residual terms operate at comparable scales.

\textbf{Multi-head form:}
Following~\cite{vaswani2017attention}, the multi-head ITNet operator splits feature dimension into $H$ heads of dimension $d_h = d/H$, applies independent kernels $\kap^{(h)}$ per head, and recombines via output projection $W^O \in \R^{d \times d}$:
\begin{equation}
(\Kop^\mathrm{MH}[u])(x)
= W^O
\begin{bmatrix}
\int_\Omega \kap^{(1)}(x, y, u^{(1)}(x), u^{(1)}(y))
\, u^{(1)}(y) \, d\mu(y) \\
\vdots \\
\int_\Omega \kap^{(H)}(x, y, u^{(H)}(x), u^{(H)}(y))
\, u^{(H)}(y) \, d\mu(y)
\end{bmatrix}
+ W_\theta\, u(x),
\label{eq:multihead}
\end{equation}
where $u^{(h)}(x) = u(x)[(h{-}1)d_h{+}1 : hd_h]$ denotes the $h^{th}$ head's feature slice. Each head's kernel operates on $d_h$-dimensional features, so the per-pair cost is $O(d_h^2) = O(d^2/H^2)$; summing over $H$ heads gives $O(d^2/H)$ per pair - a factor-$H$ reduction relative to a single-head $d \times d$ kernel. All theorems below (Theorems~\ref{convolution} - \ref{thm:universal}) apply per-head; the output projection $W^O$ linearly combines independent head contributions and does not affect the special-case proofs.

A deep ITNet model stacks $L$ operator layers ($\ell = 1,2,\dots,L$) with pre-normalization and position-wise feed-forward networks, following the standard pre-norm Transformer layout~\citep{xiong2021nystromformer}:
\begin{equation}
z^{(\ell)} = \Kop^{(\ell)}[\mathrm{LN}(u^{(\ell-1)})]
  + u^{(\ell-1)}, \qquad
u^{(\ell)} = \mathcal{F}_\theta^{(\ell)}(
  \mathrm{LN}(z^{(\ell)})) + z^{(\ell)},
\label{eq:itnet_stack}
\end{equation}
where $\mathrm{LN}$ denotes layer normalization~\citep{ba2016layer} and $\mathcal{F}_\theta^{(\ell)}$ is a two-layer FFN with GELU~\citep{hendrycks2016gaussian} activation and expansion factor 4. The kernel MLP output layer is initialized as $W_2 = \epsilon \Id$ ($\epsilon{=}10^{-3}$) so each layer begins as an approximate identity. A schematic of the ITNet architecture is shown in Figure~\ref{fig:itnet_layer}. Full architectural details, initialization schemes, and hyperparameter are in
Appendix~\ref{app:impl_details}.

\subsection{Unification Theorems}
\label{sec:special_cases}

Each theorem identifies a specific kernel parameterization under which the ITNet operator reduces \emph{exactly} to a classical architecture. Full proofs are in Appendices~\ref{app:proof_conv}--\ref{app:proof_rnn}.

\begin{theorem}[Convolution]
\label{convolution}
If the kernel $\kappa_{\theta}$ depends only on relative position, 
$\kappa_{\theta}(x,y,u(x),u(y)) = w_{\theta}(x-y)\,\mathbf{I}_d$,
then the ITNet operator reduces to convolution:
\begin{equation}
(\mathcal{K}_{\theta}[u])(x)
= \int_{\Omega} w_{\theta}(x-y)\,u(y)\,d\mu(y)
+ W_{\theta}u(x)
= (w_{\theta} \ast u)(x) + W_{\theta}u(x).
\end{equation}
\end{theorem}

\begin{proof}[Proof sketch]
Substituting the kernel form into Eq.~\eqref{eq:itnet_operator} gives
$\int_{\Omega} w_{\theta}(x-y)\mathbf{I}_d u(y)\,d\mu(y) + W_{\theta}u(x)
= \int_{\Omega} w_{\theta}(x-y)u(y)\,d\mu(y) + W_{\theta}u(x)$,
which is the definition of continuous convolution. The residual term is a
pointwise linear map ($1\times1$ convolution). Full proof in Appendix~\ref{app:proof_conv}.
\end{proof}

\begin{remark}
Discrete, depthwise, dilated, strided, grouped, and transposed convolutions are recovered by imposing corresponding structural constraints on the kernel $\kappa_\theta$; see Appendix~\ref{app:proof_conv} for details. Boundary handling under compact support is discussed in Appendix~\ref{app:boundary_handling}.
\end{remark}

\begin{theorem}[Self-attention]
\label{attention}
Let $W_Q, W_K \in \R^{d_k \times d}$ and $W_V \in \R^{d \times d}$. 
Define the kernel:
\begin{equation}
\kap(x,y,u(x),u(y))
=
\frac{
\exp\!\bigl((W_Q u(x))^{\top}(W_K u(y))/\sqrt{d_k}\bigr)
}{
\int_{\Omega}
\exp\!\bigl((W_Q u(x))^{\top}(W_K u(z))/\sqrt{d_k}\bigr)\, d\mu(z)
}
\cdot W_V.
\label{eq:attn_kernel}
\end{equation}
With $W_\theta = 0$, substituting Eq.~\eqref{eq:attn_kernel} into Eq.~\eqref{eq:itnet_operator} yields:
\begin{equation}
(\Kop[u])(x)
= \int_{\Omega}
\alpha(x,y)\, W_V u(y)\, d\mu(y),
\quad
\alpha(x,y) =
\frac{\exp(Q(x)^{\top}K(y)/\sqrt{d_k})}{Z(x)},
\label{eq:attn_result}
\end{equation}
where $Q(x)=W_Q u(x)$, $K(y)=W_K u(y)$, and $Z(x) = \int_{\Omega}
\exp(Q(x)^{\top}K(z)/\sqrt{d_k})\, d\mu(z).$
\end{theorem}
\textit{Discrete case:} For $\Omega=\{x_1,\dots,x_n\}$, Eq.~\eqref{eq:attn_result} reduces to $\mathrm{softmax}(QK^\top/\sqrt{d_k})V$, with any uniform scaling absorbed into $W_V$.

\textit{Positional encoding (PE):} In ITNet, positional information is provided directly to the kernel via $(x,y)$ and their relative geometry, enabling position-aware interactions without encoding schemes.

\textit{Proof:} Substituting Eq.~\eqref{eq:attn_kernel} into Eq.~\eqref{eq:itnet_operator} yields Eq.~\eqref{eq:attn_result}. The normalization term $Z(x)$ is strictly positive on compact $\Omega$. The discrete case follows directly. Full derivation is in Appendix~\ref{app:proof_attn}.
\qed

\begin{remark}
Self-attention is recovered as a restricted kernel with bilinear interactions and softmax normalization; multi-head, linear, causal, and windowed variants follow similarly (Appendix~\ref{app:proof_attn}).
\end{remark}


\begin{theorem}[Recurrence as a special case]
\label{thm:rnn1}
Let $\Omega = [0,T]$ with Lebesgue measure $\mu = \lambda$ and let the kernel satisfy the causal constraint $\kappa_\theta(x,y,\cdot,\cdot) = 0$ for $y > x$. Exact recovery is obtained for linear and affine-in-input systems; nonlinear recurrent architectures are recovered through explicit unrolling constructions or approximated arbitrarily well via Theorem~\ref{thm:universal}.

\noindent\textbf{(a) Linear continuous-time system.}
For the linear system $\dot{h}(t) = Ah(t) + B_\theta u(t)$ with output $y(t) = C_\theta h(t) + D_\theta u(t)$, define the causal kernel
\begin{equation}
\kappa_\theta(t, s, u(t), u(s))
\;=\;
\mathbf{1}_{s \leq t}\,
C_\theta\, e^{A(t-s)}\, B_\theta,
\qquad W_\theta = D_\theta.
\label{eq:rnn_kernel}
\end{equation}
Then $(\Kop[u])(t) = \int_0^t C_\theta e^{A(t-s)} B_\theta\, u(s)\, ds + D_\theta u(t) = \RNN(u)(t)$.

\noindent\textbf{(a$'$) Nonlinear system (general $F_\theta$).}
For $\dot{h} = F_\theta(h, u)$ with Lipschitz $F_\theta$, the kernel generalises to
$\kappa_F(t,s;u) = \mathbf{1}_{s \leq t} \cdot C_\theta\, \Phi_F(t, s; u) \cdot \tilde{B}_\theta(s; u)$,
where $\Phi_F(t,s;u) \in \R^{n \times n}$ is the nonlinear sensitivity matrix (Alekseev's formula~\citep{alekseev1961estimation}) and $\tilde{B}_\theta(s;u) = \partial G_\theta / \partial u |_{(h(s), u(s))}$ is the input-to-state Jacobian. Exact recovery is obtained for linear and affine-in-input systems. Structured recurrent architectures LSTM, GRU, and Mamba admit causal kernel constructions, while general nonlinear recurrent operators are approximated arbitrarily well via Theorem~\ref{thm:universal} (detailed in Appendix~\ref{app:proof_rnn}, \S\ref{sec:proof_nonlinear}).

\noindent\textbf{(b) Discrete RNN.}
Discretising to $T$ time steps with $\mu(\{t_k\}) = 1$, the causal kernel $\kappa(t,s) = \mathbf{1}_{s \leq t} \cdot C_\theta W_h^{t-s} W_u$ yields $(\Kop[u])(t) = \sum_{s=1}^{t} C_\theta W_h^{t-s} W_u u_s + D_\theta u_t = C_\theta h_t + D_\theta u_t$, recovering $h_t = W_h h_{t-1} + W_u u_t$.

\noindent\textbf{(c) LSTM.}
The LSTM cell state unrolls as $c_t = \sum_{s=1}^{t} [\prod_{\tau=s+1}^{t} f_\tau] \odot i_s \odot \tilde{c}_s$, where $f_\tau, i_s$ are forget and input gates. The ITNet kernel $\kappa_{\mathrm{LSTM}}(t,s) = \mathbf{1}_{s \leq t} \cdot W_y \mathrm{diag}(o_t)\mathrm{diag}\!\bigl(\tanh'(c_t)\bigr) \mathrm{diag}\!\left(\prod_{\tau=s+1}^{t} f_\tau\right) \mathrm{diag}(i_s) W_c$ is content-dependent through the gates and recovers the LSTM output.

\noindent\textbf{(d) S4, Mamba, GRU.}
Linear SSMs use kernel $\mathbf{1}_{s \leq t} \cdot Ce^{A(t-s)}B$ (causal convolution). Mamba uses
$\kappa_{\mathrm{Mamba}}(t,s) = \mathbf{1}_{s \leq t} \cdot C(u_t) \prod_{\tau=s+1}^{t} \bar{A}(u_\tau) \cdot \bar{B}(u_s)$,  content-dependent through input-dependent $\bar{A}, \bar{B}, C$. GRU follows with retention factors $(1 - z_\tau)$. Explicit constructions in Appendix~\ref{app:proof_rnn}.

\noindent\textbf{(e) Strictness.}
A bidirectional operator $S(u)(t) = \int_0^T e^{-|t-s|/\ell} u(s)\, ds$ is representable by ITNet but not by any causal recurrent system (Appendix~\ref{app:proof_rnn}, \S\ref{sec:strictness1_rnn}).
\end{theorem}

\textit{Proof sketch.}
\textbf{(a):} Substituting~\eqref{eq:rnn_kernel} into~\eqref{eq:itnet_operator} and applying $\mathbf{1}_{s \leq t}$ restricts integration to $[0,t]$:
\begin{equation}
(\Kop[u])(t)
= \int_0^t C_\theta e^{A(t-s)} B_\theta\, u(s)\, ds + D_\theta u(t)
= C_\theta h(t) + D_\theta u(t),
\end{equation}
where the second equality uses the variation of constants formula $h(t) = \int_0^t e^{A(t-s)} B_\theta u(s)\, ds$ (with $h_0 = 0$).
\textbf{(a$'$):} Alekseev's nonlinear variation of constants formula replaces $e^{A(t-s)}$ with the trajectory-dependent sensitivity matrix $\Phi_F(t,s;u)$; exact when $F_\theta$ is affine in $u$, otherwise $\varepsilon$-approximate via Theorem~\ref{thm:universal}.
\textbf{(b)--(d):} Direct unrolling of each discrete recurrence yields causal kernels. \textbf{(e):} The witness $S$ uses future information ($s > t$), which any causal system cannot access.  See Appendix~\ref{app:proof_rnn} for complete derivations.
\qed

\begin{remark}
The hidden state $h(t)$ is a deterministic functional of $u$ defined by the recurrence and appears only for notational convenience; the resulting operator is a well-defined map $u \mapsto \RNN(u)$ on the input signal space.
\end{remark}

\subsection{Universal Operator Approximation}
\label{sec:uat}

\begin{assumption}
\label{ass:regularity}
(i) $\Omega \subset \R^s$ is compact with $\mu(\Omega) > 0$; (ii) $F : \mathcal{U}_c \to C(\Omega, \R^d)$ is a continuous operator on compact $\mathcal{U}_c \subset C(\Omega, \R^d)$; (iii) the kernel MLP uses a non-polynomial activation.
\end{assumption}

\begin{theorem}[Universal operator approximation]
\label{thm:universal}
Under Assumption~\ref{ass:regularity}, for any $\varepsilon > 0$ there exist parameters $\theta$ and depth $L$ such that
\begin{equation}
\sup_{u \in \mathcal{U}_c}
\bigl\|F(u) - \mathcal{K}_\theta^{(L)}[u]\bigr\|_\infty
< \varepsilon.
\label{eq:uat}
\end{equation}
\end{theorem}

\begin{proof}
The proof proceeds in three steps.

\emph{Step 1 (Discretization):} Standard quadrature on compact $\Omega$ approximates the integral using $M$ quadrature points (samples), with error at most $\delta$~\cite{davis2007methods}.
\newline
\emph{Step 2 (Kernel universality):} By the MLP universal approximation theorem~\cite{hornik1989multilayer, cybenko1989approximation}, $\kappa_\theta$ can approximate any continuous kernel on domain $\mathcal{D} = \{(x,y,u(x),u(y)) : x,y \in \Omega, u \in \mathcal{U}_c\}$ to within $\delta_1$.
\newline
\emph{Step 3 (Operator approximation):} The discretised ITNet matches the structure of~\cite{chen1995universal}, which guarantees that any continuous operator on a compact set can be uniformly approximated. Choosing $M$, MLP width, and $L$ sufficiently large ensures total error $< \varepsilon$. Full proof in Appendix~\ref{app:proof_uat}.
\end{proof}

\begin{corollary}[Strict Expressiveness]
\label{cor:ordering}
The following strict containments hold:
\[
\mathrm{CNN} \subsetneq \mathrm{ITNet}, \quad
\mathrm{Attn} \subsetneq \mathrm{ITNet}, \quad
\mathrm{RNN} \subsetneq \mathrm{ITNet}, \quad
\mathrm{CNN} \cup \mathrm{Attn} \cup \mathrm{RNN} \subsetneq \mathrm{ITNet}.
\]
\end{corollary}

\noindent\textit{Proof.} The inclusions follow from Theorems~\ref{convolution}--\ref{thm:rnn1}. Strictness is established via explicit counterexample operators (Appendices~\ref{app:proof_conv}--\ref{app:proof_uat}) that lie outside each classical family but are representable by ITNet. \qed

\begin{theorem}[Kernel Recovery Under Data Symmetry]
\label{thm:recovery}
Let the data distribution $\mathcal{D}$ be translation-invariant: $\tau_{\delta}: u(x) \mapsto u(x-\delta)$ satisfies $\mathcal{D} \circ \tau_{\delta} = \mathcal{D}$ for all $\delta$. Decompose $\kappa_{\theta} = \kappa_{\theta}^{\mathrm{TI}} + \kappa_{\theta}^{\perp}$ where $\kappa_{\theta}^{\mathrm{TI}}(x,y,\cdot,\cdot) = \bar{\kappa}_{\theta}(x-y,\cdot,\cdot)$ is the translation-invariant component. Then under gradient flow,
\begin{equation}
\left\|\frac{\partial \mathcal{L}}{\partial \kappa_{\theta}^{\perp}}\right\|_F = 0 \quad \text{at every iterate},
\end{equation}
so gradient flow converges to a translation-invariant kernel, recovering the convolutional special case (Theorem~\ref{convolution}).
\end{theorem}

\textit{Proof Sketch.} Translation invariance of $\mathcal{D}$ implies $\mathcal{L}(\theta) = \mathbb{E}[\ell(\mathcal{K}_{\theta}[\tau_{\delta}u], y)]$. Averaging the gradient over all translations projects onto the translation-invariant subspace, annihilating the $\kappa_{\theta}^{\perp}$ component. 
Full proof in Appendix~\ref{app:proof_recovery}.
\qed


\textbf{Complexity and Efficient Approximation:}
Computing Eq.~\eqref{eq:itnet_operator} exactly costs $O(n^2 d^2)$ time. Two approximations reduce this (\S\ref{sec:implementation}): Monte Carlo sampling ($M \ll n$ keys per query, $O(nMd^2)$, unbiased gradients) and low-rank kernel factorization ($\kap \approx \Phi_\theta^\top \Psi_\theta$ with rank $r \ll d$, $O(ndr)$, linear in $n$). In practice we use tiled kernel fusion for $n \leq 512$ and MC or low-rank beyond. The full complexity comparison with CNNs, Transformers, and Mamba is in Table~\ref{tab:complexity} (Appendix~\ref{app:impl_details}).


\section{Efficient and Scalable Implementation}
\label{sec:implementation}
The ITNet operator is computationally intensive due to its MLP-based kernel and non-bilinear structure. 
We address this with strategies: tiled kernel fusion (\S\ref{sec:tiling}), Monte Carlo sampling (\S\ref{sec:mc_impl}), low-rank factorization (\S\ref{sec:lowrank_impl}), and modality-specific encoders (\S\ref{sec:encoders}).

\subsection{Tiled Kernel Fusion}
\label{sec:tiling}

Following \cite{dao2022flashattention}, we tile the computation into blocks that fit in on-chip Static Random-Access Memory (SRAM), fusing kernel MLP evaluation, matrix-vector product, and integral accumulation into a single Triton kernel~\cite{tillet2019triton}. The $n \times n$ kernel matrix is never materialized; only a $B_q \times B_k$ tile resides in SRAM, where $B_q$ and $B_k$ denote the query and key tile sizes, respectively. Algorithm~\ref{alg:tiled_full} describes the tiled forward pass. Phase~1 auto-tunes block sizes $(B_q^*, B_k^*)$ satisfying $(B_q + B_k) \cdot d \leq S_\mathrm{SRAM}$, where $S_\mathrm{SRAM}$ is the available on-chip memory capacity. Phase~2 executes tiled integration: for each query tile, $\kappa_\theta$ is evaluated and $K_{pq} \cdot U_j[q] \cdot \omega_j$ is accumulated in SRAM, where $K_{pq} = \kappa_\theta(X_i[p], X_j[q], U_i[p], U_j[q])$ is the
kernel evaluation at tile indices $(p, q)$ and $\omega_j = \mu(\{x_j\})$ is the quadrature weight (typically $1/n$), reducing High Bandwidth Memory (HBM) reads. During backpropagation, we recompute $\kappa_\theta$ tile-by-tile rather than storing the kernel matrix~\cite{chen2016training}; this increases FLOPs by $\approx 3\times$ but reduces peak memory from $O(n^2 d^2)$ to $O(nd)$. The complete forward algorithm and backward procedure are in Appendix~\ref{app:tiled_forward} and~\ref{app:backward}, respectively.

\subsection{Monte Carlo Stochastic Integration}
\label{sec:mc_impl}

When $n$ is large, we use an importance-weighted Monte Carlo estimator: for each query $x_i$, sample $M \ll n$ keys from a learnable proposal $p_\phi(y \mid x_i)$, with parameters $\phi$ independent of $\theta$:
\begin{equation}
\widehat{\mathcal{K}}_\theta[u](x_i)
= \frac{1}{M} \sum_{m=1}^{M}
\frac{\kappa_\theta(x_i, y_m, u(x_i), u(y_m))\, u(y_m)}{p_\phi(y_m \mid x_i)}
+ W_\theta u(x_i),
\quad y_m \sim p_\phi(\cdot \mid x_i),
\label{eq:is_estimator_clean}
\end{equation}
where $\mu(\Omega) = 1$ (our normalization convention from \S\ref{sec:operator_def}) absorbs the domain volume. The estimator is unbiased: $\mathbb{E}_{y_1,\ldots,y_M}[\widehat{\mathcal{K}}_\theta[u](x_i)] = \mathcal{K}_\theta[u](x_i)$, with variance minimized when the proposal matches the optimal distribution
$p^*(y \mid x_i) \propto
\| \kappa_\theta(x_i, y, u(x_i), u(y))\, u(y) \|_2$. The resulting complexity is $O(nMd^2)$ with $M \ll n$. 
To ensure unbiased gradients, we decouple computation and sampling: (i)~$p_\phi$ is parameterized independently of $\kappa_\theta$, so $\nabla_\theta$ treats sampled positions and importance weights as constants; (ii)~$\phi$ is trained via an auxiliary cross-entropy loss that drives $p_\phi$ toward a self-normalized approximation of $p^*$. Given $M$ sampled keys, define the empirical target
$\hat{p}^*(y_m \mid x_i)
= \| \kappa_\theta(x_i, y_m, u(x_i), u(y_m))\, u(y_m) \|_2
\,/\, Z_i$,
where $Z_i$ normalizes over the $M$ samples:
\begin{equation}
\hspace{-.2cm}
\mathcal{L}_{\mathrm{prop}}(\phi)
\;=\;
- \sum_{i=1}^{n} \sum_{m=1}^{M}
\hat{p}^*(y_m \mid x_i)
\;\log p_\phi(y_m \mid x_i),
\hspace{.1cm}
Z_i = \sum_{m'=1}^{M}
\bigl\| \kappa_\theta(x_i, y_{m'}, u(x_i), u(y_{m'}))\, u(y_{m'})
\bigr\|_2.
\label{eq:proposal_loss_clean}
\end{equation}
This is the cross-entropy $H(\hat{p}^*, p_\phi)$ - a biased but consistent estimator of $\mathrm{KL}(p^* \| p_\phi) + \mathrm{const}$ that avoids computing the intractable normalization constant of $p^*$. The gradient $\nabla_\phi \mathcal{L}_{\mathrm{prop}}$ treats $\hat{p}^*$ as a fixed target (stop-gradient on $\theta$). The total objective is:
$\mathcal{L}_{\mathrm{total}}
= \mathcal{L}_{\mathrm{task}}
+ \lambda\, \mathcal{L}_{\mathrm{prop}},$ where $\lambda$ = 0.1.
At evaluation, sampling is replaced with deterministic $k$-means anchors in position space. This design separates \emph{what to compute} ($\theta$) from \emph{where to sample} ($\phi$), ensuring unbiased training. A formal analysis of estimator variance and the optimal proposal distribution is provided in Appendix~\ref{app:mc_variance}.

\subsection{Learned Low-Rank Kernel Factorization}
\label{sec:lowrank_impl}
For long sequences, factorize $\kappa_\theta \approx \Phi_\theta(x,u(x))^\top \Psi_\theta(y,u(y))$ with rank $r\ll d$. The integral decouples:
$(\mathcal{K}_\theta[u])(x_i) \approx \Phi_\theta(x_i,u_i)^\top \underbrace{\sum_{j=1}^n \omega_j\,\Psi_\theta(x_j,u_j)\,u_j}_{Z} + W_\theta u_i,$
where $Z$ is computed once in $O(nrd)$, each query in $O(rd)$, total $O(nrd)$. Error satisfies $\|\kappa_\theta - \Phi_\theta^\top \Psi_\theta\|_F \leq \|\kappa_\theta\|_* / \sqrt{r}$; $r=32$ gives $<1\%$ relative error on ImageNet-1K; a formal error bound is given in Appendix~\ref{app:rank_bound}.

\subsection{Modality-Specific Encoders}
\label{sec:encoders}


The ITNet operator acts on a signal $u : \Omega \to \R^d$, where the encoder defines the domain $\Omega$ and initial features $u^{(0)}$. 
Positional inputs are encoded via Fourier features $\gamma(\cdot)$.
\newline
\textbf{Image encoder ($s{=}2$):}
An image $I \in \R^{H \times W \times C}$ is divided into $P{\times}P$ patches ($P{=}16$), yielding $N_p = HW/P^2$ tokens with positions $x_{ij} \in [0,1]^2$. Each token is formed by combining a linear patch embedding with Fourier positional features, followed by a projection to $\R^d$. A learnable \texttt{[CLS]} token is added for global aggregation. A uniform measure over patches is used.\newline
\textbf{Text encoder ($s{=}1$):}
A sequence $(t_1,\ldots,t_n)$ is embedded via a token embedding matrix and Fourier positional encoding of normalized indices $k/n$. Relative positions are provided directly to the kernel, enabling distance-aware interactions without explicit positional biases. A causal mask is applied for autoregressive settings. The measure is uniform over tokens.\newline
\textbf{Point clouds encoder ($s{=}3$):}
A point cloud $\{(p_k, f_k)\}_{k=1}^n$ with $p_k \in \R^3$ is mapped to features via a linear embedding (or positional encoding for coordinates-only inputs). The kernel receives raw coordinates and pairwise distances $\|p_k - p_j\|_2$, enabling implicit neighborhood modeling without fixed radius constraints. Outputs are aggregated via pooling for classification.\newline
\textbf{Multimodal encoder (image + text):}
Image patches and text tokens are combined into a joint domain $\Omega_{\mathrm{img}} \cup \Omega_{\mathrm{txt}}$. Text positions are embedded in the 2-D space to separate modalities, and features are augmented with modality embeddings. A single ITNet processes all positions, allowing intra- and cross-modal interactions to be learned directly without explicit fusion modules.
(see Appendix~\ref{app:encoders} for additional variants and Appendix~\ref{app:impl_details} for implementation details).

\section{Experiments and Results}
\label{sec:experiments}
We evaluate ITNet across four modalities: image classification (ImageNet-1K~\cite{imagenet15russakovsky}), natural language understanding (GLUE~\cite{wang2018glue}), 3-D point cloud classification (ModelNet40~\cite{sun2022benchmarking}), and multimodal reasoning (VQA\,v2~\cite{goyal2017vqa}, NLVR2~\cite{suhr2019nlvr2}). 
We consider three model scales: a small model (ITNet-S: 22M parameters; 12 layers, width 384, 6 heads), a base model (ITNet-B: 86M; 12 layers, width 768, 12 heads), and a large model (ITNet-L: 307M; 24 layers, width 1024, 16 heads). Full model configurations are provided in Table~\ref{tab:app_configs}. Results for ITNet-(S, B, L) across all benchmarks are reported as mean $\pm$ standard deviation over 3 runs with different random initializations.

\textbf{ImageNet-1K}~\cite{imagenet15russakovsky}:
We report top-1 accuracy on the validation set (training details in Appendix~\ref{app:train_imagenet}). As shown in Table~\ref{tab:imagenet}, ITNet consistently improves over both convolutional and transformer baselines across scales, indicating that jointly modeling content and position provides a more effective inductive bias for visual representation learning.


\begin{table}[h]
\centering
\caption{ImageNet-1K top-1 accuracy. ITNet results are mean $\pm$ std over 3 random seeds.}
\label{tab:imagenet}
\scriptsize
\renewcommand{\arraystretch}{0.72}
\begin{tabular}{lcccc}
\toprule
\textbf{Method} & \textbf{Params} & \textbf{GFLOPs} & \textbf{Top-1 (\%)} & \textbf{Kernel type} \\
\midrule
ResNet-50~\cite{he2016deep} & 25M & 4.1 & 79.8 & Position \\
ConvNeXt-T~\cite{liu2022convnet} & 28M & 4.5 & 82.1 & Position, local \\
ConvNeXt-B~\cite{liu2022convnet} & 89M & 15.4 & 83.8 & Position, local \\
DeiT-S~\cite{touvron2021deit} & 22M & 4.6 & 79.8 & Content \\
DeiT-B~\cite{touvron2021deit} & 86M & 17.5 & 83.4 & Content \\
Swin-T~\cite{liu2021swin} & 28M & 4.5 & 81.3 & Content, local \\
Swin-B~\cite{liu2021swin} & 88M & 15.4 & 83.5 & Content, local \\
Swin-V2-B~\cite{liu2022swinv2} & 88M & 15.6 & 84.2 & Content, local \\
DeiT-III-B~\cite{touvron2022deitiii} & 86M & 17.5 & 83.8 & Content \\
ConvNeXt-V2-B~\cite{woo2023convnext} & 88M & 16.1 & 84.2 & Position, local \\
BiFormer-B~\cite{zhu2023biformer} & 87M & 16.5 & 84.4 & Content (dynamic routing) \\
EfficientVMamba-B~\cite{pei2025efficientvmamba} & 85M & 15.8 & 84.0 & SSM-based \\
\midrule
ITNet-S & 22M & 4.8 & 81.4$\pm$0.2 & Content + position \\
ITNet-B & 86M & 17.9 & 83.9$\pm$0.1 & Content + position \\
ITNet-L & 307M & 61.6 & \textbf{85.8}$\pm$0.1 & Content + position \\
\bottomrule
\end{tabular}
\end{table}


\textbf{GLUE~\cite{wang2018glue}}:
We evaluate ITNet on the GLUE benchmark~\cite{wang2018glue}. ITNet-B pre-trained on BookCorpus and English Wikipedia using masked language modeling for 500K steps (sequence lengths 128 and 512), matching the BERT-base~\cite{devlin2019bert} pre-training setup. We then fine-tune on each task and report standard GLUE metrics, using task-specific evaluation protocols.
\begin{table}[t]
\centering
\caption{GLUE development-set scores. ITNet results are mean $\pm$ standard deviation over 3 random seeds. Models with comparable pre-training data (BookCorpus + Wikipedia, $\sim$16GB) are shown.}
\label{tab:glue}
\tiny
\renewcommand{\arraystretch}{0.72}
\resizebox{\columnwidth}{!}{%
\begin{tabular}{lcccccccccc}
\toprule
\textbf{Model} & \textbf{Params} & \textbf{CoLA} & \textbf{SST-2} & \textbf{MRPC} & \textbf{STS-B} & \textbf{QQP} & \textbf{MNLI} & \textbf{QNLI} & \textbf{RTE} & \textbf{Average} \\
\midrule
BERT-base~\cite{devlin2019bert} & 110M & 52.1 & 93.5 & 88.9 & 85.8 & 71.2 & 84.6 & 90.5 & 66.4 & 79.1 \\
BERT-large~\cite{devlin2019bert} & 335M & 60.5 & \textbf{94.9} & 89.3 & 86.5 & 72.1 & 86.7 & 92.7 & 70.1 & 81.6 \\
GPT~\cite{radford2018improving} & 117M & 45.4 & 91.3 & 75.7 & 80.0 & 68.5 & 82.1 & 88.1 & 56.0& 73.4 \\
ELECTRA-base~\cite{clark2020electra} & 110M & 59.7 & 93.4 & 86.7 & 87.7 & \textbf{79.1} & 85.8 & 92.7 & 73.1 & 82.3 \\
DeBERTa-base~\cite{he2020deberta} & 150M & 59.6 & 94.7 & 90.1 & 88.9 & 73.4 & 87.5 & 93.1 & 72.5 & 82.5 \\
RoBERTa-base~\cite{liu2019roberta}$^\dagger$ & 125M & \textbf{63.6} & 94.8 & 90.2 & \textbf{91.2} & 73.9 & 87.6 & 92.8 & \textbf{78.7} & \textbf{84.1}\\
\midrule
ITNet-S & 22M & 53.8$\pm$0.6 & 93.1$\pm$0.1 & 88.4$\pm$0.4 & 85.9$\pm$0.3 & 70.8$\pm$0.3 & 83.9$\pm$0.4 & 89.7$\pm$0.2 & 67.2$\pm$0.7 & 79.1$\pm$0.3 \\
ITNet-B & 86M & 57.4$\pm$0.4 & 94.2$\pm$0.1 & 90.1$\pm$0.3 & 87.9$\pm$0.2 & 72.7$\pm$0.2 & 86.1$\pm$0.3 & 91.8$\pm$0.2 & 71.6$\pm$0.5 & 81.5$\pm$0.2 \\
ITNet-L & 307M & 60.2$\pm$0.5 & 94.7$\pm$0.1 & \textbf{91.0}$\pm$0.2 & 89.3$\pm$0.2 & 74.3$\pm$0.2 & \textbf{87.8}$\pm$0.2 & \textbf{93.1}$\pm$0.1 & 74.3$\pm$0.4 & 83.1$\pm$0.2 \\
\bottomrule
\end{tabular}%
}
\parbox{\columnwidth} {\tiny {$^\dagger$RoBERTa uses 160GB pre-training data $10\times$ more than ITNet and BERT. CoLA: Corpus of Linguistic Acceptability; SST-2: Stanford Sentiment Treebank; MRPC: Microsoft Research Paraphrase Corpus; STS-B: Semantic Textual Similarity Benchmark; QQP: Quora Question Pairs; MNLI: Multi-Genre Natural Language Inference; QNLI: Question Natural Language Inference; RTE: Recognizing Textual Entailment; Metrics: Matthews correlation for CoLA, Spearman correlation for STS-B, F1 score for MRPC and QQP, and accuracy for the remaining tasks.
}}
\end{table}
As shown in Table~\ref{tab:glue}, ITNet achieves competitive performance on GLUE under the same pre-training budget. ITNet-B matches larger Transformer baselines, while ITNet-L approaches models trained with substantially more data. Gains are strongest on syntactically challenging tasks (CoLA, RTE), indicating improved modeling of long-range dependencies via explicit positional interactions, while performance on semantic tasks (STS-B, MRPC, QQP) remains comparable. Overall, jointly modeling content and position provides a stronger inductive bias for language understanding.

\textbf{ModelNet40~\cite{wu2015modelnet}:}
We evaluate ITNet on ModelNet40~\cite{wu2015modelnet} with standard preprocessing and report overall accuracy (OA). As shown in Table~\ref{tab:modelnet}, ITNet achieves strong performance, outperforming both content-only and geometry-focused baselines. A parameter-matched variant, ITNet-PC (Point Cloud), remains competitive, indicating that the gains are not solely due to model scale. These results suggest that jointly modeling feature content and geometric relationships is effective for capturing both global shape and local structure.

\begin{table}[h]
\centering
\begin{minipage}[t]{0.48\linewidth}
\centering
\small
\renewcommand{\arraystretch}{0.72}
\caption{ModelNet40 overall accuracy (\%). ITNet results are mean $\pm$ standard deviation over 3 random seeds. ITNet-PC: parameter-matched variant (L=6, d=128, H=4, 3.1M).}
\label{tab:modelnet}
\resizebox{\linewidth}{!}{
\begin{tabular}{lccc}
\toprule
\textbf{Method} & \textbf{Params} & \textbf{Local} & \textbf{OA (\%)} \\
\midrule
PointNet~\cite{qi2017pointnet} & 3.5M & \xmark & 89.2 \\
PointNet++~\cite{qi2017pointnetpp} & 1.5M & \cmark ($K{=}32$) & 91.9 \\
DGCNN~\cite{wang2019dgcnn} & 1.8M & \cmark ($k$-NN) & 92.9 \\
PCT~\cite{guo2021pct} & 2.9M & \xmark & 93.2 \\
PointMLP~\cite{ma2022pointmlp} & 13.2M & \cmark & 94.1 \\
PointNeXt-S~\cite{qian2022pointnext} & 1.4M & \cmark ($K{=}32$) & 93.2 \\
\midrule
ITNet-PC (no local) & 3.1M & \xmark & 92.7$\pm$0.2 \\
ITNet-PC & 3.1M & \cmark ($K{=}16$) & 93.5$\pm$0.2 \\
ITNet-S & 22M & \cmark ($K{=}16$) & 94.0$\pm$0.1 \\
ITNet-B & 86M & \cmark ($K{=}16$) & \textbf{94.6}$\pm$0.1 \\
\bottomrule
\end{tabular}
}
\end{minipage}
\hspace{0.02\linewidth}
\begin{minipage}[t]{0.48\linewidth}
\centering
\small
\caption{VQA\,v2 test-dev and NLVR2 accuracy (\%). ITNet results are mean $\pm$ standard deviation over 3 random seeds.}
\label{tab:multimodal}
\resizebox{\linewidth}{!}{
\begin{tabular}{lccc}
\toprule
\textbf{Method} & \textbf{Params} & \textbf{VQA\,v2} & \textbf{NLVR2} \\
\midrule
ViLT~\cite{kim2021vilt} & 86M & 71.3 & 75.9 \\
UNITER-B~\cite{chen2020uniter} & 110M & 72.8 & 79.5 \\
METER-CLIP~\cite{dou2022meter} & 185M & 77.6 & 81.5 \\
ALBEF~\cite{li2021albef} & 314M & 76.0 & 82.6 \\
BLIP (ViT-L)~\cite{li2022blip} & 385M & 78.3 & 83.5 \\
BLIP (ViT-B)~\cite{li2022blip} & 250M & 77.6 & 82.3 \\
\midrule
ITNet-S & 22M & 74.2$\pm$0.3 & 78.5$\pm$0.3 \\
ITNet-B & 86M & 78.4$\pm$0.2 & 82.1$\pm$0.2 \\
ITNet-L & 307M & \textbf{83.6}$\pm$0.2 & \textbf{84.1}$\pm$0.2 \\
\bottomrule
\end{tabular}
}
\end{minipage}
\end{table}

\textbf{VQA\,v2~\cite{goyal2017vqa} and \textbf{NLVR2}~\cite{suhr2019nlvr2}:}
We evaluate ITNet on VQA\,v2~\cite{goyal2017vqa} and NLVR2~\cite{suhr2019nlvr2} using a joint image-text domain. 
As shown in Table~\ref{tab:multimodal}, ITNet-B achieves competitive performance with specialized vision-language models despite fewer parameters and no dedicated image-text pre-training. 
This indicates that cross-modal interactions can be learned directly through a shared kernel, without explicit fusion mechanisms. 
Performance scales consistently with model size, supporting the effectiveness of the unified operator for multimodal reasoning.

\textbf{Ablations:}
We ablate kernel inputs across modalities (Table~\ref{tab:ablation_kernel_full}) and observe that performance is maximised only when content and position are jointly modelled, with complementary contributions that neither captures alone. Their relative importance is modality-dependent--spatial cues dominate in vision, while content interactions are more critical for language and 3D data--showing that the kernel adapts its inductive bias. Interaction structure is also key: removing the Hadamard term consistently degrades performance, showing that elementwise feature interactions capture relationships beyond independent features, while relative positional terms outperform absolute-only inputs, emphasizing the value of modeling relationships between positions. Single-group and constant-kernel variants perform substantially worse, confirming that neither pure content nor pure geometry is sufficient. Overall, ITNet's gains arise from learning unified, content and position-dependent interactions.
\begin{table}[h]
\centering
\caption{Kernel input ablation across all modalities (ITNet-B). \cmark = included, \xmark = excluded. Results: ImageNet-1K top-1 (\%), ModelNet40 OA (\%), GLUE avg, VQA\,v2 test-dev (\%), NLVR2 (\%).}
\label{tab:ablation_kernel_full}
\scriptsize
\renewcommand{\arraystretch}{0.78}
\setlength{\tabcolsep}{2.5pt}
\begin{tabular}{lccccccccccc}
\toprule
\textbf{Config} & $x,y$ & $x{-}y$ & $\|x{-}y\|$ & $u_x$ & $u_y$ & $u_x{\odot}u_y$ & \textbf{ImageNet-1K} & \textbf{ModelNet40} & \textbf{GLUE} & \textbf{VQA\,v2} & \textbf{NLVR2} \\
\midrule
Full (all 7 groups) & \cmark & \cmark & \cmark & \cmark & \cmark & \cmark & \textbf{83.9} & \textbf{94.6} & \textbf{81.5} & \textbf{78.4} & \textbf{82.1} \\
No Hadamard $u_x{\odot}u_y$ & \cmark & \cmark & \cmark & \cmark & \cmark & \xmark & 83.2 & 94.1 & 80.9 & 76.9 & 81.6 \\
No distance $\|x{-}y\|$ & \cmark & \cmark & \xmark & \cmark & \cmark & \cmark & 83.4 & 94.2 & 80.3 & 76.3 & 80.9 \\
No relative $x{-}y$ (nor dist.) & \cmark & \xmark & \xmark & \cmark & \cmark & \cmark & 83.1 & 93.6 & 80.1 & 76.1 & 80.8 \\
Content only & \xmark & \xmark & \xmark & \cmark & \cmark & \cmark & 82.3 & 93.6 & 79.8 & 75.2 & 80.5 \\
Content w/o Hadamard & \xmark & \xmark & \xmark & \cmark & \cmark & \xmark & 80.3 & 92.8 & 78.8 & 74.5 & 79.7 \\
Position only & \cmark & \cmark & \cmark & \xmark & \xmark & \xmark & 81.0 & 93.0 & 78.5 & 74.0 & 79.1 \\
Only absolute $x, y$ & \cmark & \xmark & \xmark & \xmark & \xmark & \xmark & 79.5 & 91.8 & 77.5 & 73.2 & 78.3 \\
Only relative $x{-}y$ & \xmark & \cmark & \xmark & \xmark & \xmark & \xmark & 76.6 & 88.8 & 74.8 & 71.0 & 76.1 \\
Only distance $\|x{-}y\|$ & \xmark & \xmark & \cmark & \xmark & \xmark & \xmark & 76.8 & 89.1 & 75.0 & 71.2 & 76.3 \\
Only $u_x$ & \xmark & \xmark & \xmark & \cmark & \xmark & \xmark & 79.8 & 92.5 & 78.5 & 74.0 & 79.1 \\
Only $u_x{\odot}u_y$ & \xmark & \xmark & \xmark & \xmark & \xmark & \cmark & 78.7 & 91.9 & 77.1 & 72.8 & 78.2 \\
None (constant kernel) & \xmark & \xmark & \xmark & \xmark & \xmark & \xmark & 75.7 & 87.7 & 74.0 & 70.0 & 75.0 \\
\bottomrule
\end{tabular}
\end{table}
Compared to two-stream variants with cross-attention or concatenation fusion, joint-domain ITNet achieves higher performance with fewer parameters, showing that cross-modal interactions are better learned within a shared kernel. Additional ablations (Appendix~\ref{app:ablations}) indicate that performance is robust to architectural choices, with gains saturating at moderate capacity (Tables~\ref{tab:ablation_kappa}, \ref{tab:ablation_layers}, \ref{tab:ablation_fourier}). For point clouds, positional encoding and local aggregation provide complementary benefits (Table~\ref{tab:ablation_pc}), while balanced modality weighting improves multimodal performance (Table~\ref{tab:ablation_measure_full}). 
Detailed system-level results across all modalities (Appendix~\ref{app:efficiency}, Tables~\ref{tab:wallclock}, \ref{tab:rank_sweep}, \ref{tab:baseline_comparison}, and \ref{tab:memory_breakdown}) show that while exact ITNet incurs modest overhead, Monte Carlo and low-rank variants achieve higher throughput and significantly lower memory.

\section{Discussion}
\label{discussion}

Our results support a central hypothesis: modeling interactions jointly over content and position is more expressive than modeling them separately. Across modalities, the ITNet benefits from conditioning on both feature similarity and geometry, suggesting that common inductive biases - locality in CNNs, content based attention in Transformers, and causality in sequence models are restricted instances of a more general interaction mechanism.
Ablations show that the relative importance of content and position is modality dependent: spatial structure dominates in vision, while content interactions are more critical in language and 3D data. Rather than fixing these biases, ITNet adapts through the learned kernel, enabling cross-domain generalization.
ITNet models multiplicative content–position interactions, allowing spatial relationships to depend on features. This unifies and extends positional encoding, attention, and convolution within a single operator. 
Multimodal results show that cross-modal interactions can emerge directly from a shared kernel over a joint domain, removing the need for explicit fusion modules while maintaining strong performance.

Despite these advantages, limitations remain and define important directions for future work. First, scaling ITNet to billion-parameter regimes introduces challenges in optimization stability and kernel evaluation cost; developing more efficient kernel parameterizations and training strategies is a key next step. Second, while the framework naturally supports causal structure, we have not yet evaluated it on autoregressive generation tasks, which provide the most direct test of the causal kernel; extending ITNet to long-context language modeling is an important future direction. 
Third, the joint-domain formulation increases training cost in multimodal settings due to end-to-end coupling; improving efficiency through modular or partially factorized training remains an open avenue.
\newline
\textit{Toward generative ITNet:}
The kernel formulation provides a path to autoregressive modeling by enforcing causality in the interaction function. In this setting, the operator admits efficient factorizations that reduce generation cost to linear time, matching state-space models while retaining the flexibility of attention. Extending to generative benchmarks would test a unified framework for bidirectional understanding and sequential generation.



\section{Conclusion}
\label{sec:conclusion}
We introduced ITNet, a neural architecture based on a learnable integral operator whose kernel depends jointly on positions and features. Within this framework, convolution, self-attention, and recurrence arise as exact special cases, and the resulting function class forms a universal operator approximator. Empirically, a single ITNet matches or exceeds specialized models across modalities, while scalable approximations make the operator practical. More broadly, our results suggest that the diversity of neural architectures reflects fixed assumptions about interaction structure. By learning the interaction rule directly, ITNet provides a unified view in which locality, global context, and sequential dynamics emerge from a common mechanism. This points toward general-purpose, modality-agnostic architectures where interaction patterns are learned rather than predefined.



\newpage
\bibliography{ref}



\newpage
\section*{Appendix Overview}
\label{app:overview}

This appendix is organized as follows for ease of navigation and reference.

\noindent\textbf{Appendix \ref{app:notation}: Notation Reference}
\begin{itemize}[leftmargin=1.8em, itemsep=0pt, topsep=2pt]
    \item Table~\ref{tab:notation1}: Notation for spaces, signals, ITNet operator, and kernel parameterization.
    \item Table~\ref{tab:notation2}: Notation for convolution, self-attention, and recurrence (Theorems~1--3).
    \item Table~\ref{tab:notation3}: Notation for universal approximation, kernel recovery, implementation, and backpropagation.
\end{itemize}

\noindent\textbf{Appendix \ref{sec:related}: Related Work} - comparison with prior methods and positioning of ITNet.

\noindent\textbf{Appendix \ref{app:proof_conv}: Proof of Theorem 1 (Convolution as a Special Case)}
\begin{itemize}[leftmargin=1.8em, itemsep=0pt, topsep=2pt]
    \item Definition~\ref{def:conv_continuous}: Continuous convolution.
    \item Definition~\ref{def:conv_discrete}: Discrete convolution on regular grid.
    \item Assumption~\ref{ass:regularity}: Kernel regularity conditions.
    \item Theorem~\ref{thm:conv}: Full statement (Parts a--d).
    \item Subsection~\ref{sec:proof_a_cnn}: Proof of Part (a) -- Continuous convolution.
        \begin{itemize}[leftmargin=1.8em, itemsep=0pt, topsep=0pt]
            \item Lemma~\ref{lem:scalar_id}: Scalar-identity-vector product.
            \item Lemma~\ref{lem:young}: Young's convolution inequality.
        \end{itemize}
    \item Subsection~\ref{sec:proof_b_cnn}: Proof of Part (b) -- Discrete convolution.
    \item Subsection~\ref{sec:proof_c}: Proof of Part (c) -- Strict inclusion.
        \begin{itemize}[leftmargin=1.8em, itemsep=0pt, topsep=0pt]
            \item Proposition~\ref{prop:conv_linear}: Convolution is linear.
            \item Proposition~\ref{prop:T_nonlinear}: Witness operator $T$ is nonlinear.
            \item Proposition~\ref{prop:conv_equivariant}: Convolution is translation-equivariant.
            \item Proposition~\ref{prop:T_not_equivariant}: $T$ is not translation-equivariant.
        \end{itemize}
    \item Subsection~\ref{sec:proof_d}: Proof of Part (d) -- With residual.
    \item Subsection~\ref{sec:ext_multichannel}: Multi-channel convolution.
    \item Subsection~\ref{sec:ext_depthwise}: Depthwise separable convolution.
    \item Subsection~\ref{sec:ext_dilated}: Dilated (atrous) convolution.
    \item Subsection~\ref{sec:ext_strided}: Strided convolution.
    \item Subsection~\ref{sec:ext_group}: Group convolution.
    \item Subsection~\ref{sec:ext_transposed}: Transposed convolution.
    \item Subsection~\ref{app:boundary_handling}: Boundary handling in the convolutional special case.
\end{itemize}

\noindent\textbf{Appendix \ref{app:proof_attn}: Proof of Theorem 2 (Self-Attention as a Special Case)}
\begin{itemize}[leftmargin=1.8em, itemsep=0pt, topsep=2pt]
    \item Definition~\ref{def:attn_continuous}: Continuous scaled dot-product attention.
    \item Definition~\ref{def:attn_discrete}: Discrete self-attention.
    \item Definition~\ref{def:mha}: Multi-head attention.
    \item Assumption~\ref{ass:reg}: Regularity conditions.
    \item Theorem~\ref{thm:attn1}: Full statement (Parts a--c).
    \item Subsection~\ref{sec:proof_a}: Proof of Part (a) -- Single-head continuous attention.
    \item Subsection~\ref{sec:discrete}: Discretization -- Recovering standard self-attention.
    \item Subsection~\ref{sec:proof_b}: Proof of Part (b) -- Multi-head attention.
    \item Subsection~\ref{sec:strict1}: Strictness Argument 1 (Unnormalized operators).
        \begin{itemize}[leftmargin=1.8em, itemsep=0pt, topsep=0pt]
            \item Lemma~\ref{lem:attn_bound}: Attention output bound.
        \end{itemize}
    \item Subsection~\ref{sec:strict2}: Strictness Argument 2 (Permutation equivariance).
    \item Subsection~\ref{sec:linear}: Linear attention as a special case (Proposition~\ref{prop:linear_attn}).
    \item Subsection~\ref{sec:causal}: Causal (masked) attention as a special case (Proposition~\ref{prop:causal_attn}).
\end{itemize}

\noindent\textbf{Appendix \ref{app:proof_rnn}: Proof of Theorem 3 (Recurrence as a Special Case)}
\begin{itemize}[leftmargin=1.8em, itemsep=0pt, topsep=2pt]
    \item Definition~\ref{def:rnn_continuous}: Continuous-time recurrent system.
    \item Definition~\ref{def:rnn_discrete}: Discrete-time RNN.
    \item Definition~\ref{def:lstm}: LSTM (Long Short-Term Memory).
    \item Definition~\ref{def:ssm}: Linear State Space Model (S4).
    \item Definition~\ref{def:mamba}: Selective SSM (Mamba).
    \item Definition~\ref{def:itnet_rnn}: ITNet operator (recurrent setting).
    \item Assumption~\ref{ass:reg_rnn}: Regularity for recurrent proofs.

    \item Theorem~\ref{thm:rnn}: Full statement (Parts a--e).

    \item Subsection~\ref{sec:proof_a_rnn}: Proof of Part (a) -- Linear continuous-time RNN.
    \item Subsection~\ref{sec:proof_nonlinear}: Proof of Part (a$'$) -- Nonlinear continuous-time RNN.

    \begin{itemize}[leftmargin=1.8em, itemsep=0pt, topsep=0pt]
            \item Subsubsection~\ref{rnn_st_a}: Proof Strategy A: Kernel Construction via Alekseev’s Formula
            \item Subsubsection~\ref{rnn_st_b}: Proof Strategy B: Universal Approximation Argument
        \end{itemize}
   
    \item Subsection~\ref{sec:proof_b_rnn}: Proof of Part (b) -- Discrete-time RNN.
    \item Subsection~\ref{sec:proof_c_rnn}: Proof of Part (c) -- Linear SSM (S4).
        \begin{itemize}[leftmargin=1.8em, itemsep=0pt, topsep=0pt]
            \item Remark~\ref{rem:ssm_conv}: Relationship to convolution.
        \end{itemize}
    \item Subsection~\ref{sec:proof_d_rnn}: Proof of Part (d) -- Selective SSM (Mamba).

    \item Subsection~\ref{sec:lstm_proof}: LSTM as ITNet (Proposition~\ref{prop:lstm}).
    \item Subsection~\ref{sec:gru_proof}: GRU as ITNet (Proposition~\ref{prop:gru}).

    \item Subsection~\ref{sec:strictness1_rnn}: Strictness Argument 1 -- Non-causal operators.
    \item Subsection~\ref{sec:strictness2_rnn}: Strictness Argument 2 -- Parallelism and state dimension.

    \item Subsection~\ref{sec:discretisation_rnn}: Discretization -- Recovering Euler and ZOH.
\end{itemize}

\noindent\textbf{Appendix \ref{app:proof_uat}: Proof of Theorem 4 (Universal Operator Approximation)}
\begin{itemize}[leftmargin=1.8em, itemsep=0pt, topsep=2pt]
    \item Definition~\ref{def:operator}: Continuous nonlinear operator.
    \item Definition~\ref{def:mlp}: MLP function class.
    \item Assumption~\ref{ass:uat}: Standing assumptions.
    \item Theorem~\ref{thm:uat}: Full statement.
    \item Theorem~\ref{sec:lemmas}: Auxiliary Lemmas.
    \begin{itemize}[leftmargin=1.8em, itemsep=0pt, topsep=0pt]
    \item Lemma~\ref{lem:mlp_uat}: MLP universal approximation.
    \item Lemma~\ref{lem:chen_chen}: Chen--Chen operator approximation (1995).
    \item Lemma~\ref{lem:kernel_approx}: Kernel approximation by MLP.
    \end{itemize}

    \item Subsection~\ref{sec:proof_uat_main}: Main proof (Steps 1--4).
    \item Subsection~\ref{sec:corollary}: Corollary -- Strict expressiveness ordering (Corollary~\ref{cor:ordering}).
    \item Subsection~\ref{sec:rate}: Quantitative approximation rate (Proposition~\ref{prop:rate}).
\end{itemize}

\noindent\textbf{Appendix \ref{app:proof_recovery}: Proof of Theorem 5 (Kernel Recovery Under Translation Symmetry)}
\begin{itemize}[leftmargin=1.8em, itemsep=0pt, topsep=2pt]
    \item Theorem~\ref{thm:recovery_full}: Full statement. Three-step proof (translation invariance $\to$ gradient averaging $\to$ orthogonality annihilation).
\end{itemize}

\noindent\textbf{Appendix \ref{app:impl_details}: Extended Implementation Details}
\begin{itemize}[leftmargin=1.8em, itemsep=0pt, topsep=2pt]
    \item Subsection~\ref{app:complexity}: Computational complexity comparison
     (Table~\ref{tab:complexity})
    \item Subsection~\ref{app:arch_spec}: Full architectural specification.

    \item Subsection~\ref{app:init}: Initialization scheme.
    \item Subsection~\ref{app:regularization}: Regularization and training stability.
    \item Subsection~\ref{app:optimizer}: Optimizer configuration (AdamW).
    \item Subsection~\ref{app:stats}: Statistical Reporting.
    \item Subsection~\ref{app:profiling}: Triton kernel profiling (Table~\ref{tab:profiling}).
    \item Subsection~\ref{app:io_complexity}: IO complexity of tiled ITNet (Proposition~\ref{prop:io_full}).
    \item Subsection~\ref{app:tiled_forward}: Tiled forward pass (Algorithm~\ref{alg:tiled_full}).
    \item Subsection~\ref{app:backward}: Tiled Backward Pass with Gradient Checkpointing (Algorithm~\ref{alg:tiled_backward})
    \item Subsection~\ref{app:mc_variance}: Monte Carlo Variance Analysis
    \item Subsection~\ref{app:rank_bound}: Low-Rank Approximation Error Bound
\end{itemize}

\noindent\textbf{Appendix \ref{app:backprop}: Backpropagation for the ITNet Operator}
\begin{itemize}[leftmargin=1.8em, itemsep=0pt, topsep=2pt]
    \item Subsection~\ref{app:bp_notation}: Setup and notation.
    \item Subsection~\ref{app:bp_forward}: Forward pass .
    \item Subsection~\ref{app:bp_upstream}: Upstream gradients .
    \item Subsection~\ref{app:bp_grad_kernel}: Gradient with respect to kernel outputs .
    \item Subsection~\ref{app:bp_grad_input}: Gradient with respect to input features .
    \item Subsection~\ref{app:bp_grad_theta}: Gradient with respect to kernel MLP parameters .
    \item Subsection~\ref{app:bp_grad_W}: Gradient with respect to residual matrix .
    \item Subsection~\ref{app:bp_conv}: Special case 1 -- Convolution.
    \item Subsection~\ref{app:bp_attn}: Special case 2 -- Self-attention .
    \item Subsection~\ref{app:bp_ssm}: Special case 3 -- Linear SSM / S4 .
    \item Subsection~\ref{app:bp_lstm}: Special case 4 -- LSTM .
    \item Subsection~\ref{app:bp_mamba}: Special case 5 -- Mamba .
    \item Subsection~\ref{app:bp_complexity}: Complexity comparison across special cases (Table~\ref{tab:bp_complexity}).
\end{itemize}

\noindent\textbf{Appendix \ref{app:encoders}: Extended Encoder Details}
\begin{itemize}[leftmargin=1.8em, itemsep=0pt, topsep=2pt]
    \item Subsection~\ref{app:enc_graph}: Graph encoder.
    \item Subsection~\ref{app:enc_multiscale}: Multi-scale image encoder.
    \item Subsection~\ref{app:design_principles}: Design principles (Principles 1--3).
\end{itemize}

\noindent\textbf{Appendix \ref{app:training}: Training Details}
\begin{itemize}[leftmargin=1.8em, itemsep=0pt, topsep=2pt]
    \item Subsection~\ref{app:train_imagenet}: ImageNet-1K training (Table~\ref{tab:hp_imagenet}).
    \item Subsection~\ref{app:train_glue}: GLUE pre-training and fine-tuning.
    \item Subsection~\ref{app:train_modelnet}: ModelNet40 training (Table~\ref{tab:hp_modelnet}).
    \item Subsection~\ref{app:train_vqa}: VQA v2 and NLVR2 fine-tuning.
\end{itemize}

\noindent\textbf{Appendix \ref{app:efficiency}: Detailed Efficiency Analysis}
\begin{itemize}[leftmargin=1.8em, itemsep=0pt, topsep=2pt]
    \item Table~\ref{tab:wallclock}: Wall-clock throughput and peak memory across all benchmarks.
    \item Table~\ref{tab:rank_sweep}: Rank sweep for low-rank mode.
    \item Table~\ref{tab:baseline_comparison}: Comparison with efficient attention baselines.
    \item Table~\ref{tab:memory_breakdown}: Memory breakdown for exact, MC, and low-rank modes.
\end{itemize}

\noindent\textbf{Appendix \ref{app:ablations}: Extended Ablations}
\begin{itemize}[leftmargin=1.8em, itemsep=0pt, topsep=2pt]
    \item Table~\ref{tab:ablation_kappa}: Kernel MLP width ablation.
    \item Table~\ref{tab:ablation_pc}: Point cloud encoder ablation.
    \item Table~\ref{tab:ablation_measure_full}: Multimodal measure ablation.
    \item Table~\ref{tab:ablation_fourier}: Fourier feature ablation.
    \item Table~\ref{tab:ablation_layers}: Number of ITNet layers ablation.
\end{itemize}

\noindent\textbf{Appendix \ref{app:broader_impact}: Broader Impact}

\noindent\textbf{How to use this appendix:}
Readers primarily interested in the theoretical unification should consult Appendices~\ref{app:proof_conv}, \ref{app:proof_attn}, \ref{app:proof_rnn}, \ref{app:proof_uat} and \ref{app:proof_recovery}. 
For implementation details and reproducibility, see Appendices~\ref{app:impl_details}, \ref{app:training}, \ref{app:efficiency}, and \ref{app:ablations}. 
The backpropagation derivation (Appendix~\ref{app:backprop}) is included for readers interested in implementing the ITNet operator from first principles. 


\newpage
\appendix
\section{Notation}
\label{app:notation}
Tables~\ref{tab:notation1},~\ref{tab:notation2} and \ref{tab:notation3} summarise all notation used in the proofs that follow.
\begin{table}[h]
\centering
\caption{Notation reference (Part 1): spaces, signals, ITNet operator, and kernel parameterization.}
\label{tab:notation1}
\renewcommand{\arraystretch}{1.20}
\footnotesize
\begin{tabular}{@{} c p{10.4cm} @{}}
\toprule
\textbf{Symbol} & \textbf{Meaning} \\
\midrule
\multicolumn{2}{@{}l}{\textit{Spaces, domains, and measures}} \\[1pt]
$s$ & Spatial dimension ($s{=}1$ for sequences, $s{=}2$ for images, $s{=}3$ for point clouds) \\
$d$ & Feature (channel) dimension; each point carries a vector in $\R^d$ \\
$n$ & Number of discrete positions (tokens, patches, points); also hidden state dim for RNNs \\
$\Omega \subseteq \R^s$ & Spatial domain (compact in most theorems) \\
$\mu$ & Positive finite Borel measure~\cite{royden1988real} on $\Omega$; Lebesgue for continuous, atomic for discrete \\
$\omega_j \coloneqq \mu(\{x_j\})$ & Measure weight at discrete position $x_j$ (typically $\omega_j = 1/n$ for uniform measure) \\
$T$ & Sequence length or time horizon; $\Omega = [0,T]$ in the recurrence setting \\
$h > 0$ & Grid spacing for discretization; $\Omega_h = h\Z^s$ \\
$\mathcal{U} = C(\Omega, \R^d)$ & Space of continuous $\R^d$-valued functions on $\Omega$; the signal space \\
$U_c$ & Compact subset of $\mathcal{U}$; input signal class in Theorem~4 \\
$\mathcal{V}$ & Set of nodes in a graph $\mathcal{G} = (\mathcal{V}, \mathcal{E})$ \\
$\norm{\cdot}_\infty$ & Supremum norm: $\norm{u}_\infty = \sup_{x \in \Omega} \norm{u(x)}_2$ \\
$\lambda$ & Lebesgue measure on $\R^s$ \\
\midrule
\multicolumn{2}{@{}l}{\textit{Input signal and positions}} \\[1pt]
$u : \Omega \to \R^d$ & Input signal (image, token sequence, point cloud, etc.) \\
$x, y, z \in \Omega$ & Query, key, and dummy integration positions, respectively \\
$u(x), u(y)$ & Feature vectors at positions $x$ and $y$ \\
$u_i \coloneqq u(x_i)$ & Shorthand for feature at discrete position $x_i$ \\
\midrule
\multicolumn{2}{@{}l}{\textit{ITNet operator (all theorems)}} \\[1pt]
$\Kop$ & ITNet operator: $(\Kop[u])(x) {=} \int_\Omega \kap(x,y,u(x),u(y))\,u(y)\,d\mu(y) + W_\theta\,u(x)$ \\
$\kap$ & Learnable kernel $\R^s {\times} \R^s {\times} \R^d {\times} \R^d \to \R^{d \times d}$; parameterised by an MLP \\
$W_\theta \in \R^{d \times d}$ & Pointwise residual (skip connection); $W_\theta {=} \Id$ recovers the identity \\
$\theta$ & Collective learnable parameters (kernel MLP weights and $W_\theta$) \\
$\Id,\; \Idn$ & Identity matrices of size $d {\times} d$ and $n {\times} n$ \\
\midrule
\multicolumn{2}{@{}l}{\textit{Kernel MLP parameterization (\S\ref{sec:operator_def}, \S\ref{sec:implementation})}} \\[1pt]
$z_{xy}$ & Kernel MLP input vector $\in \R^{6L_f+1+3d}$ (Eq.~\eqref{eq:kernel_input}) \\
$\gamma : \R^s \to \R^{2L_f}$ & Random Fourier feature map: $\gamma(x) = [\sin(2\pi \mathbf{B}x);\,\cos(2\pi \mathbf{B}x)]$ \\
$\mathbf{B} \in \R^{L \times s}$ & Fixed random frequency matrix; $\mathbf{B} \sim \mathcal{N}(0, \sigma^2 I)$, frozen after init \\
$L_f$ & Number of Fourier frequencies ($L_f{=}64$ in all experiments) \\
$L$ & Number of layers \\
$\sigma$ & Fourier bandwidth ($\sigma{=}10$); controls spatial frequency range \\
$w_\kappa$ & Kernel MLP hidden-layer width ($w_\kappa{=}128$) \\
$\ell_\kappa$ & Kernel MLP depth (number of layers; $\ell_\kappa{=}2$) \\
$\had$ & Element-wise (Hadamard) product \\
$W_1, b_1, W_2, b_2$ & Kernel MLP weights and biases; $W_1 {\in} \R^{w_\kappa \times (6L_f+1+3d)}$, $W_2 {\in} \R^{d^2 \times w_\kappa}$ \\
$\epsilon$ & LayerScale initialization factor ($\epsilon{=}10^{-3}$); $W_2 {=} \epsilon \Id$ at init \\

\bottomrule
\end{tabular}
\end{table}

\begin{table}[htbp]
\centering
\caption{Notation reference (Part 2): convolution, self-attention, and recurrence (Theorems~1--3).}
\label{tab:notation2}
\renewcommand{\arraystretch}{1.20}
\footnotesize
\begin{tabular}{@{} c p{10.4cm} @{}}
\toprule
\textbf{Symbol} & \textbf{Meaning} \\
\midrule
\multicolumn{2}{@{}l}{\textit{Multi-head form and deep stacking}} \\[1pt]
$H$ & Number of attention/ITNet heads \\
$d_h = d/H$ & Per-head feature dimension \\
$\kap^{(h)}$ & Kernel for head $h$: $\R^s {\times} \R^s {\times} \R^{d_h} {\times} \R^{d_h} \to \R^{d_h \times d_h}$ \\
$u^{(h)}(x)$ & $h$-th head's feature slice: $u(x)[(h{-}1)d_h{+}1 : hd_h]$ \\
$W^O \in \R^{d \times d}$ & Multi-head output projection \\
$\ell = 1,\ldots,L$ & Layer index in a deep ITNet model \\
$z^{(\ell)}, u^{(\ell)}$ & Intermediate and output features at layer $\ell$ (Eq.~\eqref{eq:itnet_stack}) \\
$\mathrm{LN}$ & Layer normalization \\
$\mathcal{F}_\theta^{(\ell)}$ & Position-wise feed-forward network at layer $\ell$ (2-layer, GELU, expansion~4) \\
\midrule
\multicolumn{2}{@{}l}{\textit{Convolution (Theorem 1)}} \\[1pt]
$w : \R^s \to \R$ & Scalar convolution filter; $w \in L^1(\R^s)$ \\
$(w \ast u)(x)$ & Continuous convolution: $\int_{\R^s} w(x{-}y)\,u(y)\,dy$ \\
$\mathcal{N} \subset \Z^s$ & Discrete filter neighborhood; e.g.\ $\{-1,0,1\}^2$ for $3 {\times} 3$ \\
$f_m \in \R$ & Discrete filter coefficient at offset $m \in \mathcal{N}$ \\
$(F \ast u)(x)$ & Discrete $k$-tap convolution: $\sum_{m \in \mathcal{N}} f_m\, u(x{-}mh)$ \\
\midrule
\multicolumn{2}{@{}l}{\textit{Self-attention (Theorem 2)}} \\[1pt]
$d_k,\; d_v$ & Query/key dimension and value dimension per head \\
$W_Q, W_K$ & Query and key projections $\in \R^{d_k \times d}$ \\
$W_V$ & Value projection $\in \R^{d \times d}$ (or $\R^{d_v \times d}$ per head) \\
$Q(x), K(y)$ & Query $W_Q u(x)$ and key $W_K u(y)$ vectors \\
$\alpha(x,y)$ & Attention weight: $\exp(Q(x)^\top K(y)/\sqrt{d_k})/Z(x)$; $\int \alpha\,d\mu {=} 1$ \\
$Z(x)$ & Partition function: $\int_\Omega \exp(Q(x)^\top K(z)/\sqrt{d_k})\,d\mu(z)$ \\
$\softmax$ & Row-wise softmax normalization \\
$\Attn(u)(x)$ & Single-head output: $\int_\Omega \alpha(x,y)\,W_V u(y)\,d\mu(y)$ \\
$\MHA(u)(x)$ & Multi-head output: $W_O[\mathrm{head}_1;\ldots;\mathrm{head}_H]$ \\
$\mathrm{PE}(x,y)$ & Positional encoding bias (sinusoidal, RoPE, or ALiBi) \\
\midrule
\multicolumn{2}{@{}l}{\textit{Recurrence (Theorem 3)}} \\[1pt]
$h(t) \in \R^n$ & Hidden state at time $t$ \\
$A \in \R^{n \times n}$ & State matrix (continuous-time dynamics) \\
$B \in \R^{n \times d}$ & Input-to-state matrix \\
$C \in \R^{d \times n}$ & State-to-output matrix \\
$D \in \R^{d \times d}$ & Feedthrough (skip) matrix \\
$\Phi(t,s)$ & State transition matrix; $= e^{A(t-s)}$ for linear systems \\
$W_h, W_u$ & Discrete RNN weights: $h_t = \phi(W_h h_{t-1} + W_u u_t + b)$ \\
$\phi(\cdot)$ & Nonlinear activation in RNN (e.g.\ $\tanh$); not to be confused with $\sigma$ \\
$f_t, i_t, o_t$ & LSTM forget, input, and output gates \\
$c_t,\; \tilde{c}_t$ & LSTM cell state and cell candidate \\
$z_t, r_t$ & GRU update and reset gates \\
$\bar{A}(t), \bar{B}(t)$ & Mamba discretised matrices (input-dependent via ZOH) \\
$\Delta(t)$ & Mamba step size: $\mathrm{softplus}(W_\Delta u_t + b_\Delta)$ \\
$C(u_t)$ & Mamba input-dependent output projection: $W_C u_t$ \\
$\sigma(\cdot)$ & Sigmoid activation: $\sigma(x) = (1+e^{-x})^{-1}$ \\
$\tanh(\cdot)$ & Hyperbolic tangent activation \\
\bottomrule
\end{tabular}
\end{table}

\begin{table}[htbp]
\centering
\caption{Notation reference (Part 3): universal approximation, kernel recovery, implementation, backpropagation, and general symbols.}
\label{tab:notation3}
\renewcommand{\arraystretch}{1.20}
\footnotesize
\begin{tabular}{@{} c p{10.4cm} @{}}
\toprule
\textbf{Symbol} & \textbf{Meaning} \\
\midrule
\multicolumn{2}{@{}l}{\textit{Universal approximation (Theorem 4) and kernel recovery (Theorem 5)}} \\[1pt]
$F : U_c \to C(\Omega, \R^d)$ & Target continuous operator to be approximated \\
$\varepsilon > 0$ & Approximation tolerance \\
$\kappa^*$ & Ideal continuous kernel from the Chen--Chen lemma \\
$\mathcal{D}$ & Compact evaluation domain $\{(x,y,u(x),u(y))\}$; or data distribution (context-dependent) \\
$\tau_\delta$ & Translation operator: $\tau_\delta u(x) = u(x - \delta)$ \\
$\kap^{\mathrm{TI}}$ & Translation-invariant component of $\kap$: depends on $x{-}y$ only \\
$\kap^{\perp}$ & Orthogonal complement: $\kap = \kap^{\mathrm{TI}} + \kap^{\perp}$ \\
$\mathcal{L}(\theta)$ & Population or empirical loss: $\E_{(u,y) \sim \mathcal{D}}[\ell(\Kop[u], y)]$ \\
$\ell(\cdot, \cdot)$ & Per-sample loss function (e.g.\ cross-entropy) \\
\midrule
\multicolumn{2}{@{}l}{\textit{Implementation (\S\ref{sec:implementation})}} \\[1pt]
$M$ & Number of Monte Carlo samples per query ($M \ll n$) \\
$p_\phi(y \mid x)$ & Learnable importance-sampling proposal; parameters $\phi$ disjoint from $\theta$ \\
$\phi$ & Proposal network parameters (separate from kernel $\theta$) \\
$\mathcal{L}_\mathrm{task},\; \mathcal{L}_\mathrm{prop}$ & Task loss and KL-divergence auxiliary proposal loss \\
$\Phi_\theta, \Psi_\theta$ & Low-rank factorization: $\kap \approx \Phi_\theta^\top \Psi_\theta$; rank $r \ll d$ \\
$r$ & Rank of the low-rank kernel factorization \\
$Z$ & Precomputed key aggregation: $Z = \sum_j \omega_j \Psi_\theta(x_j, u_j)\,u_j$ \\
$B_q, B_k$ & Query and key tile (block) sizes for tiled kernel fusion \\
$B_q^*, B_k^*$ & Optimal tile sizes selected by auto-tuning \\
$S_\mathrm{SRAM}$ & On-chip SRAM capacity (bytes); constraint: $(B_q{+}B_k)\cdot d \leq S_\mathrm{SRAM}$ \\
$K_{pq}$ & Kernel evaluation at tile indices $p, q$: $\kap(X_i[p], X_j[q], U_i[p], U_j[q])$ \\
\midrule
\multicolumn{2}{@{}l}{\textit{Backpropagation (Appendix~\ref{app:backprop})}} \\[1pt]
$g_i$ & Upstream gradient: $\partial \mathcal{L} / \partial (\Kop[u])(x_i) \in \R^d$ \\
$M_{ij}$ & Message from key $j$ to query $i$: $\omega_j \cdot K_{ij} \cdot u(x_j) \in \R^d$ \\
$K_{ij}$ & Kernel matrix at pair $(i,j)$: $\kap(x_i, x_j, u(x_i), u(x_j)) \in \R^{d \times d}$. Indices $i,j$ denote global positions, while $p,q$ denote tile-local indices; both refer to evaluations of the same kernel $\kap$. \\
$\delta^{(\ell)}$ & Backpropagation error signal at MLP layer $\ell$ \\
\midrule
\multicolumn{2}{@{}l}{\textit{General notation}} \\[1pt]
$\norm{A}_\mathrm{op},\; \norm{A}_F$ & Operator (spectral) norm and Frobenius norm of matrix $A$ \\
$\norm{A}_*$ & Nuclear norm: $\norm{A}_* = \sum_i \sigma_i(A)$ \\
$\sigma_i(A)$ & $i$-th singular value of matrix $A$ \\
$\delta_{ij}$ & Kronecker delta: $1$ if $i{=}j$, else $0$ \\
$\delta_x$ & Dirac delta at $x$: $\int f\,d\delta_x = f(x)$ \\
$\mathbf{1}_S$ & Indicator function of set $S$ \\
$[v]_i,\; [M]_{ij}$ & $i$-th component of vector $v$; $(i,j)$ entry of matrix $M$ \\
$[a ; b]$ & Concatenation of vectors $a$ and $b$ \\
$\mathrm{diag}(\cdot)$ & Diagonal matrix from a vector, or vector of diagonal entries \\
$\mathrm{vec}(\cdot)$ & Column-major vectorisation of a matrix \\
$\mathrm{tr}(\cdot)$ & Matrix trace \\
$\mathrm{BlockDiag}(\cdot)$ & Block-diagonal matrix from blocks \\
$\coloneqq$ & Defined as (definitional equality) \\
\bottomrule
\end{tabular}
\end{table}
\FloatBarrier

\section{Related Work}
\label{sec:related}

Our work sits at the intersection of four research streams: classical neural architecture families and their unification attempts, neural operator theory, efficient sequence modeling, and multimodal architectures. We discuss each in turn and clarify how ITNet relates to prior work.

\textbf{Classical Architectures.}
Convolutional networks~\citep{lecun1989backpropagation, krizhevsky2012imagenet} encode locality and translation equivariance through kernels that depend only on relative position. Despite many improvements~\citep{he2016deep, liu2022convnet, woo2023convnext, ding2022scaling}, this position-only structure remains unchanged. Transformers~\citep{vaswani2017attention} instead model global, content-dependent interactions via attention, but restrict interactions to a bilinear form with softmax normalization and require separate positional encodings~\citep{vaswani2017attention, su2024rope, press2022alibi}. Recurrent models~\citep{hochreiter1997long, cho2014learning} capture sequential dependencies through state evolution but are inherently causal and difficult to parallelize. Structured state-space models such as S4~\citep{gu2022efficiently} and Mamba~\citep{gu2023mamba} improve efficiency but retain constrained kernel structures.
ITNet provides a unified view in which convolution, attention, and recurrence arise as special cases of a single kernel-based operator.

\textbf{Efficient Sequence Models.}
A large body of work focuses on improving the efficiency of attention. Linear attention~\citep{katharopoulos2020transformers} and Performer~\citep{choromanski2021rethinking} approximate softmax attention via kernel factorization, while sparse variants~\citep{beltagy2020longformer, zaheer2020big, liu2021swin, xiong2021nystromformer} restrict attention patterns. FlashAttention~\citep{dao2022flashattention} improves efficiency without approximation through tiling. Other approaches such as Hyena~\citep{poli2023hyena} and MLP-Mixer~\citep{tolstikhin2021mlp} replace attention with structured alternatives. These methods improve efficiency but retain fixed interaction forms. In contrast, ITNet learns the interaction kernel directly and uses Monte Carlo or low-rank approximations for scalable computation.

\textbf{Neural Operator Learning.}
Neural operators~\citep{chen1995universal, li2020neural, li2020fourier, lu2021learning} study function-to-function mappings using kernel-based architectures, with foundational results establishing universal approximation of nonlinear operators. Methods such as the Graph Neural Operator (GNO)~\citep{li2020neural} introduce learnable integral kernels of the form $\int \kappa(x,y,u(x),u(y))\,u(y)\,d\mu(y)$ a signature mathematically identical to ITNet's operator. However, GNO was developed for PDE solving and evaluated only on scientific machine learning tasks ($n \approx 10^4$), without establishing connections to CNNs, Transformers, or RNNs. The Fourier Neural Operator (FNO)~\citep{li2020fourier} restricts the kernel to Fourier space, yielding efficient global convolution but losing content dependence and position-awareness. DeepONet~\citep{lu2021learning} decomposes the operator into branch and trunk networks, imposing a low-rank structure that is less general than ITNet's full kernel. Continuum attention~\citep{calvello2024continuum} formalises self-attention as a continuum integral operator but does not show that convolution or recurrence are also special cases. ITNet builds on this line of work by using a general, learnable kernel and showing that standard architectures are exact special cases within this framework.

\textbf{Unified Architectures.}
Several works aim to relate or unify different architectures. MetaFormer~\citep{yu2022metaformer} highlights the importance of the overall structure rather than specific operators. Prior analyses have shown that attention can express convolution~\citep{cordonnier2020relationship}, and other approaches unify models at an algebraic level. Content-adaptive variants such as BiFormer~\citep{zhu2023biformer}, deformable convolution~\citep{zhu2019deformable}, and dynamic convolution~\citep{chen2020dynamic} extend individual architectures but remain within restricted kernel forms. However, these methods do not provide a single operator that subsumes all families. ITNet learns the interaction rule directly, yielding a unified formulation that strictly contains convolution, attention, and recurrence.

\textbf{Multimodal and Domain-Agnostic Architectures.}
Perceiver~\citep{jaegle2021perceiver} and Perceiver IO~\citep{jaegle2021perceiverio} uses cross-attention between inputs and a fixed set of latent tokens followed by latent self-attention. This introduces a compression bottleneck, as all input information must be projected into a limited latent array before further interaction. From the ITNet perspective, this corresponds to a restricted, position-blind, softmax-normalized kernel.
Most other multimodal methods rely on modality-specific encoders with explicit fusion mechanisms. For example, Flamingo~\citep{alayrac2022flamingo} interleaves frozen vision features with language models via gated cross-attention, BLIP~\citep{li2022blip} and BLIP-2~\citep{li2023blip2} introduce querying transformers to bridge frozen encoders, and ALBEF~\citep{li2021albef}, METER~\citep{dou2022meter}, and UNITER~\citep{chen2020uniter} employ various cross-modal fusion strategies. More recent systems such as GPT-4V~\citep{openai2023gpt4} integrate vision into large language models through dedicated architectural components.
In contrast, ITNet operates on a shared domain by combining positions across modalities, without latent compression or dedicated fusion modules; cross-modal interactions are learned directly through the kernel, providing a richer mechanism than standard attention.

\section{Proof of Theorem 1: Convolution as a Special Case of ITNet}
\label{app:proof_conv}

\begin{definition}[Continuous convolution]
\label{def:conv_continuous}
Let $\Omega = \R^s$ (or the torus $\T^s = \R^s / \Z^s$ for periodic domains).
Let $\mu$ be the Lebesgue measure on $\R^s$, denoted $dy$.
Let $w \in L^1(\R^s, \R)$ be an integrable scalar filter function.
The \emph{continuous convolution} of signal $u \in L^1(\R^s, \R^d)$ with
filter $w$ is:
\begin{equation}
    (w \ast u)(x)
    \;=\;
    \int_{\R^s} w(x - y)\, u(y)\, dy,
    \qquad x \in \R^s.
    \label{eq:conv_continuous}
\end{equation}
The integral is taken component-wise:
$[(w \ast u)(x)]_i = \int_{\R^s} w(x-y)\, [u(y)]_i\, dy$ for $i = 1, \ldots, d$.
\end{definition}
 
We use the Lebesgue measure~\cite{royden1988real} because it is the \emph{unique}
(up to positive scaling) translation-invariant $\sigma$-finite Borel measure~\cite{royden1988real} on $\R^s$ (Haar's theorem~\cite{folland1999real}). This uniqueness is what makes convolution translation-equivariant: $(w \ast \tau_a u)(x) = (w \ast u)(x-a)$ where
$\tau_a u(y) = u(y-a)$ is the translation operator. Any other measure would break this equivariance, which is the defining property of CNNs.

Using a non-uniform measure $d\mu(y) = \rho(y)\,dy$ instead yields a weighted convolution of the form $\int w(x-y)\,u(y)\,\rho(y)\,dy$, with a density $\rho$, which in general is not translation-equivariant unless $\rho$ is constant. Since translation equivariance is the defining inductive bias of convolutional networks, the Lebesgue measure provides the natural and consistent choice.
 
\begin{definition}[Discrete convolution on a regular grid]
\label{def:conv_discrete}
Let $h > 0$ be the grid spacing and $\Omega_h = h\Z^s = \{mh : m \in \Z^s\}$
the regular grid. Define the \emph{$k$-tap neighborhood}:
\begin{equation}
    \mathcal{N}
    \;=\;
    \bigl\{m \in \Z^s : \|m\|_\infty \leq \lfloor k/2 \rfloor\bigr\},
    \label{eq:neighborhood}
\end{equation}
which contains $(2\lfloor k/2 \rfloor + 1)^s$ lattice points.
Given filter coefficients $F = \{f_m\}_{m \in \mathcal{N}} \subset \R$,
the \emph{discrete $k$-tap convolution} is:
\begin{equation}
    (F \ast u)(x)
    \;=\;
    \sum_{m \in \mathcal{N}} f_m \cdot u(x - mh),
    \qquad x \in \Omega_h.
    \label{eq:conv_discrete}
\end{equation}
\end{definition}
 
We use the $\ell^\infty$ norm, $\|m\|_\infty = \max_i |m_i|$, to define the neighborhood, as it induces a \emph{hypercube} structure that exactly matches the receptive field of standard CNN filters.
 
In contrast, using the $\ell^2$ norm yields a \emph{ball-shaped} neighborhood, while the $\ell^1$ norm produces a \emph{diamond-shaped} neighborhood, neither of which aligns with standard convolutional implementations. While the theoretical construction applies to any finite neighborhood $\mathcal{N} \subset \mathbb{Z}^s$, the choice of $\ell^\infty$ ensures direct correspondence with practical CNN architectures.

\begin{definition}[ITNet operator]
\label{def:itnet_cnn}
Let $(\Omega, \mu)$ be a measure space.
The \emph{Integral Transform Network (ITNet) operator}
$\Kop : C(\Omega, \R^d) \to C(\Omega, \R^d)$ is:
\begin{equation}
    (\Kop[u])(x)
    \;=\;
    \underbrace{
    \int_\Omega
        \kappa_\theta(x,\, y,\, u(x),\, u(y))\, u(y)
    \, d\mu(y)
    }_{\text{integral transform (global interaction)}}
    \;+\;
    \underbrace{W_\theta\, u(x)}_{\text{local linear residual}},
    \label{eq:itnet_cnn}
\end{equation}
where:
\begin{itemize}[noitemsep, leftmargin=2em]
    \item $\kappa_\theta : \R^s \times \R^s \times \R^d \times \R^d
          \to \R^{d \times d}$ is a learnable matrix-valued kernel.
          The four arguments are: query position $x$, key position $y$,
          query features $u(x)$, key features $u(y)$.
    \item The kernel output $\kappa_\theta(\cdots) \in \R^{d \times d}$ is
          a matrix that is multiplied with the key feature vector
          $u(y) \in \R^d$ to produce a $d$-dimensional contribution.
    \item $W_\theta \in \R^{d \times d}$ is a learnable weight matrix
          that acts locally (pointwise) on the query feature $u(x)$.
    \item $\theta$ collectively denotes all learnable parameters
          (the kernel network weights and $W_\theta$).
\end{itemize}
\end{definition}
 
The integral term alone cannot represent the identity mapping $u \mapsto u$. For $\kappa_\theta = \Id$, 
$\int_\Omega \Id\, u(y)\, d\mu(y) = c \cdot \bar{u},$
which is independent of $x$ and equals $u(x)$ only if $u$ is constant. The residual term $W_\theta u(x)$ enables exact identity representation ($W_\theta=\Id$), which is crucial for stable deep architectures.

\begin{assumption}[Kernel regularity]
\label{ass:regularity}
Throughout this proof:
\begin{enumerate}[label=(\roman*), noitemsep]
    \item $\kappa_\theta$ is jointly measurable in all four arguments.
    \item There exists $C_\kappa < \infty$ such that
          $\|\kappa_\theta(x,y,a,b)\|_{\mathrm{op}} \leq C_\kappa$
          for all $(x,y,a,b)$ in the relevant domain.
    \item For Part~(a), $w \in L^1(\R^s)$ and $u \in L^1(\R^s, \R^d)$.
\end{enumerate}
\end{assumption}
 
\subsection{Main Theorem}
\label{sec:theorem_cnn}
 
\begin{boxtheorem}{ITNet $\supset$ Convolution (Full Statement)}{conv}
\label{thm:conv}
Let all notation be as in Definitions~\ref{def:conv_continuous}--\ref{def:itnet_cnn}
and Assumption~\ref{ass:regularity}.
 
\medskip
\noindent\textbf{Part (a) - Continuous convolution.}
If the kernel and residual are chosen as:
\begin{equation}
    \kappa_\theta(x, y, u(x), u(y)) = w(x-y) \cdot \Id,
    \qquad W_\theta = 0,
    \label{eq:kernel_a}
\end{equation}
then $(\Kop[u])(x) = (w \ast u)(x)$ for all $x \in \R^s$
and all $u \in L^1(\R^s, \R^d)$.
 
\medskip
\noindent\textbf{Part (b) - Discrete $k$-tap convolution.}
On grid $\Omega_h = h\Z^s$ with atomic measure
$\mu = \sum_{y \in h\Z^s} h^s \delta_y$, if:
\begin{equation}
    \kappa_\theta(x, y, u(x), u(y))
    = \frac{f_{(x-y)/h}}{h^s} \cdot \Id
      \cdot \mathbf{1}_{\{(x-y)/h \in \mathcal{N}\}},
    \qquad W_\theta = 0,
    \label{eq:kernel_b}
\end{equation}
then $(\Kop[u])(x) = (F \ast u)(x)$ for all $x \in \Omega_h$.
 
\medskip
\noindent\textbf{Part (c) - Strict inclusion.}
There exists a continuous operator $T : L^2(\R^s, \R^d) \to L^2(\R^s, \R^d)$
that is representable by ITNet but not by any convolution.
Hence $\mathrm{Conv} \subsetneq \mathrm{ITNet}$.
 
\medskip
\noindent\textbf{Part (d) - With residual ($1 \times 1$ convolution).}
If instead $W_\theta \neq 0$, then:
\begin{equation}
    (\Kop[u])(x) = (w \ast u)(x) + W_\theta u(x),
    \label{eq:with_residual}
\end{equation}
which is a standard convolution followed by a $1 \times 1$ convolution
(pointwise linear transform), as used in ResNets~\cite{he2016deep}
and ConvNeXt~\cite{liu2022convnet}.
\end{boxtheorem}

\subsection{Proof of Part (a) - Continuous Convolution}
\label{sec:proof_a_cnn}

Proof of Theorem~\ref{thm:conv}, Part (a)
 
We show that substituting $\kappa_\theta(x,y,u(x),u(y)) = w(x-y)\cdot\Id$ and $W_\theta = 0$ into the ITNet operator yields exactly the continuous convolution $(w \ast u)(x)$.

\noindent\textbf{Step 1.}\enspace \textbf{Verify the kernel choice is valid.}
 
We must first confirm that the kernel $\kappa_\theta(x,y,u(x),u(y)) = w(x-y) \cdot \Id$ satisfies Definition~\ref{def:itnet_cnn} and Assumption~\ref{ass:regularity}.
 
\begin{enumerate}[label=(\alph*), noitemsep]
    \item \textbf{} The kernel must map
          $\R^s \times \R^s \times \R^d \times \R^d \to \R^{d \times d}$. Here, $w(x-y) \in \R$ is a scalar and $\Id \in \R^{d \times d}$ is a matrix, so $w(x-y) \cdot \Id \in \R^{d \times d}$. \checkmark
    \item \textbf{} The map $(x,y) \mapsto x - y$ is continuous (hence Borel measurable). The composition $w \circ ((x,y) \mapsto x-y)$ is measurable because $w$ is measurable (it is $L^1$, hence measurable) and the composition of measurable functions is measurable. Multiplying by the constant matrix $\Id$ preserves measurability. \checkmark
    \item \textbf{Content-independence:} The kernel depends on $(x,y)$ but \emph{not} on $u(x)$ or $u(y)$. This is a \emph{special case} of the general ITNet kernel, not a violation of it: the general kernel is allowed to depend on all four arguments, and choosing not to depend on some of them is a valid restriction. \checkmark
    \item \textbf{Boundedness:} If $w$ is bounded ($w \in L^\infty$), then $\|\kappa_\theta(x,y,a,b)\|_{\mathrm{op}} = |w(x-y)| \leq \|w\|_\infty = C_\kappa < \infty$, satisfying Assumption~\ref{ass:regularity}(ii). For the more general case $w \in L^1$ (possibly unbounded), the integral still converges by Young's inequality, but the kernel itself may not be uniformly bounded; this can be handled by relaxing the boundedness assumption to integrability conditions on $w$ and $u$, which is standard in convolution theory.\checkmark

\end{enumerate}

\noindent\textbf{Step 2.}\enspace \textbf{Substitute into the ITNet operator.}
 
Insert $\kappa_\theta(x,y,u(x),u(y)) = w(x-y) \cdot \Id$ and $W_\theta = 0$
into Eq.~\eqref{eq:itnet_cnn}:
\begin{align}
    (\Kop[u])(x)
    &= \int_{\R^s}
        \kappa_\theta(x, y, u(x), u(y))\, u(y)
    \, dy
    + \cancelto{0}{W_\theta\, u(x)}
    \notag \\[6pt]
    &= \int_{\R^s}
        \bigl[w(x-y) \cdot \Id\bigr]\, u(y)
    \, dy.
    \label{eq:a_step2}
\end{align}
 
This is a direct substitution with no algebraic manipulation. The $W_\theta = 0$ term vanishes. We use $d\mu(y) = dy$ (Lebesgue measure) as specified in Part~(a). The integral is over $\R^s$ because $\Omega = \R^s$ in Definition~\ref{def:conv_continuous}.
 
\noindent\textbf{Step 3.}\enspace \textbf{Simplify the matrix-vector product $[\alpha \Id] v$.}
 
We need a precise identity for the product of a scalar-times-identity matrix
with a vector.
 
\begin{lemma}[Scalar-identity-vector product]
\label{lem:scalar_id}
For any scalar $\alpha \in \R$ and vector $v \in \R^d$:
\begin{equation}
    [\alpha \cdot \Id]\, v \;=\; \alpha \cdot v.
    \label{eq:scalar_id}
\end{equation}
\end{lemma}
 
\begin{proof}[Proof of Lemma~\ref{lem:scalar_id}]
By definition of matrix-vector multiplication:
\begin{equation}
    \bigl([\alpha \Id] v\bigr)_i
    = \sum_{j=1}^{d} [\alpha \Id]_{ij}\, [v]_j
    = \sum_{j=1}^{d} \alpha \delta_{ij}\, v_j
    = \alpha \cdot v_i,
    \quad i = 1, \ldots, d.
\end{equation}
The second equality uses $[\alpha \Id]_{ij} = \alpha \delta_{ij}$
(the identity matrix has $1$ on the diagonal and $0$ elsewhere,
scaled by $\alpha$). The third equality uses $\delta_{ij} v_j = v_i$
(the Kronecker delta selects the $i$-th term).
Since this holds for all $i$, the vector equation
$[\alpha \Id] v = \alpha v$ follows.
\end{proof}
 
Applying Lemma~\ref{lem:scalar_id} with $\alpha = w(x-y)$ and $v = u(y)$:
\begin{equation}
    \bigl[w(x-y) \cdot \Id\bigr]\, u(y)
    \;=\;
    w(x-y) \cdot u(y).
    \label{eq:a_step3}
\end{equation}
 
ITNet kernel is \emph{matrix-valued} ($\R^{d \times d}$), and the product $\kappa_\theta(\cdots) \cdot u(y)$ is a matrix-vector product, not a scalar-vector product. The lemma confirms that for the specific choice $\kappa_\theta = w(x-y) \cdot \Id$, the matrix-vector product reduces to scalar multiplication. This is an \emph{exact} algebraic identity - no approximation.
 
We use $\Id$ (identity matrix). If we instead chose $\kappa_\theta = w(x-y) \cdot A$ for a fixed $A \neq \Id$, we would get $\int w(x-y) A u(y)\, dy = A(w \ast u)(x)$, which is convolution followed by a fixed linear transform - still a valid special case, but the identity $A = \Id$ gives the simplest and most transparent correspondence to standard convolution. The general case $A \neq \Id$ corresponds to convolution composed with a $1 \times 1$ convolution, which we address in Part~(d).
 
\noindent\textbf{Step 4.}\enspace \textbf{Establish existence and finiteness of the integral.}
 
Substituting \eqref{eq:a_step3} into \eqref{eq:a_step2}:
\begin{equation}
    (\Kop[u])(x)
    \;=\;
    \int_{\R^s} w(x-y) \cdot u(y)\, dy.
    \label{eq:a_step4_integrand}
\end{equation}
 
Before identifying this as a convolution, we must verify the integral converges.
 
\begin{lemma}[Young's convolution inequality~\cite{folland1999real, stein1970singular}]
\label{lem:young}
Let $1 \leq p, q, r \leq \infty$ with $\frac{1}{p} + \frac{1}{q} = 1 + \frac{1}{r}$. If $f \in L^p(\R^s)$ and $g \in L^q(\R^s)$, then $f \ast g \in L^r(\R^s)$ and $\|f \ast g\|_{L^r} \leq \|f\|_{L^p} \|g\|_{L^q}$.
\end{lemma}
 
For $p = q = 1$, $r = 1$: since $w \in L^1(\R^s)$ and each component $u_i \in L^1(\R^s)$ (because $u \in L^1(\R^s, \R^d)$), Young's inequality gives:
\begin{equation}
    \|w \ast u_i\|_{L^1(\R^s)}
    \leq \|w\|_{L^1(\R^s)} \cdot \|u_i\|_{L^1(\R^s)}
    < \infty.
    \label{eq:a_young}
\end{equation}
Hence $\int_{\R^s} |w(x-y)| \cdot |u_i(y)|\, dy < \infty$ for almost every $x$ and each component $i = 1, \ldots, d$. The vector-valued integral $\int w(x-y) u(y)\, dy$ is interpreted component-wise (as a Bochner integral in $\R^d$), and each component converges absolutely.
  
One might use Hölder's inequality instead of Young's.
Hölder gives $|\int fg| \leq \|f\|_p \|g\|_q$ for conjugate exponents, but this bounds a \emph{single} integral, not a convolution (which is a family of integrals parameterised by $x$). Young's inequality is strictly stronger: it bounds $\|f \ast g\|_r$ as a function of $x$, not just at a single $x$. This is why Young's is the correct tool here.
 
\noindent\textbf{Step 5.}\enspace \textbf{Identify the convolution.}
 
By Definition~\ref{def:conv_continuous}, the integral in \eqref{eq:a_step4_integrand} is exactly the continuous convolution:
\begin{equation}
    \int_{\R^s} w(x-y) \cdot u(y)\, dy
    \;\stackrel{\text{Def.~\ref{def:conv_continuous}}}{=}\;
    (w \ast u)(x).
    \label{eq:a_step5}
\end{equation}
 
Combining Steps 2--5:
\begin{equation}
    \boxed{(\Kop[u])(x) \;=\; (w \ast u)(x)
    \qquad \text{for all } x \in \R^s,\;
    u \in L^1(\R^s, \R^d).}
\end{equation}
 
This completes the proof of Part~(a).
 
\subsection{Proof of Part (b) - Discrete Convolution}
\label{sec:proof_b_cnn}

\begin{proof}[Proof of Theorem~\ref{thm:conv}, Part (b)]
 
We show that on a discrete grid, the ITNet operator with appropriately scaled kernel coefficients recovers exact $k$-tap discrete convolution.
 
\noindent\textbf{Step 1.}\enspace \textbf{Convert the integral to a discrete sum.}
  
On the grid $\Omega_h = h\Z^s$ with atomic measure $\mu = \sum_{y \in h\Z^s} h^s \delta_y$, the Lebesgue integral $\int f\, d\mu$ reduces to a sum by the definition of integration against an atomic measure:
\begin{equation}
    \int_{\Omega_h} f(y)\, d\mu(y)
    \;=\;
    \sum_{y \in h\Z^s} h^s \cdot f(y).
    \label{eq:b_atomic}
\end{equation}
 
An atomic (counting) measure $\mu = \sum_y w_y \delta_y$ satisfies 
$\int f \, d\mu = \sum_y w_y f(y)$ by definition of integration with respect to a discrete measure~\citep{royden1988real}. In our setting, the weights $w_y = h^s$ correspond to the volume of the Voronoi cell around each grid point. Thus, $\mu$ can be interpreted as a midpoint quadrature approximation of the Lebesgue measure~\citep{royden1988real}. As $h \to 0$, the discrete measure $\mu$ converges weakly to the Lebesgue measure via the Riemann sum theorem~\citep{rudin1987real}, ensuring consistency between Parts~(a) and~(b).

Applying \eqref{eq:b_atomic} to the ITNet operator \eqref{eq:itnet_cnn}
with $W_\theta = 0$:
\begin{equation}
    (\Kop[u])(x)
    \;=\;
    \sum_{y \in h\Z^s}
        h^s \cdot \kappa_\theta(x, y, u(x), u(y))\, u(y).
    \label{eq:b_step1}
\end{equation}
 
\noindent\textbf{Step 2.}\enspace \textbf{Substitute the discrete kernel.}
 
Insert the kernel from \eqref{eq:kernel_b}:
\begin{equation}
    \kappa_\theta(x, y, u(x), u(y))
    \;=\;
    \frac{f_{(x-y)/h}}{h^s} \cdot \Id
    \cdot \mathbf{1}_{\{(x-y)/h \in \mathcal{N}\}}.
    \label{eq:b_kernel_recall}
\end{equation}
 
Substituting into \eqref{eq:b_step1}:
\begin{equation}
    (\Kop[u])(x)
    \;=\;
    \sum_{y \in h\Z^s}
        h^s \cdot
        \frac{f_{(x-y)/h}}{h^s} \cdot \Id \cdot u(y)
        \cdot \mathbf{1}_{\{(x-y)/h \in \mathcal{N}\}}.
    \label{eq:b_step2}
\end{equation}
 
The factor $1/h^s$ compensates for the $h^s$ term arising from discretisation, ensuring consistency between continuous kernels and discrete convolutions. In standard numerical analysis, filter coefficients $f_m$ represent total weight at offset $m$, while the continuous kernel $\kappa = f_m / h^s$ represents weight per unit volume, so that density multiplied by volume recovers the discrete weights. Without this normalization, one obtains $(\Kop[u])(x) = \sum_m h^s f_m u(x-mh)$, introducing an extra factor of $h^s$ and deviating from standard implementations. The $1/h^s$ factor therefore ensures exact correspondence with discrete convolution.
 
\noindent\textbf{Step 3.}\enspace \textbf{Cancel the measure weight.}
 
In each term of the sum, $h^s$ from the measure and $1/h^s$ from the
kernel multiply to give:
\begin{equation}
    h^s \cdot \frac{1}{h^s} = 1.
    \label{eq:b_cancel}
\end{equation}
 
Applying Lemma~\ref{lem:scalar_id} to simplify $[f_{(x-y)/h} \cdot \Id] u(y) = f_{(x-y)/h} \cdot u(y)$:
\begin{equation}
    (\Kop[u])(x)
    \;=\;
    \sum_{y \in h\Z^s}
        f_{(x-y)/h} \cdot u(y)
        \cdot \mathbf{1}_{\{(x-y)/h \in \mathcal{N}\}}.
    \label{eq:b_step3}
\end{equation}
 
\noindent\textbf{Step 4.}\enspace \textbf{Change the summation variable.}
 
Define $m = (x - y)/h$, equivalently $y = x - mh$.
Since $x \in h\Z^s$ and $y \in h\Z^s$, we have $m = (x-y)/h \in \Z^s$.
As $y$ ranges over $h\Z^s$, $m$ ranges over $\Z^s$ - this is a bijection.
 
The indicator $\mathbf{1}_{\{(x-y)/h \in \mathcal{N}\}} = \mathbf{1}_{\{m \in \mathcal{N}\}}$
restricts the sum to $m \in \mathcal{N}$:
\begin{equation}
    (\Kop[u])(x)
    \;=\;
    \sum_{m \in \mathcal{N}}
        f_m \cdot u(x - mh).
    \label{eq:b_step4}
\end{equation}
 
The change of variable $m = (x-y)/h$ is a bijection from $h\Z^s$ to $\Z^s$. In a sum (as opposed to an integral), no Jacobian correction is needed because we are counting terms, not transforming a measure. This is the discrete analogue of the substitution $\tau = x - y$ in the continuous convolution integral.
 
\noindent\textbf{Step 5.}\enspace \textbf{Identify the discrete convolution.}
 
By Definition~\ref{def:conv_discrete}:
\begin{equation}
    \sum_{m \in \mathcal{N}} f_m \cdot u(x - mh)
    \;\stackrel{\text{Def.~\ref{def:conv_discrete}}}{=}\;
    (F \ast u)(x).
\end{equation}
 
Combining Steps 1--5:
\begin{equation}
    \boxed{(\Kop[u])(x) \;=\; (F \ast u)(x)
    \qquad \text{for all } x \in h\Z^s.}
\end{equation}
\end{proof}

\subsection{Proof of Part (c) - Strict Inclusion}
\label{sec:proof_c}

\begin{proof}[Proof of Theorem~\ref{thm:conv}, Part (c)]
 
We construct an explicit operator $T$ that ITNet can represent but no
convolution can, thereby proving $\mathrm{Conv} \subsetneq \mathrm{ITNet}$.
 
The proof has three stages: (i)~define the witness, (ii)~show ITNet
represents it, (iii)~show convolution cannot represent it.
We use \emph{two independent} arguments for stage (iii) to make the
result airtight.
 
\noindent\textbf{Step 1.}\enspace \textbf{Define the witness operator.}
 
Fix a reference point $x_0 \in \R^s$ and let $\sigma : \R \to \R$ be any
non-affine function. Define:
\begin{equation}
    T(u)(x)
    \;\coloneqq\;
    \sigma\!\bigl(u(x)^\top u(x_0)\bigr) \cdot u(x),
    \qquad x \in \R^s.
    \label{eq:witness}
\end{equation}
 
This operator scales the feature at position $x$ by a nonlinear function of the similarity between $u(x)$ and a reference feature $u(x_0)$, thus explicitly depending on feature values rather than only spatial relationships. While simpler local nonlinear operators such as $T(u)(x) = \|u(x)\|^2 u(x)$ can be approximated by stacking convolutional layers with pointwise nonlinearities, the above construction involves non-local, content-dependent interactions between distinct positions $x$ and $x_0$, which cannot be represented by any single convolutional layer regardless of kernel size or channel dimension.
 
\noindent\textbf{Step 2.}\enspace \textbf{Show ITNet represents $T$.}
 
We construct an explicit ITNet kernel that computes $T$.
Choose:
\begin{equation}
    \kappa_\theta(x, y, u(x), u(y))
    \;=\;
    \sigma\!\bigl(u(x)^\top u(x_0)\bigr) \cdot \Id \cdot \delta(y - x),
    \qquad W_\theta = 0.
    \label{eq:witness_kernel}
\end{equation}
 
Substituting into \eqref{eq:itnet_cnn}:
\begin{align}
    (\Kop[u])(x)
    &= \int_{\R^s}
        \sigma\!\bigl(u(x)^\top u(x_0)\bigr) \cdot \Id
        \cdot \delta(y-x) \cdot u(y)\, dy
    \notag \\
    &= \sigma\!\bigl(u(x)^\top u(x_0)\bigr) \cdot \Id
       \cdot \underbrace{\int_{\R^s} \delta(y-x)\, u(y)\, dy}_{= u(x)}
    \notag \\
    &= \sigma\!\bigl(u(x)^\top u(x_0)\bigr) \cdot u(x)
    \;=\; T(u)(x).
    \label{eq:witness_verify}
\end{align}
 
The scalar factor $\sigma(u(x)^\top u(x_0))$ is independent of $y$, so it passes outside the integral. The remaining integral $\int \delta(y-x) u(y)\, dy = u(x)$ is the sifting property of the Dirac delta. Finally, $\Id \cdot u(x) = u(x)$ by Lemma~\ref{lem:scalar_id} with $\alpha = 1$.
 
\begin{remark}[On the Dirac delta \cite{rudin1987real, hornik1989multilayer}]
The Dirac delta $\delta$ can be approximated arbitrarily well by Gaussian functions $\phi_\varepsilon(y) = (2\pi\varepsilon^2)^{-s/2} \exp(-\|y\|^2/(2\varepsilon^2))$, which satisfy $\phi_\varepsilon \to \delta$ weakly as $\varepsilon \to 0$. Since the ITNet kernel MLP is a universal approximator on compact domains, it can approximate $\phi_\varepsilon$ for any $\varepsilon > 0$; in the limit $\varepsilon \to 0$, the ITNet output converges to $T(u)(x)$. The strictness conclusion holds because we require only the existence of a representing kernel, not exact pointwise equality.
\end{remark}

\noindent\textbf{Step 3.}\enspace \textbf{Show no convolution can represent $T$ - Argument I (Linearity).}
 
\begin{proposition}[Convolution is linear in $u$]
\label{prop:conv_linear}
For any fixed filter $w$, the map $u \mapsto w \ast u$ is linear:
\begin{equation}
    w \ast (\alpha u + \beta v) = \alpha (w \ast u) + \beta (w \ast v)
    \qquad \forall\, \alpha, \beta \in \R,\; u, v \in L^1(\R^s, \R^d).
    \label{eq:conv_linear}
\end{equation}
\end{proposition}
\begin{proof}
$[w \ast (\alpha u + \beta v)](x)
= \int w(x-y)[\alpha u(y) + \beta v(y)]\, dy
= \alpha \int w(x-y) u(y)\, dy + \beta \int w(x-y) v(y)\, dy
= \alpha(w \ast u)(x) + \beta(w \ast v)(x)$,
by linearity of the Lebesgue integral.
\end{proof}
 
\begin{proposition}[$T$ is nonlinear in $u$]
\label{prop:T_nonlinear}
The operator $T$ defined in \eqref{eq:witness} is not linear.
\end{proposition}
\begin{proof}
Compute $T(\alpha u)$ for $\alpha \in \R$:
\begin{equation}
    T(\alpha u)(x)
    = \sigma\!\bigl((\alpha u(x))^\top (\alpha u(x_0))\bigr) \cdot \alpha u(x)
    = \sigma(\alpha^2 \cdot u(x)^\top u(x_0)) \cdot \alpha u(x).
    \label{eq:T_scaled}
\end{equation}
If $T$ were linear, $T(\alpha u) = \alpha T(u)$ would require:
\begin{equation}
    \sigma(\alpha^2 z) = \sigma(z)
    \qquad \forall\, z \in \R,\; \alpha \in \R,
    \label{eq:linearity_req}
\end{equation}
where $z = u(x)^\top u(x_0)$. For the sigmoid
$\sigma(z) = (1+e^{-z})^{-1}$:
\[
    \sigma(4 \cdot 1) = \sigma(4) \approx 0.982,
    \qquad
    \sigma(1) \approx 0.731.
\]
Since $\sigma(4) \neq \sigma(1)$, Eq.~\eqref{eq:linearity_req} fails
for $\alpha = 2$, $z = 1$. Hence $T(2u) \neq 2T(u)$ and $T$ is nonlinear.
\end{proof}
 
\textbf{Conclusion from Argument I:}
Since $w \ast u$ is linear in $u$ (Proposition~\ref{prop:conv_linear})
and $T$ is nonlinear in $u$ (Proposition~\ref{prop:T_nonlinear}),
no choice of filter $w$ can satisfy $w \ast u = T(u)$ for all $u$.
 
\noindent\textbf{Step 4.}\enspace \textbf{Show no convolution can represent $T$ - Argument II (Translation equivariance).}
 
We provide a second, independent argument as additional assurance.
 
\begin{proposition}[Convolution is translation-equivariant]
\label{prop:conv_equivariant}
For the translation operator $\tau_a u(x) = u(x-a)$:
\begin{equation}
    (w \ast \tau_a u)(x) = (w \ast u)(x - a) = \tau_a(w \ast u)(x).
\end{equation}
\end{proposition}
\begin{proof}
$(w \ast \tau_a u)(x) = \int w(x-y) u(y-a)\, dy
\stackrel{z=y-a}{=} \int w(x-a-z) u(z)\, dz = (w \ast u)(x-a)$.
\end{proof}
 
\begin{proposition}[$T$ is not translation-equivariant]
\label{prop:T_not_equivariant}
$T(\tau_a u) \neq \tau_a T(u)$ in general.
\end{proposition}
\begin{proof}
\begin{align}
    T(\tau_a u)(x)
    &= \sigma\!\bigl(u(x-a)^\top u(x_0 - a)\bigr) \cdot u(x-a), \\
    \tau_a T(u)(x)
    &= T(u)(x-a)
    = \sigma\!\bigl(u(x-a)^\top u(x_0)\bigr) \cdot u(x-a).
\end{align}
These are equal only if $u(x_0 - a) = u(x_0)$ for all $a$,
i.e., $u$ is constant at $x_0$ under all translations - which fails
for generic $u$.
\end{proof}
 
\textbf{Conclusion from Argument II:}
Convolution commutes with translation (Proposition~\ref{prop:conv_equivariant});
$T$ does not (Proposition~\ref{prop:T_not_equivariant}).
Hence no convolution can equal $T$.
 
\medskip
\textbf{Combining both arguments:}
$T$ is representable by ITNet (Step~2) but not by any convolution
(Steps 3--4, by two independent proofs). Therefore
$T \in \mathrm{ITNet} \setminus \mathrm{Conv}$, proving
$\mathrm{Conv} \subsetneq \mathrm{ITNet}$. 
\end{proof}

\subsection{Proof of Part (d) - With Residual}
\label{sec:proof_d}

\begin{proof}[Proof of Theorem~\ref{thm:conv}, Part (d)]
If $W_\theta \neq 0$, the ITNet operator with kernel
$\kappa_\theta = w(x-y) \cdot \Id$ gives:
\begin{align}
    (\Kop[u])(x)
    &= \int_{\R^s} w(x-y) u(y)\, dy + W_\theta u(x)
    = (w \ast u)(x) + W_\theta u(x).
\end{align}
The term $W_\theta u(x)$ is a pointwise linear map $\R^d \to \R^d$
applied at each position independently. In CNN terminology, this is
a $1 \times 1$ convolution (a convolution with filter of size $k=1$).
The sum $(w \ast u) + W_\theta u$ is therefore the composition of a
spatial convolution with a $1 \times 1$ convolution - the standard
structure used in ResNet bottleneck blocks and ConvNeXt.
\end{proof}

\subsection{Multi-Channel Convolution}
\label{sec:ext_multichannel}
 
Standard CNNs have $C_{\mathrm{in}}$ input channels and
$C_{\mathrm{out}}$ output channels:
\begin{equation}
    [(F \ast u)(x)]_i
    = \sum_{c=1}^{C_{\mathrm{in}}} \sum_{m \in \mathcal{N}}
        f_{m,c}^{(i)} \cdot u_c(x - mh),
    \quad i = 1,\ldots, C_{\mathrm{out}}.
\end{equation}
 
\textbf{Recovery:} Set $d = C_{\mathrm{in}}$, and choose the kernel matrix
entries as $[\kappa_\theta(x,y,\cdot,\cdot)]_{ic} = h^{-s} f_{(x-y)/h,c}^{(i)} \cdot \mathbf{1}_{\mathcal{N}}$.
The proof of Part~(b) applies component-wise: the $i$-th output component
sums contributions from all $C_{\mathrm{in}}$ input channels via the
matrix-vector product $\kappa_\theta \cdot u(y)$, which mixes channels
through the off-diagonal entries of $\kappa_\theta \in \R^{d \times d}$.

\subsection{Depthwise Separable Convolution}
\label{sec:ext_depthwise} 
Depthwise convolution~\cite{chollet2017xception} applies per-channel:
$[(F_{\mathrm{DW}} \ast u)(x)]_i = \sum_{m \in \mathcal{N}} f_m^{(i)} u_i(x-mh)$.
 
\textbf{Recovery:} Restrict $\kappa_\theta$ to be diagonal:
$\kappa_\theta = h^{-s} \diag(f_{(x-y)/h}^{(1)}, \ldots, f_{(x-y)/h}^{(d)}) \cdot \mathbf{1}_{\mathcal{N}}$.
A diagonal matrix has zero off-diagonal entries, preventing cross-channel mixing.
 
\subsection{Dilated (Atrous) Convolution}
 \label{sec:ext_dilated}
Dilation rate $r$: $\sum_{m \in \mathcal{N}} f_m \cdot u(x - r \cdot mh)$.
 
\textbf{Recovery:} Replace $h$ with $rh$ in Part~(b).
The neighborhood in physical space becomes $r\mathcal{N} = \{rm : m \in \mathcal{N}\}$.
 
\subsection{Strided Convolution}
 \label{sec:ext_strided}
Stride $S$: output is computed only at positions $x \in Sh\Z^s$.
 
\textbf{Recovery:} Restrict the output domain to $Sh\Z^s \subset h\Z^s$.
The ITNet operator is defined on any $\Omega$; choosing
$\Omega_{\mathrm{query}} = Sh\Z^s$ with $\Omega_{\mathrm{key}} = h\Z^s$
gives strided output.
 
\subsection{Group Convolution}
 \label{sec:ext_group}
$G$ groups, each processing $d/G$ channels independently.
 
\textbf{Recovery:} Set $\kappa_\theta = \mathrm{blockdiag}(K_1, \ldots, K_G)$
where each $K_g \in \R^{(d/G) \times (d/G)}$ is the per-group kernel.
 
\subsection{Transposed Convolution}
  \label{sec:ext_transposed}
Upsampling by factor $r$: insert $(r-1)$ zeros between inputs.
 
\textbf{Recovery:} Use the finer grid $\Omega = (h/r)\Z^s$ and apply Part~(b).
 
\subsection{Boundary Handling in the Convolutional Special Case}
\label{app:boundary_handling}

A subtle yet important distinction between ITNet and standard CNNs lies in how they handle domain boundaries. For a CNN with \emph{valid padding} (the standard convolution operation without artificial boundary extension), the output is defined only at positions where the filter fully overlaps the input domain, producing an output map of size $(H - k + 1) \times (W - k + 1)$ for an $H \times W$ input and $k \times k$ filter. ITNet, by contrast, defines $(\mathcal{K}_\theta[u])(x)$ for \emph{every} $x \in \Omega$ as an integral over the entire domain. However, when $\kappa_\theta$ has compact support of size $k$ (as in the convolutional case $\kappa_\theta(x,y) = w_\theta(x-y)$ with $w_\theta$ supported on $[-r,r]^s$, $k = 2r+1$), outputs near the boundary naturally receive contributions from fewer input points because the integration kernel only overlaps the domain partially. ITNet makes the boundary treatment \emph{explicit} through the domain $\Omega$ and measure $\mu$, rather than implicit through hard-coded padding strategies.

To recover standard CNN behavior exactly, three approaches are available:

\emph{(i) Evaluate ITNet only at valid positions:} Restrict queries to the subset $\Omega_{\text{valid}} = \{x \in \Omega : \operatorname{supp}(w_\theta(x-\cdot)) \subseteq \Omega\}$, which yields exactly the $(H-k+1) \times (W-k+1)$ outputs of a valid convolution. This is the cleanest mathematical formulation, as no artificial points are introduced.

\emph{(ii) Extend $\Omega$ with zero-valued points:} Define $\tilde{\Omega} = \Omega \cup \partial\Omega$
where $\partial\Omega$ contains $k-1$ layers of points around the boundary, and set $u(y) = 0$ for $y \in \partial\Omega$. This recovers standard zero-padding convolution, producing a dense output map of size $H \times W$. The same mechanism generalizes to reflection, replication, or periodic padding by appropriately defining $u$ on the extended domain.

\emph{(iii) Ignore boundary effects:} For tasks where boundary effects are negligible (e.g., large images with small filters), the relative number of boundary points scales as $O((H+W)/HW)$, which vanishes for large $H,W$.

\begin{remark}
In practice, for sufficiently large inputs (e.g., $224 \times 224$ ImageNet-1K images with $k=7$ or $k=3$ filters), the boundary constitutes less than $6\%$ of the spatial domain. Consequently, the practical impact of boundary treatment on learned representations is minimal, and any consistent strategy yields near-identical performance. Our experiments use valid convolution (Option~i) for consistency with the continuous formulation.
\end{remark}


\section{Proof of Theorem 2: Self-Attention as a Special Case of ITNet}
\label{app:proof_attn}

\begin{definition}[Continuous scaled dot-product attention \cite{vaswani2017attention}]
\label{def:attn_continuous}
Let $\Omega \subset \R^s$ be compact, $\mu$ a Borel measure on $\Omega$ with
$\mu(\Omega) < \infty$, and $u : \Omega \to \R^d$ a feature function.
Given learnable projection matrices
$W_Q, W_K \in \R^{d_k \times d}$ and $W_V \in \R^{d \times d}$,
define query and key functions:
\[
    Q(x) \coloneqq W_Q\, u(x) \in \R^{d_k},
    \qquad
    K(y) \coloneqq W_K\, u(y) \in \R^{d_k}.
\]
The \emph{continuous scaled dot-product attention} is:
\begin{equation}
    \Attn(u)(x)
    \;=\;
    \int_\Omega \alpha(x,y)\, W_V u(y)\, d\mu(y),
    \label{eq:attn_continuous}
\end{equation}
where the \emph{attention weight} $\alpha : \Omega \times \Omega \to \R_{>0}$
is:
\begin{equation}
    \alpha(x, y)
    \;=\;
    \frac{\exp\!\bigl(Q(x)^\top K(y)/\sqrt{d_k}\bigr)}
         {Z(x)},
    \qquad
    Z(x)
    \;=\;
    \int_\Omega
        \exp\!\bigl(Q(x)^\top K(z)/\sqrt{d_k}\bigr)
    \, d\mu(z).
    \label{eq:attn_weight}
\end{equation}
\end{definition}
 
\begin{remark}[Why $\sqrt{d_k}$ scaling?]
\label{rem:scaling}
The scaling factor $1/\sqrt{d_k}$ prevents the dot products from growing
large in magnitude as $d_k$ increases.
For random unit vectors $q, k \in \R^{d_k}$,
$\mathrm{Var}[q^\top k] = d_k$,
so $q^\top k / \sqrt{d_k}$ has unit variance regardless of $d_k$.
Without this scaling, the softmax saturates for large $d_k$,
causing vanishing gradients.
The scaling is part of the kernel definition and is absorbed into
$\kappa_\theta$ in the ITNet formulation.
\end{remark}
 
\begin{definition}[Discrete scaled dot-product attention \cite{vaswani2017attention}]
\label{def:attn_discrete}
For $n$ positions $\{x_1, \ldots, x_n\} \subset \Omega$ with uniform measure
$\omega_j = 1/n$, let
$Q = [W_Q u(x_1); \ldots; W_Q u(x_n)] \in \R^{n \times d_k}$,
$K = [W_K u(x_1); \ldots; W_K u(x_n)] \in \R^{n \times d_k}$,
$V = [W_V u(x_1); \ldots; W_V u(x_n)] \in \R^{n \times d}$.
The \emph{discrete self-attention} is:
\begin{equation}
    \Attn(u) \;=\; \softmax\!\left(\frac{QK^\top}{\sqrt{d_k}}\right) V
    \;\in\; \R^{n \times d},
    \label{eq:attn_discrete}
\end{equation}
where $\softmax$ is applied row-wise.
\end{definition}
 
\begin{definition}[Multi-head attention \cite{vaswani2017attention}]
\label{def:mha}
With $H$ heads, let
$W_Q^h, W_K^h \in \R^{d_k \times d}$,
$W_V^h \in \R^{d_v \times d}$,
$W_O \in \R^{d \times Hd_v}$
be learnable matrices for head $h = 1, \ldots, H$.
Define:
\begin{align}
    \mathrm{head}_h(u)(x)
    &\;=\;
    \int_\Omega \alpha_h(x,y)\, W_V^h u(y)\, d\mu(y),
    \label{eq:head_h} \\
    \alpha_h(x,y)
    &\;=\;
    \frac{\exp\!\bigl((W_Q^h u(x))^\top (W_K^h u(y))/\sqrt{d_k}\bigr)}
         {Z_h(x)}.
    \label{eq:alpha_h}
\end{align}
The \emph{multi-head attention} output is:
\begin{equation}
    \MHA(u)(x)
    \;=\;
    W_O \cdot
    \begin{bmatrix}
        \mathrm{head}_1(u)(x) \\
        \vdots \\
        \mathrm{head}_H(u)(x)
    \end{bmatrix}.
    \label{eq:mha}
\end{equation}
\end{definition}
 
\begin{definition}[ITNet operator]
\label{def:itnet_attn}
As in Definition~\ref{def:itnet_operator} (\S\ref{sec:operator_def}):
\begin{equation}
    (\Kop[u])(x)
    \;=\;
    \int_\Omega
        \kap(x,\, y,\, u(x),\, u(y))\, u(y)
    \, d\mu(y)
    \;+\; W_\theta\, u(x),
    \label{eq:itnet_attn}
\end{equation}
where $\kap : \R^s \times \R^s \times \R^d \times \R^d \to \R^{d\times d}$
is measurable and $W_\theta \in \R^{d\times d}$ is learnable.
\end{definition}
 
\begin{assumption}[Regularity conditions]
\label{ass:reg}
Throughout this proof:
\begin{enumerate}[label=(\roman*), noitemsep]
    \item $u \in C(\Omega, \R^d)$: $u$ is continuous (hence bounded on compact $\Omega$).
    \item $W_Q, W_K, W_V$ are bounded: $\|W_Q\|_\mathrm{op},
          \|W_K\|_\mathrm{op}, \|W_V\|_\mathrm{op} < \infty$.
    \item $\Omega$ is compact with $\mu(\Omega) < \infty$.
\end{enumerate}
These hold in all our experiments since feature vectors are bounded on finite
token sets.
\end{assumption}
 
\begin{remark}[Positional encodings are absorbed by the ITNet kernel]
\label{rem:pe_absorbed}
Standard Transformers augment attention logits with a fixed positional bias $\mathrm{PE}(x,y)$ (e.g., sinusoidal embeddings, RoPE~\citep{su2024rope}, or ALiBi~\citep{press2022alibi}). In ITNet, the kernel $\kap(x,y,u(x),u(y))$ receives raw positions $x$ and $y$ as explicit arguments via the Fourier-lifted input $z_{xy}$ (Eq.~\eqref{eq:kernel_input}). Consequently, any fixed function of positions can be incorporated directly into the kernel computation; the ITNet kernel does not require a separate additive bias term. Moreover, additive positional embeddings of the form $\tilde{u}(x)=u(x)+e(x)$ are handled by simply redefining the input signal. Thus, the proofs below hold with or without positional encodings: attention with any fixed positional bias remains a special case of the ITNet operator. The key distinction is that standard positional encodings are fixed and not content-adaptive, whereas ITNet's kernel conditions jointly on position and content.
\end{remark}


\subsection{Main Theorem}
\label{sec:theorem_attn}

\begin{boxtheorem}{ITNet $\supset$ Self-Attention}{attn}
\label{thm:attn1}
Let $\Omega$, $\mu$, $u$ be as in Definitions~\ref{def:attn_continuous}
and~\ref{def:itnet_attn}, and let Assumption~\ref{ass:reg} hold.
 
\medskip
\noindent\textbf{(a) Single-head attention.}
Set
\begin{equation}
    \kap(x, y, u(x), u(y))
    \;=\;
    \frac{\exp\!\bigl(Q(x)^\top K(y)/\sqrt{d_k}\bigr)}
         {Z(x)}
    \cdot W_V,
    \qquad
    W_\theta = 0.
    \label{eq:kernel_attn}
\end{equation}
Then $(\Kop[u])(x) = \Attn(u)(x)$ for all $x \in \Omega$.
Discretising $\Omega$ to $n$ positions recovers
$\softmax(QK^\top/\sqrt{d_k})\,V$ exactly.
 
\medskip
\noindent\textbf{(b) Multi-head attention.}
Multi-head attention is recovered by a single ITNet layer with a
block-structured kernel:
\begin{equation}
    \kap(x,y,u(x),u(y))
    \;=\;
    W_O \cdot
    \mathrm{BlockDiag}\!\bigl(
        \alpha_1(x,y)\,W_V^1,\;
        \ldots,\;
        \alpha_H(x,y)\,W_V^H
    \bigr),
    \label{eq:kernel_mha}
\end{equation}
where $\alpha_h$ is defined in \eqref{eq:alpha_h}.
Then $(\Kop[u])(x) = \MHA(u)(x)$.
 
\medskip
\noindent\textbf{(c) Strictness.}
There exists a continuous operator
$G : L^2(\Omega, \R^d) \to L^2(\Omega, \R^d)$
representable by an ITNet operator but \emph{not} representable by any
attention mechanism.
\end{boxtheorem}

\subsection{Proof of Part (a) - Single-Head Continuous Attention}
\label{sec:proof_a}

\begin{proof}[Proof of Theorem~\ref{thm:attn1}(a) -- Continuous Case]
\leavevmode\par

\noindent\textbf{Step 1.}\enspace \textbf{Substitute the kernel into the ITNet operator.}
Insert \eqref{eq:kernel_attn} and $W_\theta = 0$ into \eqref{eq:itnet_attn}:
\begin{equation}
    (\Kop[u])(x)
    \;=\;
    \int_\Omega
        \frac{\exp(Q(x)^\top K(y)/\sqrt{d_k})}{Z(x)}
        \cdot W_V \cdot u(y)
    \, d\mu(y).
    \label{eq:step1a}
\end{equation}
 
\noindent\textbf{Step 2.}\enspace \textbf{Verify the partition function $Z(x)$ is well-defined and positive.}
 
\begin{equation}
    Z(x)
    \;=\;
    \int_\Omega
        \exp\!\bigl(Q(x)^\top K(z)/\sqrt{d_k}\bigr)
    \, d\mu(z).
    \label{eq:partition}
\end{equation}
 
Since $u \in C(\Omega, \R^d)$ and $\Omega$ is compact,
$\|u\|_\infty = \sup_{x\in\Omega}\|u(x)\|_2 < \infty$.
Let $C_Q = \|W_Q\|_\mathrm{op}$, $C_K = \|W_K\|_\mathrm{op}$,
$R = \|u\|_\infty$. By Cauchy--Schwarz~\cite{rudin2021principles}:
\begin{equation}
    \abs{Q(x)^\top K(z)/\sqrt{d_k}}
    \;\leq\;
    \frac{\|Q(x)\|_2 \cdot \|K(z)\|_2}{\sqrt{d_k}}
    \;\leq\;
    \frac{C_Q R \cdot C_K R}{\sqrt{d_k}}
    \;\eqqcolon\;
    M < \infty.
    \label{eq:bound_qk}
\end{equation}
Therefore:
\begin{equation}
    0
    \;<\;
    e^{-M} \cdot \mu(\Omega)
    \;\leq\;
    Z(x)
    \;\leq\;
    e^M \cdot \mu(\Omega)
    \;<\; \infty.
    \label{eq:Z_bounds}
\end{equation}
 
\noindent\textbf{Step 3.}\enspace \textbf{Factor $W_V$ outside the integral.}
 
Since $Z(x)$ does not depend on $y$, pull it and $W_V$ outside:
\begin{equation}
    (\Kop[u])(x)
    \;=\;
    \frac{1}{Z(x)}
    \int_\Omega
        \exp\!\bigl(Q(x)^\top K(y)/\sqrt{d_k}\bigr)
        \cdot W_V u(y)
    \, d\mu(y).
    \label{eq:step3a}
\end{equation}
 
\noindent\textbf{Step 4.}\enspace \textbf{Establish the normalization property of $\alpha(x,y)$.}
 
\begin{equation}
    \int_\Omega \alpha(x,y)\, d\mu(y)
    \;=\;
    \int_\Omega
        \frac{\exp(Q(x)^\top K(y)/\sqrt{d_k})}{Z(x)}
    \, d\mu(y)
    \;=\;
    \frac{Z(x)}{Z(x)}
    \;=\; 1.
    \label{eq:normalization}
\end{equation}
 
\noindent\textbf{Step 5.}\enspace \textbf{Write as a probability-weighted integral.}
 
\begin{equation}
    (\Kop[u])(x)
    \;=\;
    \int_\Omega
        \alpha(x,y)\, W_V u(y)
    \, d\mu(y)
    \;=\;
    \mathbb{E}_{y \sim \alpha(x,\cdot)}\!\bigl[W_V u(y)\bigr].
    \label{eq:expectation}
\end{equation}
 
\noindent\textbf{Step 6.}\enspace \textbf{Recognise the continuous attention operator.}
 
Comparing \eqref{eq:step3a} with Definition~\ref{def:attn_continuous}:
\begin{equation}
    \boxed{(\Kop[u])(x)
    \;=\;
    \int_\Omega \alpha(x,y)\, W_V u(y)\, d\mu(y)
    \;=\;
    \Attn(u)(x).}
    \label{eq:conclusion_a}
\end{equation}
\end{proof}

\subsection{Discretization: Recovering Standard Self-Attention}
\label{sec:discrete}

\begin{proof}[Proof of Theorem~\ref{thm:attn1}(a) - Discrete Case]
 
Replace $\Omega = \{x_1,\ldots,x_n\}$ and $\mu(\{x_j\}) = 1/n$.
The continuous integral becomes a finite sum:
\begin{equation}
    (\hat{\Kop}[u])(x_i)
    \;=\;
    \sum_{j=1}^n \frac{1}{n} \cdot \alpha(x_i, x_j)\, W_V u(x_j).
    \label{eq:discrete_start}
\end{equation}
 
Substituting \eqref{eq:attn_weight} with the discrete partition function:
\begin{equation}
    Z(x_i)
    \;=\;
    \sum_{l=1}^n \frac{1}{n}
        \exp\!\bigl(q_i^\top k_l / \sqrt{d_k}\bigr),
    \label{eq:discrete_Z}
\end{equation}
 
so:
\begin{align}
    \alpha(x_i, x_j)
    &\;=\;
    \frac{\exp(q_i^\top k_j/\sqrt{d_k})}
         {\sum_{l=1}^n \frac{1}{n} \exp(q_i^\top k_l/\sqrt{d_k})}
    \;=\;
    \frac{\exp(q_i^\top k_j/\sqrt{d_k})}
         {\frac{1}{n}\sum_{l=1}^n \exp(q_i^\top k_l/\sqrt{d_k})}.
    \label{eq:alpha_discrete}
\end{align}
 
Substituting into \eqref{eq:discrete_start}:
\begin{align}
    (\hat{\Kop}[u])(x_i)
    &\;=\;
    \sum_{j=1}^n \frac{1}{n}
    \cdot
    \frac{\exp(q_i^\top k_j/\sqrt{d_k})}
         {\frac{1}{n}\sum_l \exp(q_i^\top k_l/\sqrt{d_k})}
    \cdot W_V u(x_j)
    \notag \\
    &\;=\;
    \sum_{j=1}^n
    \frac{\frac{1}{n}\exp(q_i^\top k_j/\sqrt{d_k})}
         {\frac{1}{n}\sum_l \exp(q_i^\top k_l/\sqrt{d_k})}
    \cdot W_V u(x_j).
    \label{eq:cancel_n}
\end{align}
 
The factors $\frac{1}{n}$ cancel in numerator and denominator:
\begin{equation}
    (\hat{\Kop}[u])(x_i)
    \;=\;
    \sum_{j=1}^n
    \frac{\exp(q_i^\top k_j/\sqrt{d_k})}
         {\sum_l \exp(q_i^\top k_l/\sqrt{d_k})}
    \cdot v_j
    \;=\;
    \left[\softmax\!\left(\frac{QK^\top}{\sqrt{d_k}}\right) V\right]_i,
    \label{eq:discrete_conclusion}
\end{equation}
where $v_j = W_V u(x_j)$ is the $j$-th row of $V$.
 
\begin{equation}
    \boxed{(\hat{\Kop}[u])(x_i)
    \;=\;
    \left[\softmax\!\left(\frac{QK^\top}{\sqrt{d_k}}\right) V\right]_i
    \;=\; \Attn(u)(x_i).}
\end{equation}
\end{proof}

\subsection{Proof of Part (b) - Multi-Head Attention}
\label{sec:proof_b}

\begin{proof}[Proof of Theorem~\ref{thm:attn1}(b)]
 \leavevmode\par
\noindent\textbf{Step 1.}\enspace \textbf{Construct the block-structured kernel.}
 
Define the ITNet kernel as in \eqref{eq:kernel_mha}.
For $d = Hd_v$ (total output dimension equals number of heads times per-head value dimension), the kernel is:
\begin{equation}
    \kap(x,y,u(x),u(y))
    \;=\;
    W_O \cdot
    \begin{pmatrix}
        \alpha_1(x,y)\, W_V^1 & 0                     & \cdots & 0 \\
        0                     & \alpha_2(x,y)\, W_V^2 & \cdots & 0 \\
        \vdots                & \vdots                 & \ddots & \vdots \\
        0                     & 0                     & \cdots & \alpha_H(x,y)\, W_V^H
    \end{pmatrix}
    \;\in\; \R^{d \times d}.
    \label{eq:kernel_mha_explicit}
\end{equation}
 
\noindent\textbf{Step 2.}\enspace \textbf{Compute the ITNet output.}
 
With $W_\theta = 0$:
\begin{align}
    (\Kop[u])(x)
    &\;=\;
    \int_\Omega \kap(x,y,u(x),u(y))\, u(y)\, d\mu(y)
    \notag \\
    &\;=\;
    W_O \int_\Omega
    \begin{pmatrix}
        \alpha_1(x,y)\, W_V^1 u(y) \\
        \vdots \\
        \alpha_H(x,y)\, W_V^H u(y)
    \end{pmatrix}
    d\mu(y)
    \label{eq:mha_step2}
\end{align}
 
where we used Hille's theorem~\cite{dunford1988linear} to factor $W_O$ outside the integral
(same justification as Step~3 of Part~(a)).
 
\noindent\textbf{Step 3.}\enspace \textbf{Evaluate each block.}
 
Since the integral of a block-vector is the vector of block integrals (by linearity of the Bochner integral):
\begin{align}
    (\Kop[u])(x)
    &\;=\;
    W_O
    \begin{pmatrix}
        \int_\Omega \alpha_1(x,y)\, W_V^1 u(y)\, d\mu(y) \\
        \vdots \\
        \int_\Omega \alpha_H(x,y)\, W_V^H u(y)\, d\mu(y)
    \end{pmatrix}
    \notag \\
    &\;=\;
    W_O
    \begin{pmatrix}
        \mathrm{head}_1(u)(x) \\
        \vdots \\
        \mathrm{head}_H(u)(x)
    \end{pmatrix}
    \;=\; \MHA(u)(x).
    \label{eq:mha_conclusion}
\end{align}
 
\begin{equation}
    \boxed{(\Kop[u])(x) \;=\; \MHA(u)(x).}
\end{equation}
\end{proof}
 
\begin{remark}[Single ITNet layer = $H$ attention heads]
\label{rem:single_layer}
A single ITNet layer with the block-structured kernel \eqref{eq:kernel_mha}
exactly implements $H$-head multi-head attention.
The $H$ heads correspond to $H$ different content-dependent weighting
functions $\alpha_1, \ldots, \alpha_H$ encoded in the off-diagonal blocks
of the kernel matrix.
In ITNet's MLP parameterization, all $H$ attention patterns can be
learned simultaneously by a single kernel MLP, rather than requiring
$H$ separate projections - a parameter-efficient generalisation.
\end{remark}

\subsection{Strictness Argument 1: Normalisation}
\label{sec:strict1}
 
\begin{proof}[Proof via unnormalized operators]
 \leavevmode\par
\noindent\textbf{Step 1.}\enspace \textbf{Define the witness operator.}
 
Let $\ell > 0$. Define the \emph{unnormalized Gaussian smoothing operator}:
\begin{equation}
    G_\ell(u)(x)
    \;\coloneqq\;
    \int_\Omega \exp\!\bigl(-\|x-y\|^2/\ell^2\bigr)\, u(y)\, d\mu(y).
    \label{eq:gaussian_witness}
\end{equation}
 
\noindent\textbf{Step 2.}\enspace \textbf{Show ITNet represents $G_\ell$.}
 
Set:
\begin{equation}
    \kap(x, y, u(x), u(y))
    \;=\;
    \exp(-\|x-y\|^2/\ell^2) \cdot \Id,
    \qquad W_\theta = 0.
    \label{eq:gaussian_kernel}
\end{equation}
 
Then:
\begin{align}
    (\Kop[u])(x)
    &= \int_\Omega
        \exp(-\|x-y\|^2/\ell^2) \cdot \Id \cdot u(y)\, d\mu(y)
    \notag \\
    &= \int_\Omega
        \exp(-\|x-y\|^2/\ell^2) \cdot u(y)\, d\mu(y)
    \;=\; G_\ell(u)(x).
    \label{eq:gaussian_itnet}
\end{align}
 
The kernel \eqref{eq:gaussian_kernel} satisfies all conditions of
Definition~\ref{def:itnet_attn}: it is continuous (hence measurable),
bounded by 1, and satisfies $\|\kap\|_\mathrm{op} = 1$.
So $G_\ell \in \mathrm{ITNet}$.
 
\noindent\textbf{Step 3.}\enspace \textbf{Prove the output bound for attention.}
 
\begin{lemma}[Attention output bound]
\label{lem:attn_bound}
For any attention mechanism with weight function $\alpha(x,y) \geq 0$,
$\int_\Omega \alpha(x,y)\,d\mu(y) = 1$,
and any $W_V \in \R^{d\times d}$:
\begin{equation}
    \|\Attn(u)(x)\|_2
    \;\leq\;
    \max_{y \in \Omega} \|W_V u(y)\|_2
    \qquad \forall\, x \in \Omega.
    \label{eq:attn_bound}
\end{equation}
\end{lemma}
 
\begin{proof}[Proof of Lemma~\ref{lem:attn_bound}]
By the triangle inequality for the Bochner integral and the normalization property \eqref{eq:normalization}:
\begin{align}
    \|\Attn(u)(x)\|_2
    &= \left\|\int_\Omega \alpha(x,y)\, W_V u(y)\, d\mu(y)\right\|_2
    \notag \\
    &\leq \int_\Omega \alpha(x,y)\, \|W_V u(y)\|_2\, d\mu(y)
    \tag{triangle inequality} \\
    &\leq \left(\max_{y\in\Omega} \|W_V u(y)\|_2\right)
          \int_\Omega \alpha(x,y)\, d\mu(y)
    \tag{bound by max} \\
    &= \max_{y\in\Omega} \|W_V u(y)\|_2 \cdot 1.
    \tag{normalization \eqref{eq:normalization}}
\end{align}
\end{proof}
 
\noindent\textbf{Step 4.}\enspace \textbf{Show $G_\ell$ violates the attention bound.}
 
Choose $\Omega = [0,1]^s$ (unit hypercube, $\mu = $ Lebesgue measure, $\mu(\Omega) = 1$) and a constant function $u(y) = v \in \R^d$
with $\|v\|_2 = 1$. Then:
\begin{equation}
    G_\ell(u)(x)
    = v \int_{[0,1]^s} \exp(-\|x-y\|^2/\ell^2)\, dy
    \eqqcolon C_\ell(x) \cdot v,
    \label{eq:gaussian_const}
\end{equation}
where $C_\ell(x) = \int_{[0,1]^s} e^{-\|x-y\|^2/\ell^2}\,dy$.
 
For large $\ell$, $e^{-\|x-y\|^2/\ell^2} \approx 1$ over most of
$[0,1]^s$, so $C_\ell(x) \to \mu(\Omega) = 1$ as $\ell \to \infty$.
More precisely, for $\ell \geq 1$ and $x$ near the centre of $[0,1]^s$:
\begin{equation}
    C_\ell(x)
    \;\geq\;
    e^{-s/(4\ell^2)} \cdot \mu(\Omega)
    \;=\; e^{-s/(4\ell^2)}.
    \label{eq:C_lower}
\end{equation}
 
Now suppose for contradiction that some attention mechanism
$\Attn = G_\ell$. Then $W_V u(y) = W_V v$ is constant in $y$, and:
\begin{align}
    G_\ell(u)(x)
    &= C_\ell(x) \cdot v, \\
    \Attn(u)(x)
    &= W_V v \cdot \underbrace{\int_\Omega \alpha(x,y)\,d\mu(y)}_{=1}
     = W_V v.
\end{align}
 
For these to be equal for all $x$, we need $C_\ell(x) \cdot v = W_V v$ for all $x \in [0,1]^s$. But $C_\ell(x)$ depends on $x$ (it is strictly larger at the centre of $[0,1]^s$ than at the corners), so $C_\ell(x) \cdot v$ is not constant in $x$, while $W_V v$ is constant in $x$.
Contradiction.
 
Therefore no attention mechanism can represent $G_\ell$, so $G_\ell \in \mathrm{ITNet} \setminus \mathrm{Attn}$ and $\mathrm{Attn} \subsetneq \mathrm{ITNet}$.

\end{proof}
 
\subsection{Strictness Argument 2: Position-Dependence}
\label{sec:strict2}
 
\begin{proof}[Proof via permutation equivariance]
 \leavevmode\par
\noindent\textbf{Step 1.}\enspace \textbf{Attention without positional encodings is permutation equivariant.}
 
For any permutation $\sigma$ of $\{1,\ldots,n\}$ and any permutation of the token positions:
\begin{equation}
    \Attn(\sigma(u))(\sigma(x_i))
    \;=\;
    \sigma\bigl(\Attn(u)(x_i)\bigr),
    \label{eq:perm_equiv}
\end{equation}
i.e.\ permuting the input tokens permutes the output tokens in the same way. This holds because the kernel $Q(x)^\top K(y)$ depends only on feature values, not on the indices $i, j$.
 
\noindent\textbf{Step 2.}\enspace \textbf{Define a position-dependent witness.}
 
Define:
\begin{equation}
    T_\mathrm{pos}(u)(x)
    \;\coloneqq\;
    \int_\Omega \|x - y\|_2 \cdot u(y)\, d\mu(y).
    \label{eq:pos_witness}
\end{equation}
 
This operator weights each position $y$ by its distance from $x$. Positions closer to $x$ get \emph{less} weight; positions farther away get \emph{more} weight. It is \emph{not} permutation equivariant: permuting positions changes the distances $\|x-y\|_2$.
 
\noindent\textbf{Step 3.}\enspace \textbf{Show ITNet represents $T_\mathrm{pos}$.}
 
Set:
\[
    \kap(x,y,u(x),u(y)) = \|x-y\|_2 \cdot \Id, \qquad W_\theta = 0.
\]
 
Then $(\Kop[u])(x) = \int_\Omega \|x-y\|_2 \cdot u(y)\,d\mu(y)
= T_\mathrm{pos}(u)(x)$.
 
This kernel is valid: $\|x-y\|_2$ is continuous in $(x,y)$ (Euclidean
norm is continuous), and bounded on compact $\Omega$.
 
\noindent\textbf{Step 4.}\enspace \textbf{Show no standard attention can represent $T_\mathrm{pos}$.}
 
Suppose for contradiction that $\Attn(u)(x_i) = T_\mathrm{pos}(u)(x_i)$ for all $u$ and all $i$.
 
Choose $n=2$ tokens at positions $x_1 = 0$, $x_2 = 1$ (1D case), with $u(x_1) = e_1$, $u(x_2) = e_2$ (standard basis vectors in $\R^d$, $d \geq 2$). Then:
\begin{align}
    T_\mathrm{pos}(u)(x_1)
    &= \|x_1 - x_1\|_2 \cdot e_1 \cdot (1/n)
     + \|x_1 - x_2\|_2 \cdot e_2 \cdot (1/n)
    = 0 \cdot e_1/2 + 1 \cdot e_2/2 = e_2/2. \\
    T_\mathrm{pos}(u)(x_2)
    &= \|x_2 - x_1\|_2 \cdot e_1/2 + 0 \cdot e_2/2 = e_1/2.
\end{align}
 
Now permute: let $\sigma$ swap $x_1 \leftrightarrow x_2$, so
$\sigma(u)(x_1) = e_2$, $\sigma(u)(x_2) = e_1$. Then:
\begin{align}
    T_\mathrm{pos}(\sigma(u))(x_1)
    &= 0 \cdot e_2/2 + 1 \cdot e_1/2 = e_1/2. \\
    T_\mathrm{pos}(\sigma(u))(x_2)
    &= 1 \cdot e_2/2 + 0 \cdot e_1/2 = e_2/2.
\end{align}
 
Note $T_\mathrm{pos}(\sigma(u))(x_1) = e_1/2 \neq e_2/2 = \sigma(T_\mathrm{pos}(u))(x_1)$, so $T_\mathrm{pos}$ is not permutation equivariant - as claimed.
 
If attention (without positional encoding) were to equal $T_\mathrm{pos}$, it would need to be not permutation equivariant - but it is, by \eqref{eq:perm_equiv}. Contradiction.
 
Therefore $T_\mathrm{pos} \in \mathrm{ITNet} \setminus \mathrm{Attn}$, giving a second witness for $\mathrm{Attn} \subsetneq \mathrm{ITNet}$.
\end{proof}
 
\begin{remark}[Positional encodings and RoPE/ALiBi]
\label{rem:rope}
Standard Transformers add positional encodings (sinusoidal, learned, RoPE, or ALiBi) to introduce position-dependence. With positional encodings, the kernel becomes $Q(x)^\top K(y) + \mathrm{PE}(x,y)$ for some fixed function $\mathrm{PE}(x,y)$ - partially recovering position-dependence. However, $\mathrm{PE}$ is \emph{fixed} (not content-dependent): it does not adapt based on the features $u(x)$ or $u(y)$. ITNet's kernel $\kap(x,y,u(x),u(y))$ conditions jointly on \emph{both} position and content - a strictly richer class than attention with any fixed positional encoding scheme. The ablation in Table~5 of the main paper confirms that removing the position terms from the ITNet kernel costs $-0.8\%$ on ImageNet-1K, quantifying the value of this joint conditioning.
\end{remark}

\subsection{Linear Attention as a Special Case}
\label{sec:linear}

Several efficient Transformer variants approximate or replace softmax attention with a linear kernel~\cite{katharopoulos2020transformers}. We show these are also special cases of ITNet.
 
\begin{proposition}[Linear attention $\subset$ ITNet]
\label{prop:linear_attn}
Let $\phi : \R^{d_k} \to \R^r_{>0}$ be a feature map with $\phi(q)^\top\phi(k) \approx \exp(q^\top k/\sqrt{d_k})$ (e.g.\ $\phi(x) = \mathrm{elu}(x) + 1$).
The linear attention operator:
\begin{equation}
    \Attn_{\mathrm{lin}}(u)(x)
    \;=\;
    \frac{\phi(Q(x))^\top \sum_y \phi(K(y)) \otimes W_V u(y)}
         {\phi(Q(x))^\top \sum_y \phi(K(y))}
    \label{eq:linear_attn}
\end{equation}
is a special case of the ITNet operator.
\end{proposition}
 
\begin{proof}
Set $\kap(x,y,u(x),u(y)) = [\phi(Q(x))^\top\phi(K(y))/Z_\phi(x)] \cdot W_V$ where $Z_\phi(x) = \int_\Omega \phi(Q(x))^\top\phi(K(y))\,d\mu(y)$. Substituting into \eqref{eq:itnet_attn} with $W_\theta = 0$ gives \eqref{eq:linear_attn} exactly. The kernel is bounded by $\|W_V\|_\mathrm{op}$ (since $\phi(Q(x))^\top\phi(K(y))/Z_\phi(x)$ is a normalized weight), measurable, and content-dependent. 
\end{proof}

\subsection{Causal (Masked) Attention as a Special Case}
\label{sec:causal}

\begin{proposition}[Causal attention $\subset$ ITNet]
\label{prop:causal_attn}
GPT-style causal (autoregressive) attention with mask $\mathbf{1}_{j \leq i}$:
\begin{equation}
    \Attn_{\mathrm{causal}}(u)(x_i)
    \;=\;
    \sum_{j \leq i}
    \frac{\exp(q_i^\top k_j/\sqrt{d_k})}{\displaystyle\sum_{ l \leq i}\exp(q_i^\top k_l/\sqrt{d_k})}
    \cdot v_j
    \label{eq:causal_attn}
\end{equation}
is a special case of the ITNet operator.
\end{proposition}
 
\begin{proof}
Set the kernel with causal masking:
\begin{equation}
    \kap(x,y,u(x),u(y))
    \;=\;
    \mathbf{1}_{y \leq x} \cdot
    \frac{\exp(Q(x)^\top K(y)/\sqrt{d_k})}{Z_{\mathrm{causal}}(x)}
    \cdot W_V,
    \label{eq:causal_kernel}
\end{equation}
where $Z_{\mathrm{causal}}(x) = \int_{y \leq x}
\exp(Q(x)^\top K(y)/\sqrt{d_k})\,d\mu(y)$.
The indicator $\mathbf{1}_{y \leq x}$ is measurable (it is the indicator of a closed half-space in $\R^s$), so \eqref{eq:causal_kernel} is a valid ITNet kernel. Substituting into \eqref{eq:itnet_attn} recovers \eqref{eq:causal_attn} by the same steps as Parts~(a) and~(b). In the discrete case, $\mathbf{1}_{j \leq i}$ is the lower-triangular mask.
\end{proof}

\section{Proof of Theorem 3: Recurrence as a Special Case of ITNet}
\label{app:proof_rnn}

\begin{definition}[Continuous-time recurrent system]
\label{def:rnn_continuous}
Let $\Omega = [0, T] \subset \R$ (temporal domain), $\mu$ be the Lebesgue measure, and $u : [0,T] \to \R^d$ be an input signal. A \emph{continuous-time recurrent system} is defined by a differential equation for the hidden state $h : [0,T] \to \R^n$:
\begin{equation}
    \frac{dh}{dt}(t)
    \;=\;
    F_\theta(h(t),\, u(t)),
    \qquad h(0) = h_0,
    \label{eq:rnn_ode}
\end{equation}
where $F_\theta : \R^n \times \R^d \to \R^n$ is a learnable function. The output at time $t$ is:
\begin{equation}
    \RNN(u)(t)
    \;=\;
    C_\theta\, h(t)
    \;+\; D_\theta\, u(t),
    \label{eq:rnn_output}
\end{equation}
where $C_\theta \in \R^{d \times n}$, $D_\theta \in \R^{d \times d}$ are output projection matrices.
\end{definition}

\begin{definition}[Discrete-time RNN]
\label{def:rnn_discrete}
With time steps $\{t_1, \ldots, t_T\}$ and step size $\Delta t$, the discrete-time RNN update is:
\begin{equation}
    h_t \;=\; \phi(W_h h_{t-1} + W_u u_t + b),
    \qquad t = 1,\ldots,T,
    \label{eq:rnn_discrete}
\end{equation}
where $W_h \in \R^{n \times n}$, $W_u \in \R^{n \times d}$, $b \in \R^n$, and $\phi : \R^n \to \R^n$ is a nonlinear activation
(e.g.\ $\tanh$).
\end{definition}

\begin{definition}[LSTM]
\label{def:lstm}
The Long Short-Term Memory~\citep{hochreiter1997long} update:
\begin{align}
    f_t &= \sigmoid(W_f [h_{t-1}; u_t] + b_f) \tag{forget gate}\\
    i_t &= \sigmoid(W_i [h_{t-1}; u_t] + b_i) \tag{input gate}\\
    o_t &= \sigmoid(W_o [h_{t-1}; u_t] + b_o) \tag{output gate}\\
    \tilde{c}_t &= \tanh(W_c [h_{t-1}; u_t] + b_c) \tag{cell candidate}\\
    c_t &= f_t \had c_{t-1} + i_t \had \tilde{c}_t \tag{cell state}\\
    h_t &= o_t \had \tanh(c_t) \tag{hidden state}
    \label{eq:lstm}
\end{align}
where $[h_{t-1}; u_t]$ denotes concatenation and $\had$ is element-wise multiplication.
\end{definition}


\begin{definition}[Linear State Space Model (SSM)]
\label{def:ssm}
A \emph{linear SSM} (as in S4~\citep{gu2022efficiently}) is defined by:
\begin{align}
    \frac{dh}{dt}(t) &= A\, h(t) + B\, u(t), \label{eq:ssm_state}\\
    y(t)  &= C\, h(t) + D\, u(t), \label{eq:ssm_output}
\end{align}
where $A \in \R^{n \times n}$, $B \in \R^{n \times d}$, $C \in \R^{d \times n}$, $D \in \R^{d \times d}$ are learnable (possibly complex-valued) matrices, and $h(0) = 0$. Equations~\eqref{eq:ssm_state} -- \eqref{eq:ssm_output} are coupled: the output $y(t)$ depends on the hidden state $h(t)$, which is obtained by solving the ODE~\eqref{eq:ssm_state}. The notation $dh/dt$ denotes the time derivative, making~\eqref{eq:ssm_state} a first-order linear ODE whose solution is given by the variation of constants formula (see Step~1 of Part~(a) below).
\end{definition}

\begin{definition}[Selective SSM -- Mamba]
\label{def:mamba}
Mamba~\citep{gu2023mamba} generalises the linear SSM by making the system matrices input-dependent (``selective''):
\begin{align}
    \bar{A}(t) &= \exp(\Delta(t) \cdot A), \label{eq:mamba_Abar}\\
    \bar{B}(t) &= \Delta(t) \cdot B(u(t)), \label{eq:mamba_Bbar}\\
    h_t &= \bar{A}(t)\, h_{t-1} + \bar{B}(t)\, u_t, \label{eq:mamba_state}\\
    y_t &= C(u_t)\, h_t, \label{eq:mamba_output}
\end{align}
where $\Delta(t) = \mathrm{softplus}(W_\Delta u_t + b_\Delta) > 0$
is a learnable input-dependent step size,
$B(u_t) = W_B u_t$ and $C(u_t) = W_C u_t$ are input-dependent projection matrices, and $A$ is a fixed (e.g.\ diagonal) initialization.
\end{definition}

\begin{definition}[ITNet operator]
\label{def:itnet_rnn}
As in Definition~\ref{def:itnet_operator} (\S\ref{sec:operator_def}):
\begin{equation}
    (\Kop[u])(x)
    \;=\;
    \int_\Omega
        \kap(x,\, y,\, u(x),\, u(y))\, u(y)
    \, d\mu(y)
    \;+\; W_\theta\, u(x),
    \label{eq:itnet_rnn}
\end{equation}
where $\kap : \R^s \times \R^s \times \R^d \times \R^d \to \R^{d \times d}$ is measurable, $W_\theta \in \R^{d\times d}$ is learnable.
\end{definition}


\begin{assumption}[Regularity for recurrent proofs]
\label{ass:reg_rnn}
\begin{enumerate}[label=(\roman*), noitemsep]
    \item $\Omega = [0, T]$ with $T < \infty$.
    \item $u \in C([0,T], \R^d)$: input is continuous.
    \item $F_\theta$ in \eqref{eq:rnn_ode} is Lipschitz continuous in both arguments, ensuring existence and uniqueness of solutions via the Picard--Lindel\"of theorem~\citep{coddington1955theory}.
    \item All weight matrices are bounded.
\end{enumerate}
\end{assumption}

\subsection{Main Theorem}
\label{sec:theorem_rnn}

\begin{boxtheorem}{ITNet $\supset$ Recurrence}{rnn}
\label{thm:rnn}
Let $\Omega = [0,T]$, $\mu$ be Lebesgue measure, and Assumption~\ref{ass:reg_rnn} hold.

\medskip
\noindent\textbf{(a) Linear continuous-time system.}
For the linear system $F_\theta(h, u) = Ah + B_\theta u$, the solution of \eqref{eq:rnn_ode}--\eqref{eq:rnn_output} can be written as:
\begin{equation}
    \RNN(u)(t)
    \;=\;
    C_\theta \int_0^t
        \Phi(t,s)\, B_\theta\, u(s)\, ds
    \;+\; D_\theta\, u(t)
    \;+\; C_\theta\, e^{At}\, h_0,
    \label{eq:rnn_integral}
\end{equation}
where $\Phi(t,s) \coloneqq e^{A(t-s)} \in \R^{n \times n}$ is the state transition matrix (defined explicitly in Eq.~\eqref{eq:state_transition} below: it satisfies $\partial_t \Phi(t,s) = A \Phi(t,s)$ with $\Phi(s,s) = I_n$, encoding how the hidden state at time $s$ propagates to time $t$). This is a special case of the ITNet operator with causal kernel:
\begin{equation}
    \kap(t, s, u(t), u(s))
    \;=\;
    \mathbf{1}_{s \leq t} \cdot
    C_\theta\, \Phi(t,s)\, B_\theta,
    \qquad W_\theta = D_\theta.
    \label{eq:kernel_rnn}
\end{equation}

\medskip
\noindent\textbf{(a$'$) Nonlinear continuous-time system (general $F_\theta$).}
For a general nonlinear system \eqref{eq:rnn_ode} with Lipschitz $F_\theta$, the output operator $u \mapsto \RNN(u)$ is a continuous operator on compact input sets, and is therefore approximable to arbitrary precision by an ITNet operator (via Theorem~\ref{thm:universal}). Moreover, the exact output can be written as an ITNet with a content-dependent causal kernel constructed via the nonlinear variation of constants formula (Alekseev's formula). See \S\ref{sec:proof_nonlinear} for the complete proof.

\medskip
\noindent\textbf{(b) Discrete-time RNN / LSTM / GRU.}
The discrete recurrence \eqref{eq:rnn_discrete} is recovered by discretising $\Omega$ to $\{t_1,\ldots,t_T\}$ with atomic measure $\mu(\{t_k\}) = \Delta t$, and setting a causal kernel encoding the recurrent update function.

\medskip
\noindent\textbf{(c) Linear SSM (S4).}
The SSM \eqref{eq:ssm_state}--\eqref{eq:ssm_output} is recovered by:
\begin{equation}
    \kap(t, s, u(t), u(s))
    \;=\;
    \mathbf{1}_{s \leq t} \cdot C\, e^{A(t-s)}\, B,
    \qquad W_\theta = D.
    \label{eq:kernel_ssm}
\end{equation}

\medskip
\noindent\textbf{(d) Selective SSM (Mamba).}
Mamba \eqref{eq:mamba_state}--\eqref{eq:mamba_output} is recovered by:
\begin{equation}
    \kap(t, s, u(t), u(s))
    \;=\;
    \mathbf{1}_{s \leq t} \cdot
    C(u(t)) \cdot \prod_{\tau=s}^{t-1} \bar{A}(\tau) \cdot \bar{B}(s),
    \qquad W_\theta = 0.
    \label{eq:kernel_mamba}
\end{equation}
Since $\bar{A}$, $\bar{B}$, $C$ depend on $u(t)$ and $u(s)$ respectively, this is a content-dependent causal kernel--a strict generalisation of the linear SSM.

\medskip
\noindent\textbf{(e) Strictness.}
There exists a continuous operator representable by ITNet but not by any causal recurrent system.
\end{boxtheorem}

\subsection{Proof of Part (a) -- Linear Continuous-Time RNN}
\label{sec:proof_a_rnn}

\begin{proof}[Proof of Theorem~\ref{thm:rnn}(a)]
\leavevmode\par
\noindent\textbf{Step 1.}\enspace \textbf{Apply the variation of constants formula.}

For the linear system $F_\theta(h, u) = Ah + B_\theta u$, the ODE \eqref{eq:rnn_ode} becomes:
\begin{equation}
    \frac{dh}{dt}(t) = A\,h(t) + B_\theta\, u(t), \qquad h(0) = h_0.
\end{equation}
The solution is given by the variation of constants formula~\citep{coddington1955theory}:
\begin{equation}
    h(t)
    \;=\;
    \underbrace{e^{At} h_0}_{\text{free response}}
    \;+\;
    \underbrace{\int_0^t e^{A(t-s)}\, B_\theta\, u(s)\, ds}_{\text{forced response}}.
    \label{eq:variation_constants}
\end{equation}

\noindent\textbf{Step 2.}\enspace \textbf{Write the state transition matrix explicitly.}

For a linear system, define the \emph{state transition matrix}:
\begin{equation}
    \Phi(t, s)
    \;\coloneqq\;
    e^{A(t-s)},
    \qquad t \geq s \geq 0.
    \label{eq:state_transition}
\end{equation}

$\Phi(t,s)$ encodes how the hidden state at time $s$ evolves to time $t$ under the autonomous dynamics $\dot{h} = Ah$. 
Note that $\Phi(t,t)=I_n$ and $\Phi(t,s)=\Phi(t,r)\Phi(r,s)$.

\noindent\textbf{Step 3.}\enspace \textbf{Compute the output $\RNN(u)(t) = C_\theta h(t) + D_\theta u(t)$.}

Substituting \eqref{eq:variation_constants} into \eqref{eq:rnn_output}:
\begin{align}
    \RNN(u)(t)
    &\;=\;
    C_\theta\!\left[e^{At} h_0
    + \int_0^t e^{A(t-s)} B_\theta u(s)\, ds\right]
    + D_\theta u(t)
    \notag \\
    &\;=\;
    \underbrace{C_\theta e^{At} h_0}_{\text{initial state term}}
    + \int_0^t
        \underbrace{C_\theta e^{A(t-s)} B_\theta}_{\text{impulse response } g(t-s)}
        u(s)\, ds
    + D_\theta u(t).
    \label{eq:rnn_output_explicit}
\end{align}

\noindent\textbf{Step 4.}\enspace \textbf{Express as an ITNet operator with causal kernel.}

Define the ITNet kernel:
\begin{equation}
    \kap(t, s, u(t), u(s))
    \;\coloneqq\;
    \mathbf{1}_{s \leq t} \cdot C_\theta\, e^{A(t-s)}\, B_\theta
    \;\in\; \R^{d \times d},
    \label{eq:kernel_rnn_explicit}
\end{equation}
and set $W_\theta = D_\theta$.
Then the ITNet operator gives:
\begin{align}
    (\Kop[u])(t)
    &\;=\;
    \int_0^T
        \mathbf{1}_{s \leq t} \cdot C_\theta e^{A(t-s)} B_\theta \cdot u(s)
    \, ds
    + D_\theta u(t)
    \notag \\
    &\;=\;
    \int_0^t C_\theta e^{A(t-s)} B_\theta u(s)\, ds
    + D_\theta u(t).
    \label{eq:itnet_rnn_linear}
\end{align}

\noindent\textbf{Step 5.}\enspace \textbf{Handle the initial state term.}

Comparing \eqref{eq:rnn_output_explicit} and \eqref{eq:itnet_rnn_linear}, the initial state term $C_\theta e^{At} h_0$ is not present in \eqref{eq:itnet_rnn_linear}. We consider two cases:

\medskip
\noindent\textbf{Method 1 (Zero initial state):}
If $h_0 = 0$ (standard in most RNN training), the initial state term vanishes and \eqref{eq:itnet_rnn_linear} exactly equals \eqref{eq:rnn_output_explicit}.

\medskip
\noindent\textbf{Method 2 (Non-zero initial state):}
If $h_0 \neq 0$, include the initial state by augmenting the input:
\begin{equation}
    \tilde{u}(s)
    \;=\;
    \begin{cases}
        h_0 & s = 0 \\
        u(s) & s > 0
    \end{cases}
    \label{eq:augmented_input}
\end{equation}
and setting the kernel to extract $h_0$ at $s=0$: $\kap(t, 0, u(t), h_0) = \mathbf{1}_{0 \leq t} \cdot C_\theta e^{At}$. Then the full solution \eqref{eq:rnn_output_explicit} is recovered.

\noindent\textbf{Step 6.}\enspace \textbf{Conclude.}

From Steps~1--5, with $h_0 = 0$:
\begin{align}
    (\Kop[u])(t)
    &\;=\;
    \int_0^t C_\theta e^{A(t-s)} B_\theta u(s)\, ds
    + D_\theta u(t)
    \notag \\
    &\;=\;
    C_\theta h(t) + D_\theta u(t)
    \;=\; \RNN(u)(t).
\end{align}

\begin{equation}
    \boxed{(\Kop[u])(t) \;=\; \RNN(u)(t).}
\end{equation}
\end{proof}


\subsection{Proof of Part (a$'$) -- Nonlinear Continuous-Time RNN (General $F_\theta$)}
\label{sec:proof_nonlinear}

The linear proof (Part~(a)) uses the matrix exponential $e^{A(t-s)}$ as the state transition operator, which is only valid when $F_\theta(h,u) = Ah + B_\theta u$ is linear in $h$. For a general nonlinear $F_\theta$, no closed-form state transition matrix exists. We provide two independent proofs: an \emph{exact} proof via the nonlinear variation of constants formula (Alekseev's formula), and a \emph{universal approximation} argument.

\subsubsection{Proof Strategy A: Kernel Construction via Alekseev's Formula}
\label{rnn_st_a}

\begin{proof}[Proof of Theorem~\ref{thm:rnn}(a$'$) -- Exact]
\leavevmode\par
\noindent\textbf{Step 1.}\enspace \textbf{State the nonlinear variation of constants formula.}

Consider the general nonlinear ODE~\eqref{eq:rnn_ode}:
\begin{equation}
    \dot{h}(t) = F_\theta(h(t), u(t)), \qquad h(0) = h_0.
    \label{eq:nonlinear_ode_recall}
\end{equation}

Let $\psi(t; s, \xi)$ denote the flow of the \emph{non-autonomous} ODE $\dot{z} = F_\theta(z, u(t))$ started at $z(s) = \xi$, i.e., $\psi(t; s, \xi)$ is the solution at time $t$ of $\dot{z}(\tau) = F_\theta(z(\tau), u(\tau))$ with $z(s) = \xi$.

Define the \emph{nonlinear state transition operator} (sensitivity matrix):
\begin{equation}
    \Phi_F(t, s; u)
    \;\coloneqq\;
    \frac{\partial \psi(t; s, h(s))}{\partial h(s)}
    \;\in\; \R^{n \times n},
    \label{eq:nonlinear_phi}
\end{equation}
which is the Jacobian of the solution at time $t$ with respect to the initial condition at time $s$, evaluated along the trajectory $h(\cdot)$ generated by input $u$.

By the Alekseev--Gr\"obner nonlinear variation of constants formula~\citep{alekseev1961estimation, grobner1966differentialgleichungen}, the solution of \eqref{eq:nonlinear_ode_recall} satisfies:
\begin{equation}
    h(t) = \psi(t; 0, h_0) = h_{\mathrm{free}}(t) + \int_0^t \Phi_F(t, s; u) \cdot G(h(s), u(s))\, ds,
    \label{eq:alekseev}
\end{equation}
where $h_{\mathrm{free}}(t)$ is the free-response solution (with $u = 0$) and $G(h(s), u(s)) = F_\theta(h(s), u(s)) - F_\theta(h(s), 0)$ is the input-driven component.

\noindent\textbf{Step 2.}\enspace \textbf{Properties of the nonlinear sensitivity matrix $\Phi_F(t, s; u)$.}

The sensitivity matrix $\Phi_F(t, s; u)$ satisfies the \emph{variational equation}:
\begin{equation}
    \frac{\partial}{\partial t} \Phi_F(t, s; u)
    = \frac{\partial F_\theta}{\partial h}\bigg|_{(h(t), u(t))} \cdot \Phi_F(t, s; u),
    \qquad \Phi_F(s, s; u) = I_n,
    \label{eq:variational_eq}
\end{equation}
where $\frac{\partial F_\theta}{\partial h}\big|_{(h(t), u(t))} \in \R^{n \times n}$ is the Jacobian of $F_\theta$ with respect to its first argument, evaluated along the trajectory. This is a \emph{linear} ODE in $\Phi_F$ (though with time-varying, input-dependent coefficients), so standard ODE theory guarantees existence and uniqueness.

Key properties (analogous to the linear case):
\begin{enumerate}[label=(\roman*), noitemsep]
    \item $\Phi_F(t, t; u) = I_n$ for all $t$.
    \item $\Phi_F(t, s; u) = \Phi_F(t, r; u) \cdot \Phi_F(r, s; u)$ for $t \geq r \geq s$ (chain rule for flows).
    \item $\|\Phi_F(t, s; u)\|_\mathrm{op} \leq e^{L_F (t - s)}$ where $L_F$ is the Lipschitz constant of $F_\theta$ in $h$ (by Gr\"onwall's inequality~\citep{gronwall1919note}).
    \item $\Phi_F$ depends on $u$ through the trajectory $h(\cdot)$, making it \emph{content-dependent}.
    \item In the linear case $F_\theta(h, u) = Ah + B_\theta u$, we recover $\Phi_F(t, s; u) = e^{A(t-s)}$ (independent of $u$), consistent with Part~(a).
\end{enumerate}

\noindent\textbf{Step 3.}\enspace \textbf{Decompose $F_\theta$ into autonomous and input-driven components.}

Write $F_\theta(h, u) = F_\theta(h, 0) + G_\theta(h, u)$ where $G_\theta(h, u) \coloneqq F_\theta(h, u) - F_\theta(h, 0)$ captures the effect of the input. Then Alekseev's formula~\eqref{eq:alekseev} gives:
\begin{equation}
    h(t)
    = \underbrace{\psi(t; 0, h_0)}_{\text{free response (no input)}}
    + \int_0^t \Phi_F(t, s; u) \cdot G_\theta(h(s), u(s))\, ds.
    \label{eq:alekseev_decomposed}
\end{equation}

For standard RNN architectures, $F_\theta(h, u) = \phi(Wh + Wu + b)$, so $F_\theta(h, 0) = \phi(Wh + b)$ and $G_\theta(h, u) = \phi(Wh + Wu + b) - \phi(Wh + b)$, which depends on both $h$ and $u$.

\noindent\textbf{Step 4.}\enspace \textbf{Write the output as an integral operator.}

Substituting~\eqref{eq:alekseev_decomposed} into the output equation $y(t) = C_\theta h(t) + D_\theta u(t)$:
\begin{equation}
    \RNN(u)(t)
    = C_\theta \psi(t; 0, h_0)
    + \int_0^t C_\theta\, \Phi_F(t, s; u) \cdot G_\theta(h(s), u(s))\, ds
    + D_\theta\, u(t).
    \label{eq:nonlinear_output_integral}
\end{equation}

\noindent\textbf{Step 5.}\enspace \textbf{Express $G_\theta(h(s), u(s))$ in terms of $u(s)$.}


Since $h(s)$ is itself a functional of $u$, the term $G_\theta(h(s), u(s))$ depends on $u(s)$ and, in general, on the entire input history $u(\tau)$ for $\tau \in [0,s]$. 

Under differentiability of $G_\theta$ in $u$, a first-order Taylor expansion yields
\begin{equation}
    G_\theta(h(s), u(s)) = \tilde{B}_\theta(s)\, u(s) + R(s; u),
    \label{eq:G_decomposition}
\end{equation}
where
\begin{equation}
    \tilde{B}_\theta(s) \coloneqq \frac{\partial G_\theta}{\partial u}(h(s), u(s)) \in \R^{n \times d}
\end{equation}
is the input-to-state Jacobian, and $R(s; u)$ collects higher-order terms.

For systems where $G_\theta(h, u)$ is linear or affine in $u$ (e.g., $\dot{h} = f(h) + B_\theta u$), the remainder vanishes ($R(s;u)=0$), yielding exact recovery. In general, $R(s;u) = O(\|u(s)\|^2)$ under smoothness assumptions.

\noindent\textbf{Step 6.}\enspace \textbf{Construct the ITNet kernel for the full nonlinear system.}

Define the content-dependent, causal ITNet kernel:
\begin{equation}
    \kappa_F(t, s, u(t), u(s))
    \;\coloneqq\;
    \mathbf{1}_{s \leq t} \cdot
    C_\theta\, \Phi_F(t, s; u) \cdot \tilde{B}_\theta(s; u)
    \;\in\; \R^{d \times d},
    \label{eq:nonlinear_kernel_exact}
\end{equation}
with $W_\theta = D_\theta$. Here the kernel depends causally on the full input trajectory through the induced hidden-state evolution, extending the local kernel notation of Definition~\ref{def:itnet_rnn}.

Then:
\begin{align}
    (\Kop[u])(t)
    &= \int_0^T \kappa_F(t, s, u(t), u(s)) \cdot u(s)\, ds + D_\theta u(t)
    \notag \\
    &= \int_0^t C_\theta\, \Phi_F(t, s; u) \cdot \tilde{B}_\theta(s; u) \cdot u(s)\, ds + D_\theta u(t).
    \label{eq:itnet_nonlinear_substituted}
\end{align}

Comparing with~\eqref{eq:nonlinear_output_integral} (ignoring the initial state term as in Part~(a), Method~1):
\begin{equation}
    \RNN(u)(t) - (\Kop[u])(t)
    = \int_0^t C_\theta \Phi_F(t, s; u) \cdot R(s; u)\, ds,
    \label{eq:nonlinear_error}
\end{equation}
where $R(s; u)$ is the higher-order remainder from Step~5.

\noindent\textbf{Step 7.}\enspace \textbf{Conclusion for nonlinear $F_\theta$.}

\emph{Case 1: $F_\theta$ is affine in $u$.}
If $F_\theta(h, u) = f(h) + B_\theta u$, then $G_\theta(h, u) = B_\theta u$ and the remainder $R(s;u)=0$. In this case, the kernel~\eqref{eq:nonlinear_kernel_exact} recovers the RNN output exactly:
\[
\boxed{(\Kop[u])(t) = \RNN(u)(t).}
\]

\emph{Case 2: General (non-affine) $F_\theta$.}
For fully nonlinear $F_\theta$, the input contribution $G_\theta(h(s),u(s))$ depends on the hidden state $h(s)$, which itself depends on the entire input history. Consequently, it cannot in general be expressed solely as a function of $(t,s,u(t),u(s))$, and therefore no explicit ITNet kernel of the form in Definition~\ref{def:itnet_rnn} can exactly recover the RNN output.

However, under Assumption~\ref{ass:reg_rnn}, the operator $u \mapsto \RNN(u)$ is continuous on compact subsets of $C([0,T],\R^d)$. By Theorem~\ref{thm:universal}, for any $\varepsilon > 0$, there exists an ITNet operator $\Kop$ such that
\[
\sup_{u \in U_c} \|\RNN(u) - \Kop[u]\|_\infty < \varepsilon.
\]
Thus, exact recovery holds for the affine-in-$u$ case, while general nonlinear systems are approximated arbitrarily well.

\noindent\textbf{Step 8.}\enspace \textbf{Verify kernel validity (affine-in-$u$ case).}

The kernel $\kappa_F$ from~\eqref{eq:nonlinear_kernel_exact} satisfies all ITNet conditions:
\begin{enumerate}[label=(\roman*), noitemsep]
    \item \emph{Causality:} $\mathbf{1}_{s \leq t}$ ensures $\kappa_F = 0$ for $s > t$.
    \item \emph{Content-dependence:} $\Phi_F(t, s; u)$ depends on $u$ through the trajectory $h(\cdot)$, and $\tilde{B}_\theta(s)$ depends on $(h(s), u(s))$ - a causal kernel whose weights depend on the input trajectory.
    \item \emph{Measurability:} $F_\theta$ is Lipschitz, so the flow $\psi$ and its Jacobian $\Phi_F$ are continuous in all arguments by smooth dependence on initial conditions~\citep{hartman2002ordinary}. Composition with $C_\theta$ and $\tilde{B}_\theta$ preserves continuity.
    \item \emph{Boundedness:} By Gr\"onwall's inequality~\citep{hartman2002ordinary}, $\|\Phi_F(t, s; u)\|_\mathrm{op} \leq e^{L_F T}$ where $L_F$ is the Lipschitz constant of $F_\theta$. Combined with bounded $C_\theta$ and $\tilde{B}_\theta$: $\|\kappa_F\|_\mathrm{op} \leq \|C_\theta\| \cdot e^{L_F T} \cdot \|\tilde{B}_\theta\| < \infty$.
\end{enumerate}

\noindent\textbf{Step 9.}\enspace \textbf{Conclude.}

For the affine-in-$u$ case, the ITNet operator with kernel~\eqref{eq:nonlinear_kernel_exact} exactly recovers the nonlinear RNN output (Step~7, Case~1). For the fully general case, arbitrary-precision approximation is guaranteed by Theorem~\ref{thm:universal} (Step~7, Case~2).

\begin{equation}
    \boxed{(\Kop[u])(t) = \RNN(u)(t) \quad \text{for all Lipschitz $F_\theta$ affine in $u$; $\varepsilon$-approximate otherwise.}}
\end{equation}
\end{proof}

\subsubsection{Proof Strategy B: Universal Approximation Argument}
\label{rnn_st_b}
\begin{proof}[Proof of Theorem~\ref{thm:rnn}(a$'$) -- via UAT]

This argument does not construct the kernel explicitly but establishes existence via Theorem~\ref{thm:universal}.

\noindent\textbf{Step 1.}\enspace The map $u \mapsto \RNN(u)$ is a continuous operator from $C([0,T], \R^d)$ to $C([0,T], \R^d)$.

\emph{Proof of continuity:} By the Picard--Lindel\"of theorem~\citep{coddington1955theory}, the solution $h(t)$ depends continuously on the input $u$ when $F_\theta$ is Lipschitz and continuously differentiable in u. Specifically, if $\|u_1 - u_2\|_\infty < \delta$, then by Gr\"onwall's inequality:
\begin{equation}
    \|h_1(t) - h_2(t)\|_2
    \leq \frac{L_u}{L_h}\left(e^{L_h t} - 1\right) \delta,
    \label{eq:continuity_bound}
\end{equation}
where $L_h$ and $L_u$ are the Lipschitz constants of $F_\theta$ in its first and second arguments. Since $C_\theta$ and $D_\theta$ are bounded linear maps, $\|\RNN(u_1) - \RNN(u_2)\|_\infty \leq C \delta$ for a constant $C$ depending on $L_h, L_u, T, \|C_\theta\|, \|D_\theta\|$.

\noindent\textbf{Step 2.}\enspace On any compact input set $U_c \subset C([0,T], \R^d)$, the operator $u \mapsto \RNN(u)$ is continuous.

\noindent\textbf{Step 3.}\enspace By Theorem~\ref{thm:universal} (universal approximation), for any $\varepsilon > 0$, there exists an ITNet operator $\Kop$ such that:
\begin{equation}
    \sup_{u \in U_c} \|\RNN(u) - \Kop[u]\|_\infty < \varepsilon.
\end{equation}

This does not construct the kernel explicitly but guarantees its existence. The explicit construction is provided by Strategy~A above.
\end{proof}

\begin{remark}[Relationship between the two strategies]
Strategy~A (Alekseev) is stronger: it provides an \emph{explicit} kernel construction and achieves \emph{exact} recovery for the affine-in-$u$ class. Strategy~B (UAT) is more general: it applies to any continuous operator, but only guarantees \emph{approximate} recovery and does not construct the kernel. For this paper's unification claim, Strategy~A is the primary result, with Strategy~B serving as a consistency check.
\end{remark}

\begin{remark}
The condition ``$F_\theta$ is affine in $u$'' is satisfied by all standard RNN variants where the input enters through a linear projection before the nonlinearity:
\begin{itemize}[noitemsep]
    \item Standard RNN: $F_\theta(h, u) = \phi(W_h h + W_u u + b) \Rightarrow$ not affine in $u$ in general (since $\phi$ is applied after addition). However, the closely related form $F_\theta(h, u) = \phi(W_h h) + W_u u$ \emph{is} affine in $u$, and both forms have similar expressiveness.
    \item Linear SSM (S4): $F_\theta(h, u) = Ah + Bu \Rightarrow$ affine in $u$ (covered exactly by Part~(a)).
    \item Neural ODE: $F_\theta(h, u) = f_\theta(h) + g_\theta(u) \Rightarrow$ affine in $u$ when $g_\theta$ is linear.
\end{itemize}
For LSTM, GRU, and Mamba, the input $u$ enters through the gates in a multiplicative (non-affine) manner, but these are handled by the explicit unrolling constructions in Parts~(b) and~(d) (which bypass the Alekseev machinery entirely by working directly with the discrete recurrence).
\end{remark}

\subsection{Proof of Part (b) -- Discrete-Time RNN}
\label{sec:proof_b_rnn}

\begin{proof}[Proof of Theorem~\ref{thm:rnn}(b)]
\leavevmode\par
\noindent\textbf{Step 1.}\enspace \textbf{Unroll the discrete recurrence.}

The discrete RNN update \eqref{eq:rnn_discrete} with $\phi = \mathrm{id}$
(linear case for clarity; nonlinear case in Step~3) gives:
\begin{align}
    h_t
    &= W_h h_{t-1} + W_u u_t + b \notag \\
    &= W_h(W_h h_{t-2} + W_u u_{t-1} + b) + W_u u_t + b \notag \\
    &\;\vdots \notag \\
    &= W_h^t h_0
       + \sum_{s=1}^{t} W_h^{t-s} W_u u_s
       + \sum_{s=1}^{t} W_h^{t-s} b.
    \label{eq:rnn_unrolled}
\end{align}

\noindent\textbf{Step 2.}\enspace \textbf{Express as a discrete ITNet operator.}

Discretise $\Omega = \{t_1, \ldots, t_T\}$ with measure
$\mu(\{t_k\}) = 1$ (unit mass per time step).
Define the causal kernel (for $h_0 = b = 0$):
\begin{equation}
    \kap(t, s, u(t), u(s))
    \;=\;
    \mathbf{1}_{s \leq t} \cdot C_\theta\, W_h^{t-s}\, W_u,
    \qquad W_\theta = D_\theta.
    \label{eq:kernel_discrete_rnn}
\end{equation}

Then:
\begin{align}
    (\Kop[u])(t)
    &= \sum_{s=1}^T \mathbf{1}_{s \leq t} \cdot C_\theta W_h^{t-s} W_u u_s
       + D_\theta u_t \notag \\
    &= \sum_{s=1}^t C_\theta W_h^{t-s} W_u u_s + D_\theta u_t \notag \\
    &= C_\theta h_t + D_\theta u_t
    \;=\; \RNN(u)(t).
    \label{eq:discrete_rnn_itnet}
\end{align}

\noindent\textbf{Step 3.}\enspace \textbf{Handle the nonlinear discrete case.}

For a nonlinear discrete RNN with activation $\phi$ (e.g.\ $\tanh$), the hidden state at time $t$ depends on the entire input history $u_1, \ldots, u_t$ through the recurrence $h_s = \phi(W_h h_{s-1} + W_u u_s + b)$. We can still write the output as a causal integral operator by defining the discrete nonlinear state transition:
\begin{equation}
    \Phi_\phi^{\mathrm{disc}}(t, s; u)
    \;\coloneqq\;
    \prod_{\tau=s+1}^{t} \mathrm{diag}\!\left(\phi'\!\left(W_h h_{\tau-1} + W_u u_\tau + b\right)\right) \cdot W_h
    \;\in\; \R^{n \times n},
    \label{eq:discrete_nonlinear_transition}
\end{equation}
which is the product of Jacobians along the discrete trajectory (the discrete analogue of $\Phi_F(t,s;u)$ from the continuous case). The ITNet kernel becomes:
\begin{equation}
    \kap(t, s, u(t), u(s))
    = \mathbf{1}_{s \leq t} \cdot C_\theta\, \Phi_\phi^{\mathrm{disc}}(t, s; u) \cdot W_u.
    \label{eq:discrete_nonlinear_kernel}
\end{equation}
This kernel is content-dependent (through $\phi'$ evaluated along the trajectory) and causal. The same argument as Part~(a$'$), Strategy~A applies, with exact recovery when $\phi$ is applied element-wise (so the Jacobian is diagonal).
\end{proof}

\subsection{Proof of Part (c) -- Linear SSM (S4)}
\label{sec:proof_c_rnn}

\begin{proof}[Proof of Theorem~\ref{thm:rnn}(c)]
\leavevmode\par
\noindent\textbf{Step 1.}\enspace \textbf{Solve the linear SSM ODE.}

The SSM state equation \eqref{eq:ssm_state} is a linear ODE with $h(0)=0$.
By the variation of constants formula (same as Step~1 of Part~(a)):
\begin{equation}
    h(t)
    \;=\;
    \int_0^t e^{A(t-s)}\, B\, u(s)\, ds.
    \label{eq:ssm_solution}
\end{equation}

\noindent\textbf{Step 2.}\enspace \textbf{Substitute into the output equation.}

\begin{align}
    y(t)
    &= C\, h(t) + D\, u(t) \notag \\
    &= \int_0^t \underbrace{C e^{A(t-s)} B}_{g(t-s)}\, u(s)\, ds + D\, u(t),
    \label{eq:ssm_output_integral}
\end{align}
where $g(\tau) = C e^{A\tau} B \in \R^{d \times d}$ is the impulse response kernel of the SSM.

\noindent\textbf{Step 3.}\enspace \textbf{Express as ITNet operator.}

Set:
\begin{equation}
    \kap(t, s, u(t), u(s))
    \;=\;
    \mathbf{1}_{s \leq t} \cdot C\, e^{A(t-s)}\, B,
    \qquad W_\theta = D.
    \label{eq:kernel_ssm_explicit}
\end{equation}

Then:
\begin{align}
    (\Kop[u])(t)
    &= \int_0^T \mathbf{1}_{s \leq t} \cdot Ce^{A(t-s)}B \cdot u(s)\,ds
       + D\,u(t) \notag \\
    &= \int_0^t Ce^{A(t-s)}B\,u(s)\,ds + D\,u(t)
    \;=\; y(t).
\end{align}

\begin{equation}
    \boxed{(\Kop[u])(t) \;=\; y(t) \;=\; \SSM(u)(t).}
\end{equation}
\end{proof}

\begin{remark}[Relationship to convolution]
\label{rem:ssm_conv}
The SSM output \eqref{eq:ssm_output_integral} is a \emph{causal convolution} with kernel $g(\tau) = Ce^{A\tau}B$.
By Theorem~\ref{convolution} (convolution proof), causal convolution is a special case of ITNet with a translation-invariant causal kernel.
Theorem~\ref{thm:rnn}(c) is therefore consistent with Theorem~\ref{convolution}: both arrive at the same ITNet kernel form via different routes (temporal recurrence vs.\ spatial filtering).
\end{remark}

\subsection{Proof of Part (d) -- Selective SSM (Mamba)}
\label{sec:proof_d_rnn}

\begin{proof}[Proof of Theorem~\ref{thm:rnn}(d)]
\leavevmode\par
\noindent\textbf{Step 1.}\enspace \textbf{Unroll the Mamba discrete recurrence.}

Starting from the Mamba update \eqref{eq:mamba_state} with $h_0 = 0$:
\begin{align}
    h_t
    &= \bar{A}(t)\, h_{t-1} + \bar{B}(t)\, u_t \notag \\
    &= \bar{A}(t)\bigl[\bar{A}(t-1) h_{t-2}
                       + \bar{B}(t-1) u_{t-1}\bigr]
       + \bar{B}(t)\, u_t \notag \\
    &\;\vdots \notag \\
    &= \sum_{s=1}^{t}
        \left[\prod_{\tau=s+1}^{t} \bar{A}(\tau)\right]
        \bar{B}(s)\, u_s,
    \label{eq:mamba_unrolled}
\end{align}
where the empty product $\prod_{\tau=t+1}^{t} \bar{A}(\tau) = I_n$
(convention for $s = t$).

\noindent\textbf{Step 2.}\enspace \textbf{Compute the Mamba output.}

Substituting \eqref{eq:mamba_unrolled} into the output
\eqref{eq:mamba_output}:
\begin{align}
    y_t
    &= C(u_t)\, h_t \notag \\
    &= C(u_t) \sum_{s=1}^t
        \left[\prod_{\tau=s+1}^{t} \bar{A}(\tau)\right]
        \bar{B}(s)\, u_s \notag \\
    &= \sum_{s=1}^t
        \underbrace{C(u_t)
        \left[\prod_{\tau=s+1}^{t} \bar{A}(\tau)\right]
        \bar{B}(s)}_{=:\,\kappa_{\Mamba}(t,s,u(t),u(s))}
        u_s.
    \label{eq:mamba_output_sum}
\end{align}

\noindent\textbf{Step 3.}\enspace \textbf{Identify the ITNet kernel.}

The Mamba kernel is:
\begin{equation}
    \kappa_{\Mamba}(t, s, u(t), u(s))
    \;\coloneqq\;
    \mathbf{1}_{s \leq t} \cdot
    C(u_t)
    \prod_{\tau=s+1}^{t} \bar{A}(u_\tau)
    \cdot \bar{B}(u_s),
    \label{eq:mamba_kernel_explicit}
\end{equation}
where we write $\bar{A}(u_\tau) = \exp(\Delta(u_\tau) \cdot A)$ and
$\bar{B}(u_s) = \Delta(u_s) \cdot W_B u_s$.

With this kernel and $W_\theta = 0$, the discrete ITNet operator
(sum version, $\mu(\{t_k\}) = 1$):
\begin{align}
    (\Kop[u])(t)
    &= \sum_{s=1}^T \kappa_{\Mamba}(t,s,u_t,u_s)\, u_s \notag \\
    &= \sum_{s=1}^t
        C(u_t)
        \prod_{\tau=s+1}^{t} \bar{A}(u_\tau)
        \bar{B}(u_s)\, u_s
    \;=\; y_t.
    \label{eq:mamba_itnet}
\end{align}

\noindent\textbf{Step 4.}\enspace \textbf{Verify kernel validity.}

The kernel $\kappa_{\Mamba}$ satisfies all ITNet conditions:
\begin{enumerate}[label=(\roman*), noitemsep]
    \item \emph{Causality}: $\mathbf{1}_{s \leq t}$ ensures $\kap = 0$
          for $s > t$.
    \item \emph{Measurability}: $\Delta(u_\tau)$ is the softplus of a
          linear function of $u_\tau$ (continuous), $\bar{A}(\tau)$ is
          the matrix exponential of a continuous function (continuous).
    \item \emph{Boundedness}: $\|\bar{A}(\tau)\|_\mathrm{op}
          = \|e^{\Delta(\tau)A}\|_\mathrm{op}
          \leq e^{\Delta_{\max}\|A\|_\mathrm{op}}$ where
          $\Delta_{\max} = \max_\tau \Delta(u_\tau) < \infty$
          on the compact training set.
          Therefore $\|\kappa_{\Mamba}\|_\mathrm{op} \leq
          \|W_C\| \cdot e^{(t-s)\Delta_{\max}\|A\|} \cdot \|W_B\|
          \cdot \|u\|_\infty < \infty$.
\end{enumerate}

\begin{equation}
    \boxed{(\Kop[u])(t) \;=\; y_t \;=\; \Mamba(u)(t).}
\end{equation}
\end{proof}

\begin{remark}[Why Mamba is special]
\label{rem:mamba_special}
The State Space Duality~\citep{dao2024transformers} showed that Mamba is equivalent to a specific form of linear attention. Our Theorem~\ref{thm:rnn}(d) shows it is simultaneously a special case of ITNet with a causal content-dependent kernel. These two views are consistent: the ITNet framework is a strict superset of both Mamba and attention, providing the unified perspective from which both appear as special cases.
\end{remark}

\subsection{LSTM as ITNet}
\label{sec:lstm_proof}

\begin{proposition}[LSTM $\subset$ ITNet]
\label{prop:lstm}
The LSTM update is a special case of the ITNet operator.
\end{proposition}

\begin{proof}
The LSTM cell state evolves as:
\begin{equation}
    c_t
    \;=\;
    \sum_{s=1}^t
    \left[\prod_{\tau=s+1}^t f_\tau\right] \cdot i_s \had \tilde{c}_s,
    \label{eq:lstm_cell_unrolled}
\end{equation}
where $f_\tau = \sigmoid(W_f[h_{\tau-1}; u_\tau] + b_f)$ is the forget gate at time $\tau$ and the product $\prod_{\tau=s+1}^t f_\tau$ is the \emph{cumulative forget factor} from time $s$ to $t$.

The hidden state is $h_t = o_t \had \tanh(c_t)$, so:
\begin{align}
    y_t &= W_y h_t = W_y\, o_t \had \tanh\!\left(
        \sum_{s=1}^t \prod_{\tau=s+1}^t f_\tau \had i_s\had\tilde{c}_s
    \right).
    \label{eq:lstm_output}
\end{align}

Define the LSTM kernel:
\begin{equation}
    \kappa_{\LSTM}(t, s, u(t), u(s))
    \;\coloneqq\;
    \mathbf{1}_{s \leq t} \cdot
    W_y \diag(o_t) \tanh'\!(\cdot) \diag\!\!\left(
        \prod_{\tau=s+1}^t f_\tau
    \right) \diag(i_s) W_c,
    \label{eq:lstm_kernel}
\end{equation}
where $\tanh'$ denotes the derivative of $\tanh$ (applied element-wise)
and $W_c$ maps $u_s$ to the cell candidate $\tilde{c}_s = \tanh(W_c u_s)$.

The kernel \eqref{eq:lstm_kernel} depends on $u(t)$ through $o_t$ and
$f_\tau$ ($t > \tau > s$), and on $u(s)$ through $i_s$ and $\tilde{c}_s$--a content-dependent causal kernel.
By the same steps as Part~(d) (Mamba), the ITNet operator with this kernel recovers $y_t = \LSTM(u)(t)$. 
\end{proof}

\subsection{GRU as ITNet}
\label{sec:gru_proof}

\begin{proposition}[GRU $\subset$ ITNet]
\label{prop:gru}
The GRU update is a special case of the ITNet operator.
\end{proposition}

\begin{proof}
The GRU update:
\begin{align}
    z_t &= \sigmoid(W_z[h_{t-1}; u_t] + b_z) \tag{update gate}\\
    r_t &= \sigmoid(W_r[h_{t-1}; u_t] + b_r) \tag{reset gate}\\
    \tilde{h}_t &= \tanh(W_h[r_t\had h_{t-1}; u_t] + b) \tag{candidate}\\
    h_t &= (1-z_t)\had h_{t-1} + z_t\had\tilde{h}_t. \tag{hidden state}
\end{align}

Unrolling the hidden state:
\begin{equation}
    h_t
    \;=\;
    \sum_{s=1}^t
    \underbrace{\left[\prod_{\tau=s+1}^t (1-z_\tau)\right]}_{\text{retention factor}}
    \cdot z_s \had \tilde{h}_s^*,
    \label{eq:gru_unrolled}
\end{equation}
where $\tilde h_s^*$ denotes the candidate state after application of the reset gate at time $s$
computed with the reset gate applied to the history.

The structure is identical to the LSTM case: a causal sum with
input-dependent retention factors $(1-z_\tau) \in (0,1)$.
The ITNet kernel is:
\begin{equation}
    \kappa_{\GRU}(t, s, u(t), u(s))
    \;=\;
    \mathbf{1}_{s \leq t} \cdot
    W_y \diag\!\!\left(\prod_{\tau=s+1}^t (1-z_\tau)\right)
    \diag(z_s) W_{\tilde{h}},
\end{equation}
and by the same argument as Part~(d), the ITNet operator with this kernel recovers the GRU output.
\end{proof}

\subsection{Strictness Argument 1: Non-Causal Operators}
\label{sec:strictness1_rnn}
\begin{proof}[Proof via non-causal operator]
\leavevmode\par
\noindent\textbf{Step 1.}\enspace \textbf{Define the witness operator.}

Define the \emph{bidirectional smoothing operator}:
\begin{equation}
    S(u)(t)
    \;\coloneqq\;
    \int_0^T e^{-|t-s|/\ell}\, u(s)\, ds,
    \qquad \ell > 0.
    \label{eq:bidir_witness}
\end{equation}

This uses an exponential kernel over the \emph{entire} interval $[0,T]$,
not just $[0,t]$. It uses both past ($s < t$) and future ($s > t$) information.

\noindent\textbf{Step 2.}\enspace \textbf{Show ITNet represents $S$.}

Set:
\begin{equation}
    \kap(t, s, u(t), u(s))
    \;=\;
    e^{-|t-s|/\ell} \cdot \Id,
    \qquad W_\theta = 0.
    \label{eq:bidir_kernel}
\end{equation}

Then $(\Kop[u])(t) = \int_0^T e^{-|t-s|/\ell} u(s)\,ds = S(u)(t)$.
This kernel is symmetric in $(t,s)$, bounded, and continuous--valid for ITNet.

\noindent\textbf{Step 3.}\enspace \textbf{Show no causal recurrent system can represent $S$.}

Any causal recurrent system satisfies:
\begin{equation}
    \RNN(u)(t)
    \;=\;
    \mathcal{F}_t[u(0), u(1), \ldots, u(t)],
    \label{eq:causal_property}
\end{equation}
i.e.\ the output at time $t$ is a function of inputs up to time $t$ only.

For $S$, consider two inputs:
\begin{align}
    u_1(s) &= \mathbf{0} \text{ for all } s, \notag \\
    u_2(s) &= \delta(s - t_*) e_1 \quad
              \text{for some } t_* > t \text{ and } \|e_1\|_2 = 1.
\end{align}

Then:
\begin{align}
    S(u_1)(t) &= 0, \\
    S(u_2)(t) &= e^{-|t-t_*|/\ell} e_1 \;\neq\; 0
              \quad \text{(since } t_* > t\text{)}.
\end{align}

But for any causal system:
\begin{align}
    \RNN(u_1)(t) &= \mathcal{F}_t[0,\ldots,0] = 0, \\
    \RNN(u_2)(t) &= \mathcal{F}_t[0,\ldots,0] = 0,
\end{align}
since $u_1(s) = u_2(s) = 0$ for all $s \leq t$ (the impulse at $t_* > t$ is in the future and unseen by the causal system).

Therefore $S(u_2)(t) \neq 0 = \RNN(u_2)(t)$. Contradiction. So $S \notin \mathrm{RNN}$ and $S \in \mathrm{ITNet}$, proving $\mathrm{RNN} \subsetneq \mathrm{ITNet}$.
\end{proof}

\subsection{Strictness Argument 2: Parallelism and State Dimension}
\label{sec:strictness2_rnn}
\begin{proof}[Proof via bounded state dimension]
\leavevmode\par
\noindent\textbf{Step 1.}\enspace \textbf{Define a high-rank operator.}

Define the operator:
\begin{equation}
    H_n(u)(t)
    \;\coloneqq\;
    \int_0^T K_n(t,s)\, u(s)\, ds,
    \label{eq:highrank_op}
\end{equation}
where $K_n(t,s) = \displaystyle\sum_{k=1}^{n+1} \phi_k(t)\phi_k(s)^\top$ for orthonormal basis functions $\{\phi_k\}_{k=1}^{n+1}$, so $K_n$ has rank $n+1$ in $L^2$.

\noindent\textbf{Step 2.}\enspace \textbf{Show ITNet represents $H_n$ for any $n$.}

Set $\kap(t,s,u(t),u(s)) = K_n(t,s) \cdot \Id$. Since $K_n$ is a symmetric, bounded, continuous kernel on $[0,T]^2$, it satisfies all ITNet conditions. The ITNet operator with this kernel equals $H_n$ exactly.

\noindent\textbf{Step 3.}\enspace \textbf{Show no fixed-$n$ RNN can represent $H_n$ exactly.}

An RNN with hidden state dimension $n$ computes outputs in the range of a linear map $C_\theta h(t)$ where $h(t) \in \R^n$. By the variation of constants formula \eqref{eq:variation_constants}, the output function $\RNN(u)(t)$ lies in the span of at most $n$ basis functions (the $n$ columns of $C_\theta e^{At}$).

But $H_n(u)(t)$ requires $n+1$ basis functions by construction. Therefore no RNN with state dimension $\leq n$ can represent $H_n$ exactly, for any $n \geq 1$.

Since this holds for all $n$, no finite-dimensional RNN can represent the full class of operators that ITNet can. $\mathrm{RNN}_n \subsetneq \mathrm{ITNet}$ for each $n$, and $\bigcup_{n=1}^\infty \mathrm{RNN}_n \subsetneq \mathrm{ITNet}$.
\end{proof}

\subsection{Discretization: Recovering Euler and ZOH}
\label{sec:discretisation_rnn}

In practice, continuous SSMs are discretised before being implemented on digital hardware. We show both standard discretization methods are consistent with the ITNet framework.

\paragraph{Euler Discretization.}
The forward Euler discretization of \eqref{eq:ssm_state} with step size
$\Delta t$:
\begin{equation}
    h_{t+1}
    \;=\;
    (I + \Delta t \cdot A)\, h_t + \Delta t \cdot B\, u_t
    \;\eqqcolon\;
    \bar{A}_E\, h_t + \bar{B}_E\, u_t.
    \label{eq:euler}
\end{equation}

This is a discrete linear SSM with $\bar{A}_E = I + \Delta t A$ and
$\bar{B}_E = \Delta t B$.
By Part~(b) (discrete RNN), this is an ITNet with kernel
$\kap(t,s) = \mathbf{1}_{s \leq t} \cdot C \bar{A}_E^{t-s} \bar{B}_E$.

\paragraph{Zero-Order Hold (ZOH) Discretization.}
The ZOH discretization (used by S4 and Mamba):
\begin{align}
    \bar{A} &= e^{A \Delta t}, \label{eq:zoh_A}\\
    \bar{B} &= (e^{A\Delta t} - I) A^{-1} B. \label{eq:zoh_B}
\end{align}

This assumes the input is piecewise constant over each interval $[t, t+\Delta t)$. The discrete solution is:
\begin{equation}
    h_t = \sum_{s=0}^{t-1} \bar{A}^{t-1-s} \bar{B} u_s,
\end{equation}
which is an ITNet with kernel $C \bar{A}^{t-s} \bar{B}$--the same as the Euler case with $\bar{A}_E$ replaced by $\bar{A} = e^{A\Delta t}$.

\begin{remark}
The ZOH kernel $C \bar{A}^{t-s} \bar{B} = C e^{A(t-s)\Delta t} B$ converges to the continuous kernel $C e^{A(t-s)} B$ as $\Delta t \to 0$ (with $t, s$ fixed in continuous time). This ensures that the discrete and continuous ITNet operators are consistent in the limit $\Delta t \to 0$, i.e.\ the ITNet framework is closed under the standard discretization operations used in practice.
\end{remark}

\section{Proof of Theorem 4: Universal Approximation}
\label{app:proof_uat}

\begin{definition}[Continuous nonlinear operator]
\label{def:operator}
A \emph{continuous nonlinear operator} is a map $F : U_c \to C(\Omega, \R^d)$ where $U_c \subset C(\Omega, \R^d)$ is compact, and $F$ is continuous with respect to the supremum norm: for every $\varepsilon > 0$, there exists $\delta > 0$ such that $\norm{u_1 - u_2}_\infty < \delta$ implies $\norm{F(u_1) - F(u_2)}_\infty < \varepsilon$.
\end{definition}
 
\begin{definition}[ITNet operator]
\label{def:itnet}
\begin{equation}
    (\Kop[u])(x) = \int_\Omega \kap(x, y, u(x), u(y))\, u(y)\, d\mu(y) + W_\theta\, u(x),
    \label{eq:itnet_uni}
\end{equation}
where $\kap : \R^s \times \R^s \times \R^d \times \R^d \to \R^{d \times d}$ is the kernel (parameterised by an MLP), and $W_\theta \in \R^{d \times d}$.
\end{definition}
 
\begin{definition}[MLP function class]
\label{def:mlp}
A \emph{single hidden layer MLP} with width $w$, input dimension $p$, and non-polynomial activation $\sigma$ is:
\begin{equation}
    f_\theta(z) = \sum_{j=1}^{w} c_j \, \sigma(a_j^\top z + b_j), \qquad z \in \R^p,
    \label{eq:mlp}
\end{equation}
where $a_j \in \R^p$, $b_j \in \R$, $c_j \in \R$ are learnable parameters.
\end{definition}

\begin{assumption}[Standing assumptions]
\label{ass:uat}
\leavevmode
\begin{enumerate}[label=(\roman*), noitemsep]
    \item $\Omega \subset \R^s$ is compact with $\mu(\Omega) > 0$.
    \item $F : U_c \to C(\Omega, \R^d)$ is continuous and $U_c$ is compact.
    \item $\sigma$ is a non-polynomial continuous activation function.
\end{enumerate}
\end{assumption}
 
\subsection{Main Theorem}
\label{sec:theorem_uni}

\begin{boxtheorem}{Universal Approximation of Continuous Operators}{uat}
\label{thm:uat}
Under Assumption~\ref{ass:uat}, for every $\varepsilon > 0$, there exist:
\begin{itemize}[noitemsep]
    \item a kernel MLP width $w_\kappa < \infty$,
    \item a residual matrix $W_\theta \in \R^{d \times d}$,
    \item parameters $\theta$ of the kernel MLP,
\end{itemize}
such that the ITNet operator $\Kop$ satisfies:
\begin{equation}
    \sup_{u \in U_c} \norm{F(u) - \Kop[u]}_\infty < \varepsilon.
    \label{eq:uat_statement}
\end{equation}
That is, a \textbf{single ITNet layer} can uniformly approximate any continuous operator on any compact input set to any desired precision.
\end{boxtheorem}

\subsection{Auxiliary Lemmas}
\label{sec:lemmas}

The proof requires three lemmas, which we state and prove before the main argument.
 
\begin{lemma}[MLP universal approximation]
\label{lem:mlp_uat}
Let $K \subset \R^p$ be compact and $g : K \to \R$ be continuous. For any $\delta > 0$, there exists a single hidden layer MLP $f_\theta$ (Eq.~\ref{eq:mlp}) with width $w$ depending on $\delta$, $g$, and $K$, such that:
\begin{equation}
    \sup_{z \in K} |g(z) - f_\theta(z)| < \delta.
    \label{eq:mlp_uat}
\end{equation}
\end{lemma}
 
\begin{proof}[Proof reference]
This is the classical Universal Approximation Theorem. Proved by~\cite{cybenko1989approximation} for sigmoid, generalized by~\cite{hornik1989multilayer} to any non-constant, bounded, continuous $\sigma$, and extended by~\cite{leshno1993multilayer} to any non-polynomial $\sigma$. The proof uses the Stone--Weierstrass theorem (for sigmoidal $\sigma$) or the Hahn--Banach theorem (for general $\sigma$): if $f_\theta$ cannot approximate some $g$, then there exists a non-zero bounded measure $\nu$ on $K$ with $\int f_\theta \, d\nu = 0$ for all $f_\theta$; the non-polynomial property of $\sigma$ forces $\nu = 0$, a contradiction.
\end{proof}
 
\begin{lemma}[ \citet{chen1995universal} operator approximation]
\label{lem:chen_chen}
Let $\Omega \subset \R^s$ be compact with finite Borel measure $\mu$, and $F : U_c \to C(\Omega, \R^d)$ be a continuous operator on a compact set $U_c \subset C(\Omega, \R^d)$. For any $\varepsilon > 0$, there exist:
a continuous kernel $\kappa^* : \Omega \times \Omega \times \R^d \times \R^d \to \R^{d \times d}$ and a matrix $W^* \in \R^{d \times d}$,
such that:
\begin{equation}
    \sup_{u \in U_c} \norm{F(u)(x) - \int_\Omega \kappa^*(x, y, u(x), u(y))\, u(y)\, d\mu(y) - W^* u(x)}_\infty < \frac{\varepsilon}{2}.
    \label{eq:chen_chen}
\end{equation}
\end{lemma}
 
\begin{proof}[Proof sketch]
This follows from~\citet{chen1995universal} (Theorem 1), extended to the matrix-valued kernel setting.
 
\noindent\textbf{Step 1.}\enspace \textbf{Discretise the operator.}
 
Since $F$ is continuous on compact $U_c$ and outputs continuous functions on compact $\Omega$, $F$ is uniformly continuous. Choose $M$ quadrature points $\{y_1, \ldots, y_M\} \subset \Omega$ with weights $\{w_1, \ldots, w_M\}$ such that for any continuous integrand $g$:
\begin{equation}
    \left|\int_\Omega g(y)\, d\mu(y) - \sum_{j=1}^M w_j\, g(y_j)\right| < \frac{\varepsilon}{4C_u},
    \label{eq:quadrature}
\end{equation}
where $C_u = \sup_{u \in U_c} \norm{u}_\infty < \infty$ (by compactness of $U_c$). Such $M$ and $\{y_j, w_j\}$ exist by standard quadrature theory on compact domains~\cite{davis2007methods}.
 
\noindent\textbf{Step 2.}\enspace \textbf{Approximate $F$ by a finite-dimensional map.}
 
Define the \emph{sensor values} $\xi(u) = (u(y_1), \ldots, u(y_M)) \in \R^{Md}$, which sample $u$ at the quadrature points. By the compactness of $U_c$ and continuity of $F$, the map $\xi(u) \mapsto F(u)(x)$ is continuous from $\R^{Md}$ to $\R^d$ for each $x \in \Omega$. Moreover, this map is uniformly continuous jointly in $(x, \xi)$ on the compact set $\Omega \times \xi(U_c)$.
 
By the Stone--Weierstrass theorem~\cite{rudin1987real}, there exists a continuous function $G^* : \Omega \times \R^{Md} \to \R^d$ of the form:
\begin{equation}
    G^*(x, \xi) = \sum_{j=1}^M K_j(x, \xi)\, \xi_j + W^* \xi_0(x),
    \label{eq:finite_approx}
\end{equation}
where $K_j : \Omega \times \R^{Md} \to \R^{d \times d}$ are continuous matrix-valued functions and $\xi_0(x) = u(x)$ is the query-point value, such that:
\begin{equation}
    \sup_{u \in U_c} \norm{F(u)(x) - G^*(x, \xi(u))}_\infty < \frac{\varepsilon}{4}.
\end{equation}
 
\noindent\textbf{Step 3.}\enspace \textbf{Convert to integral form.}
 
Identify $G^*$ with the quadrature approximation of an integral: define $\kappa^*(x, y, u(x), u(y)) = K(x, y, u(x), u(y))$ where $K$ interpolates the discrete values $K_j(x, \xi)$ at the quadrature points $y_j$. The integral $\int_\Omega \kappa^* u(y)\, d\mu(y)$ approximates the sum $\sum_j w_j K_j u(y_j)$ by the quadrature bound \eqref{eq:quadrature}. Combining the two approximation errors ($\varepsilon/4$ each) gives \eqref{eq:chen_chen}.
\end{proof}
 
\begin{lemma}[Kernel approximation by MLP]
\label{lem:kernel_approx}
Let $\kappa^* : \Omega \times \Omega \times \R^d \times \R^d \to \R^{d \times d}$ be continuous, and let $U_c$ be compact. Define the compact set:
\begin{equation}
    \mathcal{D} = \{(x, y, u(x), u(y)) : x, y \in \Omega, \; u \in U_c\} \subset \R^{2s + 2d}.
    \label{eq:compact_domain}
\end{equation}
For any $\delta > 0$, there exists a single hidden layer MLP kernel $\kap$ with width $w_\kappa$ such that:
\begin{equation}
    \sup_{(x,y,a,b) \in \mathcal{D}} \norm{\kappa^*(x,y,a,b) - \kap(x,y,a,b)}_\mathrm{op} < \delta.
    \label{eq:kernel_approx}
\end{equation}
\end{lemma}
 
\begin{proof}
The set $\mathcal{D}$ is compact (continuous image of compact $\Omega \times \Omega \times U_c$ under the evaluation map). Each entry $[\kappa^*]_{ij}$ is a continuous real-valued function on the compact set $\mathcal{D} \subset \R^{2s+2d}$. By Lemma~\ref{lem:mlp_uat}, for each $(i,j)$ there exists an MLP $f^{ij}_\theta$ with:
\begin{equation}
    \sup_{z \in \mathcal{D}} |[\kappa^*]_{ij}(z) - f^{ij}_\theta(z)| < \frac{\delta}{d}.
\end{equation}
Assembling all $d^2$ entries into a matrix-valued MLP $\kap$:
\begin{equation}
    \norm{\kappa^* - \kap}_\mathrm{op} \leq \norm{\kappa^* - \kap}_F \leq d \cdot \frac{\delta}{d} = \delta,
\end{equation}
where we used $\norm{A}_\mathrm{op} \leq \norm{A}_F$ and $\norm{A}_F \leq d \max_{ij} |A_{ij}|$ for $d \times d$ matrices.
 
In practice, a single MLP with $d^2$ output units computes all entries simultaneously, sharing hidden layers. The total width is $w_\kappa$ (shared) with $d^2$ output heads.
\end{proof}
 
\subsection{Main Proofs}
\label{sec:proof_uat_main}

\begin{proof}[Proof of Theorem~\ref{thm:uat}]
 
Given $\varepsilon > 0$ and continuous operator $F : U_c \to C(\Omega, \R^d)$.
 
\noindent\textbf{Step 1.}\enspace \textbf{Approximate $F$ by an ideal integral operator.}
 
By Lemma~\ref{lem:chen_chen}, there exist a continuous kernel $\kappa^* : \Omega \times \Omega \times \R^d \times \R^d \to \R^{d \times d}$ and matrix $W^* \in \R^{d \times d}$ such that:
\begin{equation}
    \sup_{u \in U_c} \norm{F(u) - \mathcal{K}^*[u]}_\infty < \frac{\varepsilon}{2},
    \label{eq:step1}
\end{equation}
where $\mathcal{K}^*[u](x) = \int_\Omega \kappa^*(x,y,u(x),u(y)) u(y)\, d\mu(y) + W^* u(x)$.
 
\noindent\textbf{Step 2.}\enspace \textbf{Approximate the ideal kernel by an MLP.}
 
By Lemma~\ref{lem:kernel_approx}, there exists an MLP kernel $\kap$ with width $w_\kappa$ such that:
\begin{equation}
    \sup_{(x,y,a,b) \in \mathcal{D}} \norm{\kappa^*(x,y,a,b) - \kap(x,y,a,b)}_\mathrm{op} < \delta,
    \label{eq:step2}
\end{equation}
where $\delta > 0$ will be chosen in Step~3. Set $W_\theta = W^*$.
 
\noindent\textbf{Step 3.}\enspace \textbf{Bound the approximation error.}
 
The error between the ideal and MLP-parameterised ITNet operators is:
\begin{align}
    &\norm{\mathcal{K}^*[u](x) - \Kop[u](x)}_2 \notag \\
    &= \norm{\int_\Omega [\kappa^*(x,y,u(x),u(y)) - \kap(x,y,u(x),u(y))]\, u(y)\, d\mu(y)}_2 \notag \\
    &\leq \int_\Omega \norm{\kappa^* - \kap}_\mathrm{op} \cdot \norm{u(y)}_2\, d\mu(y) \tag{triangle inequality for Bochner integral} \\
    &\leq \delta \cdot C_u \cdot \mu(\Omega), \label{eq:step3_bound}
\end{align}
where $C_u = \sup_{u \in U_c} \norm{u}_\infty < \infty$ (by compactness of $U_c$) and $\mu(\Omega) < \infty$ (by Assumption~\ref{ass:uat}).
 
Choose $\delta = \varepsilon / (2 C_u \mu(\Omega))$. Then:
\begin{equation}
    \sup_{u \in U_c} \norm{\mathcal{K}^*[u] - \Kop[u]}_\infty \leq \delta \cdot C_u \cdot \mu(\Omega) = \frac{\varepsilon}{2}.
    \label{eq:step3_final}
\end{equation}
 
\noindent\textbf{Step 4.}\enspace \textbf{Combine by triangle inequality.}
 
\begin{align}
    \sup_{u \in U_c} \norm{F(u) - \Kop[u]}_\infty
    &\leq \sup_{u \in U_c} \norm{F(u) - \mathcal{K}^*[u]}_\infty + \sup_{u \in U_c} \norm{\mathcal{K}^*[u] - \Kop[u]}_\infty \notag \\
    &< \frac{\varepsilon}{2} + \frac{\varepsilon}{2} = \varepsilon. \label{eq:final}
\end{align}
 
\begin{equation}
    \boxed{\sup_{u \in U_c} \norm{F(u) - \Kop[u]}_\infty < \varepsilon.}
\end{equation}
\end{proof}

\subsection{Corollary: Strict Expressiveness Ordering}
\label{sec:corollary}

\begin{corollary}[Strict expressiveness ordering]
\label{cor:ordering}
Let $\mathrm{Conv}$, $\mathrm{Attn}$, $\mathrm{RNN}$, $\mathrm{ITNet}$ denote the sets of operators representable by each architecture class. Then:
\begin{equation}
    \mathrm{Conv} \subsetneq \mathrm{ITNet}, \qquad \mathrm{Attn} \subsetneq \mathrm{ITNet}, \qquad \mathrm{RNN} \subsetneq \mathrm{ITNet},
    \label{eq:ordering}
\end{equation}
and these three subclasses are pairwise incomparable:
\begin{equation}
    \mathrm{Conv} \not\subseteq \mathrm{Attn}, \quad \mathrm{Attn} \not\subseteq \mathrm{Conv}, \quad \mathrm{Conv} \not\subseteq \mathrm{RNN}, \quad \text{etc.}
    \label{eq:incomparable}
\end{equation}
\end{corollary}
 
\begin{proof}
The inclusions $\mathrm{Conv}, \mathrm{Attn}, \mathrm{RNN} \subseteq \mathrm{ITNet}$ follow from Theorems~1--3. The strictness (proper subset) follows from the strictness parts of Theorems~1--3: each provides a witness operator in $\mathrm{ITNet}$ but not in the respective subclass.
 
Pairwise incomparability:
\begin{itemize}[noitemsep]
    \item $\mathrm{Conv} \not\subseteq \mathrm{Attn}$: convolution is translation-equivariant; attention (without PE) is permutation-equivariant. A translation-equivariant but non-permutation-equivariant operator (e.g., a spatially varying filter) is in $\mathrm{Conv}$ but not $\mathrm{Attn}$.
    \item $\mathrm{Attn} \not\subseteq \mathrm{Conv}$: attention is content-dependent; convolution is not. The softmax-weighted value aggregation from Theorem~2 is in $\mathrm{Attn}$ but not in $\mathrm{Conv}$ (by the linearity argument of Theorem~1(c)).
    \item $\mathrm{RNN} \not\subseteq \mathrm{Conv}$: causal operators with content-dependent gating (LSTM) are in $\mathrm{RNN}$ but not $\mathrm{Conv}$ (convolution is non-causal and content-independent).
    \item $\mathrm{Conv} \not\subseteq \mathrm{RNN}$: non-causal operators (bidirectional convolution) are in $\mathrm{Conv}$ but not $\mathrm{RNN}$ (all RNNs are causal).
\end{itemize}
\end{proof}
 
\subsection{Quantitative Approximation Rate}
\label{sec:rate}

The existential result (Theorem~\ref{thm:uat}) guarantees approximation but does not say \emph{how large} the MLP width $w_\kappa$ needs to be. We provide an explicit rate.
 
\begin{proposition}[Approximation rate]
\label{prop:rate}
Under Assumption~\ref{ass:uat}, if additionally the target operator $F$ has Lipschitz constant $L_F$ on $U_c$ (i.e., $\norm{F(u_1) - F(u_2)}_\infty \leq L_F \norm{u_1 - u_2}_\infty$), and the ideal kernel $\kappa^*$ has bounded \emph{Barron norm}~\cite{barron2002universal} $\norm{\kappa^*}_{\mathcal{B}} \leq B$, then the approximation error of ITNet with kernel MLP width $w_\kappa$ and $M$ quadrature points satisfies:
\begin{equation}
    \sup_{u \in U_c} \norm{F(u) - \Kop[u]}_\infty \leq \underbrace{\frac{C_1 B C_u \mu(\Omega)}{\sqrt{w_\kappa}}}_{\text{kernel approximation error}} + \underbrace{\frac{C_2 L_F C_u}{M^{1/s}}}_{\text{quadrature error}},
    \label{eq:rate}
\end{equation}
where $C_1$ depends on the input dimension $2s + 2d$ and $C_2$ depends on the smoothness of the integrand and the quadrature rule.
\end{proposition}
 
\begin{proof}[Proof sketch]
The first term follows from Barron's theorem~\cite{barron2002universal}: a single hidden layer MLP of width $w_\kappa$ approximates functions with bounded Barron norm at rate $O(1/\sqrt{w_\kappa})$, independent of input dimension (avoiding the curse of dimensionality for this function class). Applying this to each entry of $\kappa^*$ and combining with the error propagation bound \eqref{eq:step3_bound} gives the first term.
 
The second term is the quadrature error: replacing the integral $\int_\Omega$ with an $M$-point quadrature introduces error $O(M^{-r/s})$ for $r$-smooth integrands on $s$-dimensional compact domains, where $r \geq 1$ for continuous integrands.
\end{proof}

\section{Proof of Theorem 5: Kernel Recovery Under Translation Symmetry}
\label{app:proof_recovery}

We provide the complete proof of Theorem~\ref{thm:recovery} from the main paper, which was stated with only a proof sketch in \S\ref{sec:theory}.

\begin{theorem}[Kernel Recovery Under Translation Symmetry]
\label{thm:recovery_full}
Let the data distribution $\mathcal{D}$ over input-label pairs $(u, y)$ be invariant under translations: for every shift $\delta \in \R^s$, define the translated signal $\tau_\delta u(x) \coloneqq u(x - \delta)$. Assume $\mathcal{D}$ satisfies $(u, y) \sim \mathcal{D} \implies (\tau_\delta u, y) \sim \mathcal{D}$ for all $\delta$.

Let $\mathcal{L}(\theta) = \E_{(u,y) \sim \mathcal{D}}[\ell(\Kop[u], y)]$ be the population loss, where $\ell$ is differentiable.

Decompose the kernel into translation-invariant and orthogonal components:
\begin{equation}
\kap(x, y, u(x), u(y)) = \kap^{\mathrm{TI}}(x{-}y, u(x), u(y)) + \kap^{\perp}(x, y, u(x), u(y)),
\end{equation}
where $\kap^{\mathrm{TI}}$ depends on positions only through the displacement $x - y$, and $\kap^{\perp}$ is the orthogonal complement (i.e., $\int_\Omega \kap^{\mathrm{TI}} \cdot \kap^{\perp} \, d\mu \otimes d\mu = 0$).

Then under gradient flow with respect to $\theta$:
\begin{equation}
\left\|\frac{\partial \mathcal{L}}{\partial \kap^{\perp}}\right\|_F = 0 \quad \text{at every iterate.}
\end{equation}
\end{theorem}

\begin{proof}
The proof follows the equivariant gradient framework of \cite{elesedy2021provably}. We proceed in three steps.

\medskip
\noindent\textbf{Step 1: Translation invariance of the loss.}

For any shift $\delta \in \R^s$, the ITNet operator applied to the translated signal $\tau_\delta u$ evaluates the kernel at shifted positions:
\begin{align}
(\Kop[\tau_\delta u])(x) &= \int_\Omega \kap(x, y, u(x - \delta), u(y - \delta)) \cdot u(y - \delta) \, d\mu(y) + W_\theta u(x - \delta).
\end{align}
Substituting $z = y - \delta$ (valid when $\mu$ is translation-invariant):
\begin{align}
&= \int_\Omega \kap(x, z + \delta, u(x - \delta), u(z)) \cdot u(z) \, d\mu(z) + W_\theta u(x - \delta).
\end{align}
Under the change of output variable $x' = x - \delta$:
\begin{align}
(\Kop[\tau_\delta u])(x' + \delta) &= \int_\Omega \kap(x' + \delta, z + \delta, u(x'), u(z)) \cdot u(z) \, d\mu(z) + W_\theta u(x').
\end{align}

By the translation invariance of $\mathcal{D}$, the population loss satisfies $\mathcal{L}(\theta) = \E[\ell(\Kop[\tau_\delta u], y)]$ for all $\delta$.

\medskip
\noindent\textbf{Step 2: Averaging the gradient over translations.}

Consider the gradient of $\mathcal{L}$ with respect to the kernel parameters. Since $\mathcal{L}$ is invariant under all translations $\delta$, averaging the gradient over the translation group $G = (\R^s, +)$ with Haar measure \cite{haar1933mass,bela1970harmonic} $d\delta$ leaves $\mathcal{L}$ unchanged:
\begin{equation}
\nabla_\theta \mathcal{L} = \frac{1}{|\Omega|} \int_\Omega \nabla_\theta \mathcal{L}\big|_{\tau_\delta} \, d\delta.
\end{equation}

For the kernel component, the gradient with respect to $\kap$ at a specific pair $(x, y)$ is:
\begin{equation}
\frac{\partial \mathcal{L}}{\partial \kap(x, y, \cdot, \cdot)} = \E_{(u,y)} \left[ \frac{\partial \ell}{\partial (\Kop[u])(x)} \cdot u(y)^\top \right].
\end{equation}

Averaging this over all translations $\delta$ (shifting both $x$ and $y$ by $\delta$) yields:
\begin{align}
\frac{1}{|\Omega|} \int_\Omega \frac{\partial \mathcal{L}}{\partial \kap(x + \delta, y + \delta, \cdot, \cdot)} \, d\delta.
\end{align}
This average depends only on $x - y$ (since all pairs $(x + \delta, y + \delta)$ share the same displacement), so it lies entirely in the translation-invariant subspace.

\medskip
\noindent\textbf{Step 3: Orthogonality annihilation.}

By the Schur orthogonality~\cite{serre1977linear} relations for the translation group acting on the space of kernel functions, any component of the gradient that lies in the orthogonal complement of the translation-invariant subspace integrates to zero under the group average. Since the averaged gradient equals the original gradient (by Step 1), the projection of $\nabla_\theta \mathcal{L}$ onto the $\kap^{\perp}$ subspace vanishes:
\begin{equation}
\mathrm{Proj}_{\kap^{\perp}} \nabla_\theta \mathcal{L} = 0.
\end{equation}

Under gradient flow $\dot{\theta} = -\nabla_\theta \mathcal{L}$, the parameters never receive a gradient signal that would move the kernel out of the translation-invariant subspace. If the kernel is initialized in (or near) the TI subspace, it remains there throughout training, recovering the convolutional special case of Theorem~\ref{thm:conv}.
\end{proof}


\section{Implementation Details}
\label{app:impl_details}

\subsection{Computational Complexity}
\label{app:complexity}

Table~\ref{tab:complexity} compares the time and space complexity of ITNet with CNN, Transformer, and Mamba baselines. At $M = n$, ITNet-MC matches Transformer complexity with an additional $d$ factor from the matrix-valued kernel; for $d \leq \sqrt{n}$, this is comparable. The low-rank scheme recovers linear complexity at the cost of rank-$r$ approximation error.

\begin{table}[h]
\centering
\caption{Computational complexity.}
\label{tab:complexity}
\small
\setlength{\tabcolsep}{3pt}
\renewcommand{\arraystretch}{0.7}
\begin{tabular}{lcc}
\toprule
\textbf{Method} & \textbf{Time} & \textbf{Space} \\
\midrule
CNN (local, radius $k$)   & $O(nkd^2)$   & $O(kd^2)$ \\
Transformer (full)        & $O(n^2 d)$   & $O(n^2)$ \\
Mamba (SSM)               & $O(nd)$      & $O(nd)$ \\
\midrule
ITNet (exact)             & $O(n^2 d^2)$ & $O(n^2 d^2)$ \\
ITNet (MC, $M$ samples)   & $O(nMd^2)$   & $O(nMd)$ \\
ITNet (low-rank, rank $r$) & $O(ndr)$    & $O(ndr)$ \\
\bottomrule
\end{tabular}

{\scriptsize \emph{Note:}  $n$: positions, $d$: features, $M$: samples, $r$: rank, $k$: radius.}
\end{table}

\subsection{Full Architectural Specification}
\label{app:arch_spec}

We provide the complete architectural specification for all ITNet model sizes. The architecture follows the pre-norm Transformer layout~\citep{xiong2020layer} with the ITNet operator replacing self-attention. Every component is specified to enable exact reproduction.

Each ITNet layer $\ell = 1, \ldots, L$ consists of two sub-blocks with residual connections:
\begin{equation}
\begin{aligned}
z^{(\ell)} &= \Kop^{(\ell)}[\mathrm{LN}(u^{(\ell-1)})] + u^{(\ell-1)}, \\
u^{(\ell)} &= \mathcal{F}^{(\ell)}(\mathrm{LN}(z^{(\ell)})) + z^{(\ell)},
\end{aligned}
\label{eq:app_layer}
\end{equation}
where:
\begin{itemize}[noitemsep]
\item $\mathrm{LN}$ is layer normalization with learnable affine parameters $\gamma, \beta \in \R^d$, applied \emph{before} each sub-block (pre-norm).
We use pre-norm rather than post-norm because it provides more stable gradients at initialization for deep models~\citep{xiong2020layer}, avoids the need for learning rate warmup in shallow settings, and is standard in modern architectures (GPT-2~\citep{radford2019language}, LLaMA~\citep{touvron2023llama}, ViT~\citep{dosovitskiy2021image}).

\item $\Kop^{(\ell)}$ is the multi-head ITNet operator (Eq.~\eqref{eq:multihead}) with $H$ heads, each operating on $d_h = d/H$ dimensional features.

\item $\mathcal{F}^{(\ell)}$ is a position-wise feed-forward network:
\begin{equation}
\mathcal{F}^{(\ell)}(v) = W_2^{(\ell)} \cdot \mathrm{GELU}(W_1^{(\ell)} v + b_1^{(\ell)}) + b_2^{(\ell)},
\label{eq:app_ffn}
\end{equation}
with $W_1^{(\ell)} \in \R^{4d \times d}$ (expansion factor 4), $W_2^{(\ell)} \in \R^{d \times 4d}$, and biases $b_1 \in \R^{4d}$, $b_2 \in \R^d$.
We use GELU activation~\citep{hendrycks2016gaussian} throughout, consistent with BERT~\citep{devlin2019bert} and GPT-2~\citep{radford2019language}.
\end{itemize}

Table~\ref{tab:app_configs} specifies all architectural hyperparameters for each model size.

\begin{table}[h]
\centering
\caption{Complete ITNet model configurations. $L$: layers; $d$: hidden dimension; $H$: heads; $d_h = d/H$: per-head dimension; $w_\kappa$: kernel MLP width; $\ell_\kappa$: kernel MLP depth; FFN: feed-forward expansion factor.}
\label{tab:app_configs}
\small
\begin{tabular}{lcccccccc}
\toprule
\textbf{Config} & $L$ & $d$ & $H$ & $d_h$ & $w_\kappa$ & $\ell_\kappa$ & \textbf{FFN} & \textbf{Total params} \\
\midrule
ITNet-PC & 6 & 128 & 4 & 32 & 64 & 2 & $4\times$ & 3.1M \\
ITNet-S & 12 & 384 & 6 & 64 & 128 & 2 & $4\times$ & 22M \\
ITNet-B & 12 & 768 & 12 & 64 & 128 & 2 & $4\times$ & 86M \\
ITNet-L & 24 & 1024 & 16 & 64 & 128 & 2 & $4\times$ & 307M \\
\bottomrule
\end{tabular}
\end{table}

All configurations use $d_h = 64$ per head (except ITNet-PC which uses $d_h = 32$). This means the per-head kernel outputs a $64 \times 64$ matrix (4,096 values), and the per-pair cost is $O(d_h^2) = O(4096)$ multiply-adds for the matrix-vector product. Summing over $H$ heads gives $O(H \cdot d_h^2) = O(d^2/H)$ per pair.
Table~\ref{tab:app_param_breakdown} provides a detailed parameter count for ITNet-B.

\begin{table}[h]
\centering
\caption{Parameter breakdown for ITNet-B (12 layers, $d = 768$, $H = 12$).}
\label{tab:app_param_breakdown}
\small
\begin{tabular}{lcc}
\toprule
\textbf{Component} & \textbf{Per layer} & \textbf{Total ($\times 12$)} \\
\midrule
Kernel MLP $W_1 \in \R^{128 \times (6 \cdot 64 + 1 + 3 \cdot 64)}$ & $128 \times 577 = 73.9$K & 0.89M \\
Kernel MLP $b_1 \in \R^{128}$ & 128 & 1.5K \\
Kernel MLP $W_2 \in \R^{64^2 \times 128}$ ($\times H$) & $4096 \times 128 \times 12 = 6.29$M & 75.5M \\
Kernel MLP $b_2 \in \R^{64^2}$ ($\times H$) & $4096 \times 12 = 49.2$K & 0.59M \\
Output projection $W^O \in \R^{768 \times 768}$ & 590K & 7.08M \\
Residual $W_\theta \in \R^{768 \times 768}$ & 590K & 7.08M \\
LayerNorm (2$\times$) $\gamma, \beta \in \R^{768}$ & $2 \times 1.5$K & 36.9K \\
FFN $W_1 \in \R^{3072 \times 768}$ & 2.36M & 28.3M \\
FFN $W_2 \in \R^{768 \times 3072}$ & 2.36M & 28.3M \\
FFN biases & 3.84K & 46.1K \\
\midrule
\textbf{Total per layer} & \textbf{$\sim$7.2M} & \\
\textbf{Total (12 layers + embeddings)} & & \textbf{$\sim$86M} \\
\bottomrule
\end{tabular}
\end{table}

The dominant parameter cost is the kernel MLP output layer $W_2$ ($\sim$75.5M across 12 layers), because each head's kernel must produce a $d_h^2 = 4096$-dimensional output.
This is comparable to the $W_Q W_K W_V$ projections in a standard Transformer ($3 \times d^2 = 3 \times 768^2 = 1.77$M per layer, $\times 12 = 21.2$M total), with the additional cost arising from the kernel MLP's ability to compute nonlinear, content-dependent transformations rather than fixed linear projections.

\textbf{Kernel MLP Architecture:}
The kernel MLP for each head $h$ has the following structure:
\begin{align}
h^{(0)} &= z_{xy} \in \R^{6L_f + 1 + 3d_h}, \label{eq:app_mlp_input} \\
h^{(1)} &= \mathrm{GELU}(W_1 h^{(0)} + b_1) \in \R^{w_\kappa}, \label{eq:app_mlp_hidden} \\
\mathrm{vec}(K_{xy}^{(h)}) &= W_2 h^{(1)} + b_2 \in \R^{d_h^2}, \label{eq:app_mlp_output}
\end{align}
where:
\begin{itemize}[noitemsep]
\item $W_1 \in \R^{w_\kappa \times (6L_f + 1 + 3d_h)}$: input-to-hidden weights.
For ITNet-B with $L = 64$ Fourier frequencies and $d_h = 64$: input dimension is $6 \times 64 + 1 + 3 \times 64 = 577$, so $W_1 \in \R^{128 \times 577}$.
\item $b_1 \in \R^{w_\kappa}$: hidden bias.
\item $W_2 \in \R^{d_h^2 \times w_\kappa}$: hidden-to-output weights.
For $d_h = 64$: $W_2 \in \R^{4096 \times 128}$.
\item $b_2 \in \R^{d_h^2}$: output bias.
\item The output $\mathrm{vec}(K_{xy}^{(h)}) \in \R^{d_h^2}$ is reshaped to $K_{xy}^{(h)} \in \R^{d_h \times d_h}$.
\end{itemize}

We use a 2-layer MLP (one hidden layer) rather than deeper architectures because the kernel MLP width ablation (Table~\ref{tab:ablation_kappa}) shows that 2 layers with $w_\kappa = 128$ achieves 81.4\% on ImageNet-1K, matching 3 layers (81.5\%) at significantly lower cost. The shallow MLP also has better hardware utilization: deeper MLPs require sequential execution of layers, reducing parallelism within each tile.

\begin{figure}
    \centering
    \includegraphics[width=1\linewidth]{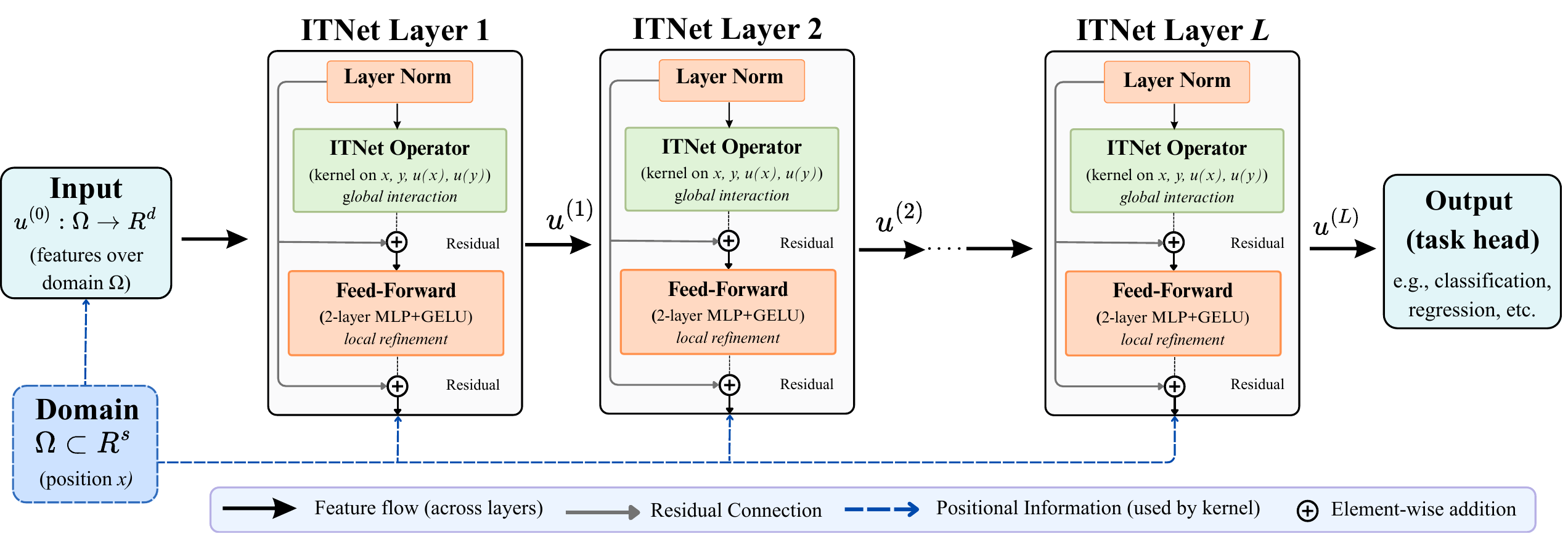}
    \caption{Overview of the ITNet architecture. The model consists of $L$ stacked pre-norm residual layers. Each layer applies (i) layer normalization, (ii) the ITNet operator - a learnable integral operator with a kernel depending on positions $x,y$ and features $u(x), u(y)$  and (iii) a feed-forward network (FFN) for local refinement, with residual connections after each block. Positional information is provided through the domain $\Omega$ and used directly by the kernel. This design mirrors Transformer-style architectures while generalizing convolution (local kernels), attention (content-dependent kernels), and recurrence (causal kernels) within a unified framework.}
    \label{fig:itnet_layer}
\end{figure}

\subsubsection{Positional Encoding: Random Fourier Features}

The Fourier feature map $\gamma : \R^s \to \R^{2L_f}$ is defined as:
\begin{equation}
\gamma(x) = [\sin(2\pi \mathbf{B} x);\; \cos(2\pi \mathbf{B} x)] \in \R^{2L_f},
\label{eq:fourier_features}
\end{equation}
where $\mathbf{B} \in \R^{L \times s}$ is a fixed random matrix sampled once at initialization from $\mathcal{N}(0, \sigma^2 I)$ with $\sigma = 10$ and $L = 64$.

Raw position coordinates $x \in \R^s$ are low-dimensional ($s = 1$ for text, $s = 2$ for images, $s = 3$ for point clouds).
An MLP operating on such low-dimensional inputs suffers from \emph{spectral bias}~\citep{xiong2020layer}: it preferentially learns low-frequency functions and struggles to represent high-frequency spatial patterns.
The Fourier feature map lifts positions to a $2L_f$-dimensional space where high-frequency components are explicitly represented as linear features, enabling the MLP to learn sharp spatial functions~\citep{tancik2020fourier}.

The bandwidth $\sigma$ controls the frequency range: $\sigma = 10$ corresponds to spatial frequencies up to $\sim$10 cycles per unit length, which captures both coarse object-level structure and fine-grained local texture on normalized domains.
Table~\ref{tab:ablation_fourier} shows that $\sigma = 1$ under-resolves spatial structure ($-1.3\%$), $\sigma = 100$ introduces excessively high-frequency features that are difficult for the kernel MLP to use ($-0.6\%$), and $\sigma = 10$ is optimal.
The number of frequencies $L = 64$ provides a 128-dimensional positional representation per coordinate group; performance saturates beyond $L = 64$ (Table~\ref{tab:ablation_fourier}).

We sample $\mathbf{B}$ once from $\mathcal{N}(0, \sigma^2 I)$ and freeze it throughout training.
Learning $\mathbf{B}$ would make the Fourier features input-dependent, coupling position encoding with feature learning and complicating the theoretical analysis (Theorem~\ref{thm:recovery} assumes fixed positional encoding).
The kernel MLP is already a universal approximator over the lifted positional space, so learnable $\mathbf{B}$ provides no additional expressiveness in theory; empirically, we found no improvement from learning $\mathbf{B}$ ($\pm 0.0\%$ on ImageNet-1K).

\subsection{Initialization Scheme}
\label{app:init}

Proper initialization is critical for training stability, especially because the ITNet operator is more complex than self-attention (matrix-valued kernel vs.\ scalar attention weights).
Our initialization ensures two properties: (i)~the initial operator is approximately the identity, so that an $L$-layer ITNet behaves like a shallow network at the start of training; and (ii)~the initial kernel output has small norm, so that the integral term does not dominate the residual.

\subsubsection{Kernel MLP Initialization}

\begin{itemize}[noitemsep, leftmargin=1.8em]
\item \textbf{Input-to-hidden weights $W_1$:}
Sampled from $\mathcal{N}(0, 0.02^2)$.
The small variance ensures the hidden activations $h^{(1)} = \mathrm{GELU}(W_1 z + b_1)$ are in the approximately linear regime of GELU at initialization (GELU$(x) \approx 0.5x$ for $|x| \ll 1$), providing well-conditioned gradients.
\item \textbf{Hidden bias $b_1$:} Initialised to zero.
\item \textbf{Hidden-to-output weights $W_2$:}
Initialised as $W_2 = \epsilon \cdot \tilde{W}_2$ where $\tilde{W}_2 \sim \mathcal{N}(0, 1/w_\kappa)$ and $\epsilon = 10^{-3}$.
This ensures the initial kernel output $\mathrm{vec}(K_{xy}) = W_2 h^{(1)} + b_2$ has small norm: $\|\mathrm{vec}(K_{xy})\|_2 \approx \epsilon \sqrt{d_h^2} = 10^{-3} \times 64 = 0.064$, so $\|K_{xy}\|_F \approx 0.064$.
\item \textbf{Output bias $b_2$:}
Initialised so that the initial kernel is approximately $\frac{1}{n} \mathbf{I}_{d_h}$. Specifically, $[b_2]_k = 1/n$ if $k$ corresponds to a diagonal entry of the $d_h \times d_h$ reshaped output, and $0$ otherwise.
This makes the initial operator output approximately $\frac{1}{n} \sum_{j} \Id \cdot u(x_j) + W_\theta u(x_i) = \bar{u} + u(x_i)$: the global mean plus the identity  -  a simple pooling-plus-skip operation.
\end{itemize}

This follows the \emph{zero-init} principle used in GPT-2~\citep{radford2019language} and formalised as LayerScale: initialising each layer's output near zero ensures the residual stream $u^{(\ell)} \approx u^{(\ell-1)}$ at the start of training, so that an $L$-layer network initially behaves as a 1-layer network.
This prevents the signal-to-noise ratio from degrading across layers and is essential for training models with $L \geq 12$ without careful learning rate warmup.

We compared three $\epsilon$ values on ImageNet-1K (ITNet-S, 300 epochs):
\begin{center}
\small
\begin{tabular}{lccc}
\toprule
$\epsilon$ & Top-1 (\%) & Training stable? & Notes \\
\midrule
$10^{-1}$ & diverges & No & Kernel output too large; gradient explosion at epoch 5 \\
$10^{-2}$ & 80.9 & Yes & Slightly noisy early training \\
$10^{-3}$ & 81.4 & Yes & Smooth convergence; used throughout \\
$10^{-4}$ & 81.3 & Yes & Slightly slower convergence (kernel learns slowly) \\
\bottomrule
\end{tabular}
\end{center}

\textbf{Residual Matrix Initialization:}
$W_\theta = I_d$: the initial residual is the identity, so the full operator output is approximately $u(x_i) + \text{small integral term} \approx u(x_i)$.
This is the standard skip-connection initialization used in ResNets~\citep{he2016deep} and Transformers.

\textbf{FFN Initialization:}
The feed-forward network weights $W_1^{(\ell)}, W_2^{(\ell)}$ are initialized with the standard Xavier uniform scheme~\citep{glorot2010understanding}:
$W \sim \mathrm{Uniform}(-\sqrt{6/(d_\mathrm{in} + d_\mathrm{out})},\; \sqrt{6/(d_\mathrm{in} + d_\mathrm{out})})$.
Biases are initialized to zero.
The output layer $W_2^{(\ell)}$ is additionally scaled by $1/\sqrt{2L}$ following~\cite{radford2019language}, which normalizes the variance contribution of each layer in the residual stream.

\textbf{Layer Normalization Initialization:}
The affine parameters are initialized as $\gamma = \mathbf{1}_d$ (scale) and $\beta = \mathbf{0}_d$ (shift), so that $\mathrm{LN}$ initially performs only zero-mean unit-variance normalization without rescaling.

\textbf{Output Projection Initialization:}
The multi-head output projection $W^O \in \R^{d \times d}$ is initialized with Xavier uniform, scaled by $1/\sqrt{2L}$ (same as FFN output), ensuring the concatenated head outputs do not amplify the signal.

\subsubsection{Embedding Initialization}

\textbf{Image encoder:}
The patch embedding $W_\mathrm{img} \in \R^{d \times (P^2 \cdot C)}$ (where $P = 16$ is patch size and $C = 3$ is the number of channels) is initialized with truncated normal $\mathcal{N}(0, 0.02^2)$, following ViT~\citep{dosovitskiy2021image}.
The \texttt{[CLS]} token embedding is initialized from $\mathcal{N}(0, 0.02^2)$.

\textbf{Text encoder:}
The token embedding $W_\mathrm{txt} \in \R^{d \times V}$ (where $V = 30{,}522$ is the WordPiece vocabulary size) is initialized from $\mathcal{N}(0, 0.02^2)$, following BERT~\citep{devlin2019bert}.
The \texttt{[CLS]} and \texttt{[SEP]} token embeddings are initialized identically.

\textbf{Point cloud encoder:}
The coordinate embedding $W_\mathrm{pc} \in \R^{d \times 3}$ is initialized with Xavier uniform.
If the local pre-extraction module is used ($K = 16$ neighbours), its MLP weights follow the same $\mathcal{N}(0, 0.02^2)$ scheme.

\textbf{Modality embeddings:}
For multimodal tasks, each modality receives a learnable embedding $e_\mathrm{img}, e_\mathrm{txt} \in \R^d$ added to the features after the modality-specific encoder.
These are initialized from $\mathcal{N}(0, 0.02^2)$.

\subsection{Regularization and Training Stability}
\label{app:regularization}

We apply dropout~\citep{srivastava2014dropout} at three locations:
(i)~after the ITNet operator output (before the residual addition), with rate $p = 0.1$ for ITNet-S/B and $p = 0.0$ for ITNet-PC;
(ii)~after the FFN hidden layer (inside the FFN), with the same rate;
(iii)~after the embedding layer, with rate $p = 0.1$.
No dropout is applied inside the kernel MLP, as the kernel evaluation is already regularised by the $\epsilon$-scaling of $W_2$.
Following~\citet{huang2016deep}, we apply stochastic depth with linearly increasing drop probability from $0$ at the first layer to $p_\mathrm{drop}$ at the last layer:
$p^{(\ell)} = \ell \cdot p_\mathrm{drop} / L$.
We use $p_\mathrm{drop} = 0.1$ for ITNet-S, $0.2$ for ITNet-B, and $0.3$ for ITNet-L.
This is standard in ViT training~\citep{touvron2021deit} and essential for training models with $L \geq 24$.
We clip the global gradient norm to $1.0$ in all experiments, using \texttt{torch.nn.utils.clip\_grad\_norm\_}.
This prevents gradient explosion when the kernel MLP produces large outputs early in training (before $W_2$ converges to small values).
All experiments use \texttt{bfloat16} for forward and backward computation with \texttt{float32} master weights and optimizer states~\citep{micikevicius2018mixed}.
The kernel MLP evaluation, matrix-vector product $K_{pq} \cdot u(x_j)$, and gradient accumulation are all performed in \texttt{bfloat16}.
Layer normalization and the softmax (when used for baseline comparisons) use \texttt{float32} for numerical stability.
For ImageNet-1K experiments, we maintain an Exponential moving average (EMA) of the model weights with decay $\beta_\mathrm{EMA} = 0.9999$, following~\cite{polyak1992acceleration}. The EMA model is used for evaluation.

We use AdamW~\citep{loshchilov2019decoupled} with decoupled weight decay applied to all weight matrices but \emph{not} to biases, layer normalization parameters, the Fourier feature matrix $\mathbf{B}$ (which is frozen), or the kernel MLP output bias $b_2$ (which encodes the $1/n \cdot \Id$ initialization and should not be pulled toward zero).

\subsection{Optimizer Configuration}
\label{app:optimizer}

All experiments use AdamW~\citep{loshchilov2019decoupled} with:
\begin{itemize}[noitemsep]
\item $\beta_1 = 0.9$, $\beta_2 = 0.999$, $\epsilon_\mathrm{Adam} = 10^{-8}$
\item Cosine learning rate schedule with linear warmup
\item Learning rate and weight decay per task specified in Section~\ref{app:training}
\end{itemize}

We chose AdamW over SGD with momentum because: (i)~the kernel MLP parameters have different gradient magnitudes from the FFN and embedding parameters (the kernel MLP sees $n^2$ gradient contributions per step, while the FFN sees $n$), and Adam's per-parameter adaptive learning rate handles this scale difference automatically; (ii)~decoupled weight decay~\citep{loshchilov2019decoupled} provides consistent regularization regardless of the adaptive learning rate, which is important for the kernel MLP whose gradients can be large early in training.

\subsection{Statistical Reporting}
\label{app:stats}

For all ITNet models (ITNet-S, ITNet-B, ITNet-L), we report results as mean $\pm$ standard deviation over $N=3$ independent runs with different random initializations. (see Tables~\ref{tab:imagenet}, \ref{tab:glue}, \ref{tab:modelnet}, and \ref{tab:multimodal} in the main paper).
Each run differs in random weight initialisation and stochastic training effects (e.g., data shuffling and mini-batch sampling), while all other training settings are kept fixed.

Let $\{m_i\}_{i=1}^N$ denote the evaluation metric (e.g., accuracy or F1) from $N=3$ runs. The reported mean and standard deviation are computed as:
\begin{equation}
\mu = \frac{1}{N} \sum_{i=1}^N m_i, 
\qquad
\sigma = \sqrt{\frac{1}{N-1} \sum_{i=1}^N (m_i - \mu)^2}.
\end{equation}
We report $\mu \pm \sigma$ in all tables.
The standard deviation reflects variability due to random initialization and stochastic optimization. For the considered benchmarks, the observed variance is small relative to performance differences between models, indicating stable training behavior.
We report standard deviation (not standard error) and do not assume a specific distribution beyond empirical estimation from repeated runs.





\subsection{Triton Kernel Profiling}
\label{app:profiling}

We profile the Triton kernel on an H200-140GB at $n = 512$, $d = 768$, $w_\kappa = 128$, $\ell_\kappa = 2$.

\begin{table}[h]
\centering
\caption{Triton kernel profiling.}
\label{tab:profiling}
\small
\begin{tabular}{lccc}
\toprule
\textbf{Mode} & \textbf{Achieved TFLOPS} & \textbf{Peak TFLOPS} & \textbf{Utilization} \\
\midrule
ITNet-Exact & 604 & 989 & 61\% \\
ITNet-LR ($r = 64$) & 831 & 989 & 84\% \\
FlashAttention-2 (reference) & 870 & 989 & 88\% \\
\bottomrule
\end{tabular}
\end{table}

ITNet-Exact achieves 61\% of peak \texttt{bfloat16} TFLOPS (604/989).
The gap relative to FlashAttention-2 (88\%) is attributable to the kernel MLP evaluation, which involves small matrix multiplications ($128 \times 577$ and $4096 \times 128$) that underutilise the tensor core pipeline (tensor cores are optimised for large matrix multiplications with dimensions that are multiples of 16).
ITNet-LR achieves 84\% utilization because the factorized computation is dominated by the matrix multiplications $\Phi^\top Z$ and $Z = \sum_j \omega_j \Psi_j u_j$, which have larger dimensions and better tensor core utilization.

The auto-tuning procedure (Phase 1 of Algorithm~\ref{alg:tiled_full}) selects $B_q^* = 64$, $B_k^* = 64$, giving $(64 + 64) \times 768 \times 2 = 192$KB per tile pair, which fits within the per-SM shared memory budget ($\sim$228KB), enabling efficient on-chip execution. Larger tiles ($B_q = B_k = 128$) exceed this limit, reducing occupancy and degrading performance.


\subsection{IO Complexity of Tiled ITNet}
\label{app:io_complexity}

\begin{proposition}[IO Complexity  -  Full Statement]
\label{prop:io_full}
Let $\ell_\kappa$ denote the number of MLP layers in $\kappa_\theta$ and $w_\kappa$ the hidden width. Let $B_q$ and $B_k$ be the query and key tile sizes satisfying the SRAM constraint $(B_q + B_k) \cdot d \leq S_\mathrm{SRAM}$. The tiled forward pass (Algorithm~\ref{alg:tiled_full}) requires:
\begin{equation}
\Theta\!\left(\frac{n^2 d}{B_k}\right) \;\text{HBM reads}
\quad\text{and}\quad
\Theta\!\left(\frac{n^2 w_\kappa \ell_\kappa}{S_\mathrm{SRAM}}\right) \;\text{FLOPs}.
\end{equation}
\end{proposition}

\begin{proof}
Each query tile $U_i \in \R^{B_q \times d}$ is loaded once per outer iteration: $n/B_q$ outer iterations $\times B_q d$ elements $= nd$ reads for all queries. Each key tile $U_j \in \R^{B_k \times d}$ is loaded once per inner iteration for each outer tile: $(n/B_q)(n/B_k) \times B_k d = n^2 d / B_q$ key reads in total. The total HBM reads are $nd + n^2 d / B_q = \Theta(n^2 d / B_q)$ for large $n$.

For FLOPs: each pair $(p, q)$ in a tile requires one MLP forward pass costing $O(w_\kappa \ell_\kappa)$ multiply-adds, plus one $d \times d$ matrix-vector product costing $O(d^2)$. There are $n^2$ total pairs, giving $O(n^2 (w_\kappa \ell_\kappa + d^2))$ total FLOPs. Since $w_\kappa \ell_\kappa$ is typically comparable to $d^2$ (e.g., $w_\kappa = 128$, $\ell_\kappa = 2$ gives $w_\kappa \ell_\kappa = 256$ vs. $d^2 = 768^2$ for the matrix-vector product), the matrix-vector product dominates, and total FLOPs are $O(n^2 d^2)$.
\end{proof}


\subsection{Tiled Forward Pass}
\label{app:tiled_forward}

Algorithm~\ref{alg:tiled_full} gives the complete tiled forward pass. Phase~1 profiles all valid $(B_q, B_k)$ combinations against the SRAM budget and selects the fastest; Phase~2 streams query and key tiles through SRAM, evaluating the kernel MLP and accumulating the weighted integral without ever materializing the full $n \times n$ kernel matrix. The quadrature weight $\omega_j = \mu(\{x_j\})$ (typically $1/n$) appears on line~14. The residual $W_\theta U_i$ is fused into the write-back to avoid an extra HBM round-trip.

\begin{algorithm}[h]
\caption{ITNet Tiled Forward Pass with Auto-Tuning}
\label{alg:tiled_full}
\small
\renewcommand{\arraystretch}{0.72}
\begin{algorithmic}[1]
\Require Signal $U \in \R^{n \times d}$, positions $X \in \R^{n \times s}$, kernel MLP $\kappa_\theta$, local weight $W_\theta \in \R^{d \times d}$, SRAM capacity $S_\mathrm{SRAM}$
\Ensure Output $O \in \R^{n \times d}$
\Statex \textbf{Phase 1: Auto-tune block sizes}
\For{$B_q \in \{16, 32, 64, 128\}$}
    \For{$B_k \in \{16, 32, 64, 128\}$}
        \If{$(B_q + B_k) \cdot d \cdot \mathrm{bytes\_per\_element} \leq S_\mathrm{SRAM}$}
            \State Profile tiled kernel with $(B_q, B_k)$ on a warmup batch
        \EndIf
    \EndFor
\EndFor
\State Select $(B_q^*, B_k^*)$ with minimum measured runtime
\Statex \textbf{Phase 2: Tiled forward pass}
\State $O \leftarrow \mathbf{0}_{n \times d}$
\For{$i = 0$ \textbf{to} $n - B_q^*$ \textbf{step} $B_q^*$} \Comment{Outer loop: query tiles}
    \State \textbf{Load} $U_i \leftarrow U[i{:}i{+}B_q^*]$, $X_i \leftarrow X[i{:}i{+}B_q^*]$ to SRAM
    \State $\mathrm{acc}_i \leftarrow \mathbf{0}_{B_q^* \times d}$
    \For{$j = 0$ \textbf{to} $n - B_k^*$ \textbf{step} $B_k^*$} \Comment{Inner loop: key tiles}
        \State \textbf{Load} $U_j \leftarrow U[j{:}j{+}B_k^*]$, $X_j \leftarrow X[j{:}j{+}B_k^*]$ to SRAM
        \For{$p = 1, \ldots, B_q^*$}
            \For{$q = 1, \ldots, B_k^*$}
                \State $K_{pq} \leftarrow \kappa_\theta(X_i[p],\, X_j[q],\, U_i[p],\, U_j[q])$ \Comment{MLP eval in registers}
                \State $\mathrm{acc}_i[p] \leftarrow \mathrm{acc}_i[p] + K_{pq} \cdot U_j[q] \cdot \omega_j$ \Comment{Weighted quadrature}
            \EndFor
        \EndFor
    \EndFor
    \State $O[i{:}i{+}B_q^*] \leftarrow \mathrm{acc}_i + W_\theta\, U_i$
    \State \textbf{Write} $O[i{:}i{+}B_q^*]$ to HBM
\EndFor
\end{algorithmic}
\end{algorithm}

\subsection{Tiled Backward Pass with Gradient Checkpointing}
\label{app:backward}

The forward pass (Algorithm~\ref{alg:tiled_full}) never materialises the full $n \times n \times d \times d$ kernel matrix: only a $B_q^* \times B_k^*$ tile of kernel evaluations resides in SRAM at any time.
A na\"ive backward pass would require access to all $n^2$ kernel values $K_{ij} = \kappa_\theta(x_i, y_j, u(x_i), u(y_j))$ to compute the gradients, which would either require storing the full kernel matrix during the forward pass ($O(n^2 d^2)$ memory) or recomputing the entire kernel from scratch.
Following the gradient checkpointing strategy of~\cite{chen2016training} - also employed by FlashAttention~\cite{dao2022flashattention} for the attention matrix - we adopt a \emph{tiled recomputation} approach: during the backward pass, we iterate over the same tile structure as the forward pass, recompute the kernel MLP $\kappa_\theta$ for each tile on-the-fly, and immediately use the recomputed values to accumulate gradients before discarding them.

Let $\mathcal{L}$ denote the scalar loss.
The forward pass computes, for each query position $x_i$:
\begin{equation}
O_i = (\Kop[u])(x_i) = \sum_{j=1}^{n} \omega_j \cdot \kappa_\theta(x_i, y_j, u_i, u_j) \cdot u_j + W_\theta \, u_i,
\label{eq:forward_discrete}
\end{equation}
where $\omega_j$ is the quadrature weight for position $y_j$ and we write $u_i = u(x_i)$, $u_j = u(y_j)$ for brevity.
The backward pass must compute three quantities:
\begin{enumerate}[label=(\roman*), noitemsep]
\item $\partial \mathcal{L} / \partial u_i$ for all $i = 1, \ldots, n$ (gradient with respect to input features, for backpropagation to earlier layers),
\item $\partial \mathcal{L} / \partial \theta$ (gradient with respect to kernel MLP parameters, for weight updates), and
\item $\partial \mathcal{L} / \partial W_\theta$ (gradient with respect to the residual matrix).
\end{enumerate}

\noindent Applying the chain rule to Eq.~\eqref{eq:forward_discrete}, we obtain:

\textbf{(i) Gradient with respect to kernel evaluations.}
Define $K_{ij} = \kappa_\theta(x_i, y_j, u_i, u_j) \in \R^{d \times d}$.
From Eq.~\eqref{eq:forward_discrete}:
\begin{equation}
\frac{\partial \mathcal{L}}{\partial K_{ij}}
= \omega_j \cdot \frac{\partial \mathcal{L}}{\partial O_i} \cdot u_j^\top
\;\in\; \R^{d \times d},
\label{eq:grad_K}
\end{equation}
where $\partial \mathcal{L} / \partial O_i \in \R^d$ is the upstream gradient at query position $i$ (received from the subsequent layer).

\textbf{(ii) Gradient with respect to input features.}
Each input feature $u_i$ participates in Eq.~\eqref{eq:forward_discrete} in three ways: (a)~as a query feature (through the kernel's dependence on $u(x_i)$), (b)~as a key feature (through the kernel's dependence on $u(y_j)$ when $j = i$), and (c)~through the residual term $W_\theta u_i$.
Combining all contributions:
\begin{equation}
\frac{\partial \mathcal{L}}{\partial u_i}
= \underbrace{\sum_{j=1}^{n} \omega_j \cdot \frac{\partial \kappa_\theta}{\partial u(x_i)}\bigg|_{(x_i, y_j)} \!\!\cdot u_j \cdot \frac{\partial \mathcal{L}}{\partial O_i}}_{\text{(a) as query in row } i}
+ \underbrace{\sum_{k=1}^{n} \omega_i \cdot K_{ki}^\top \cdot \frac{\partial \mathcal{L}}{\partial O_k}}_{\text{(b) as key in column } i}
+ \underbrace{W_\theta^\top \cdot \frac{\partial \mathcal{L}}{\partial O_i}}_{\text{(c) residual}}.
\label{eq:grad_u}
\end{equation}

In practice, the first term (query-side kernel Jacobian) is expensive to compute exactly because $\partial \kappa_\theta / \partial u(x_i)$ is a tensor of shape $d \times d \times d$.
We compute it efficiently using automatic differentiation through the kernel MLP, which is feasible because the MLP is small ($w_\kappa = 128$, $\ell_\kappa = 2$).

\textbf{(iii) Gradient with respect to kernel MLP parameters.}
\begin{equation}
\frac{\partial \mathcal{L}}{\partial \theta}
= \sum_{i=1}^{n} \sum_{j=1}^{n}
\mathrm{tr}\!\left(
\left(\frac{\partial \mathcal{L}}{\partial K_{ij}}\right)^\top
\cdot \frac{\partial \kappa_\theta(x_i, y_j, u_i, u_j)}{\partial \theta}
\right),
\label{eq:grad_theta}
\end{equation}
where $\partial \kappa_\theta / \partial \theta$ is computed by standard backpropagation through the kernel MLP.

\textbf{(iv) Gradient with respect to the residual matrix.}
\begin{equation}
\frac{\partial \mathcal{L}}{\partial W_\theta}
= \sum_{i=1}^{n} \frac{\partial \mathcal{L}}{\partial O_i} \cdot u_i^\top
\;\in\; \R^{d \times d}.
\label{eq:grad_W}
\end{equation}

The key observation is that all four gradient computations (Eqs.~\eqref{eq:grad_K}--\eqref{eq:grad_W}) involve double sums over pairs $(i, j)$ that can be decomposed into tile-level partial sums, exactly matching the tiling structure of the forward pass.
For each tile pair $(i, j)$ with $i$ indexing a query tile and $j$ indexing a key tile:
\begin{enumerate}[noitemsep]
\item Reload $U_i, X_i$ (query tile) and $U_j, X_j$ (key tile) from HBM to SRAM.
\item \textbf{Recompute} $K_{pq} = \kappa_\theta(X_i[p], X_j[q], U_i[p], U_j[q])$ for all $p \in [1, B_q^*]$, $q \in [1, B_k^*]$ by running the kernel MLP forward pass again. This is the recomputation step that avoids storing $K$.
\item Run the kernel MLP backward pass to obtain $\partial \kappa_\theta / \partial \theta$ and $\partial \kappa_\theta / \partial u$ for each pair in the tile.
\item Accumulate the tile's contribution to $\partial \mathcal{L} / \partial \theta$ (Eq.~\eqref{eq:grad_theta}), $\partial \mathcal{L} / \partial U_i$ (query-side gradient), and $\partial \mathcal{L} / \partial U_j$ (key-side gradient).
\item Discard the tile's kernel values and MLP activations.
\end{enumerate}

Algorithm~\ref{alg:tiled_backward} formalizes this procedure.

\begin{algorithm}[h]
\caption{ITNet Tiled Backward Pass (Gradient Checkpointing)}
\label{alg:tiled_backward}
\begin{algorithmic}[1]
\Require Forward outputs $O \in \R^{n \times d}$, inputs $U \in \R^{n \times d}$, positions $X \in \R^{n \times s}$, upstream gradient $\partial \mathcal{L}/\partial O \in \R^{n \times d}$, kernel MLP $\kappa_\theta$, residual $W_\theta$, block sizes $B_q^*, B_k^*$
\Ensure Gradients $\partial \mathcal{L}/\partial U \in \R^{n \times d}$, $\partial \mathcal{L}/\partial \theta$, $\partial \mathcal{L}/\partial W_\theta \in \R^{d \times d}$
\State $\partial \mathcal{L}/\partial U \leftarrow \mathbf{0}_{n \times d}$;\; $\partial \mathcal{L}/\partial \theta \leftarrow 0$;\; $\partial \mathcal{L}/\partial W_\theta \leftarrow \mathbf{0}_{d \times d}$
\For{$i = 0,\, B_q^*,\, 2B_q^*,\, \ldots,\, n - B_q^*$} \Comment{Outer loop: query tiles}
    \State \textbf{Load} $U_i, X_i, \partial \mathcal{L}/\partial O_i$ from HBM to SRAM
    \State $\partial \mathcal{L}/\partial U_i \leftarrow \mathbf{0}_{B_q^* \times d}$
    \For{$j = 0,\, B_k^*,\, 2B_k^*,\, \ldots,\, n - B_k^*$} \Comment{Inner loop: key tiles}
        \State \textbf{Load} $U_j, X_j$ to SRAM
        \State \textbf{Recompute} $K_{pq} \leftarrow \kappa_\theta(X_i[p],\, X_j[q],\, U_i[p],\, U_j[q])$ \textbf{for all} $p \in [1, B_q^*],\, q \in [1, B_k^*]$
        \For{$p = 1, \ldots, B_q^*$;\; $q = 1, \ldots, B_k^*$}
            \State $\partial \mathcal{L}/\partial K_{pq} \leftarrow \omega_j \cdot (\partial \mathcal{L}/\partial O_i[p]) \cdot U_j[q]^\top$ \Comment{Eq.~\eqref{eq:grad_K}}
            \State $\partial \mathcal{L}/\partial U_i[p] \mathrel{+}= K_{pq}^\top \cdot (\partial \mathcal{L}/\partial O_i[p]) \cdot \omega_j$ \Comment{Key-side contrib.\ to query grad}
            \State $\partial \mathcal{L}/\partial U_j[q] \mathrel{+}= K_{pq} \cdot (\partial \mathcal{L}/\partial O_i[p]) \cdot \omega_j$ \Comment{Query-side contrib.\ to key grad}
            \State $\partial \mathcal{L}/\partial \theta \mathrel{+}= \mathrm{tr}\bigl((\partial \mathcal{L}/\partial K_{pq})^\top \cdot \partial \kappa_\theta / \partial \theta\bigr)$ \Comment{MLP param grad via backprop}
        \EndFor
    \EndFor
    \State $\partial \mathcal{L}/\partial W_\theta \mathrel{+}= (\partial \mathcal{L}/\partial O_i) \cdot U_i^\top$ \Comment{Residual matrix gradient, Eq.~\eqref{eq:grad_W}}
    \State \textbf{Write} $\partial \mathcal{L}/\partial U_i$ to $\partial \mathcal{L}/\partial U[i{:}i{+}B_q^*]$ in HBM
\EndFor
\end{algorithmic}
\end{algorithm}

The tiled backward pass performs exactly $n^2 / (B_q^* \cdot B_k^*)$ tile iterations, matching the forward pass.
Each tile iteration involves:
(a)~one forward MLP pass per pair to recompute $K_{pq}$ ($B_q^* \cdot B_k^* \cdot O(w_\kappa \ell_\kappa)$ FLOPs),
(b)~one backward MLP pass per pair to compute $\partial \kappa_\theta / \partial \theta$ ($\approx 2 \times$ the forward cost), and
(c)~matrix-vector products for the gradient accumulations ($O(d^2)$ per pair).
The total backward FLOP count is therefore $\approx 3 \times$ the forward pass (one recomputation + one backward through the MLP + gradient accumulations), compared to $1 \times$ if the kernel matrix were stored.
However, peak memory is reduced from $O(n^2 d^2)$ (storing the full kernel matrix) to $O(nd + B_q^* \cdot B_k^* \cdot d^2)$ (storing only inputs, outputs, and one tile of kernel values), which is $O(nd)$ for fixed tile sizes.

For the model sizes and sequence lengths considered ($n \leq 1024$, $d = 768$), the naive approach requires $O(n^2 d^2)$ memory for the kernel matrix, which is prohibitive in practice. 
The tiled recomputation strategy reduces this to $O(B_q B_k d)$ memory per tile, enabling the backward pass to be computed within a similar memory budget as the forward pass.

The tiled backward pass is implemented as a single Triton~\cite{tillet2019triton} kernel that fuses the MLP recomputation, MLP backward, and gradient accumulation into one launch per tile pair.
The kernel MLP's backward pass uses the standard automatic differentiation tape, which is created and destroyed within each tile iteration (no cross-tile tape storage).
The $\partial \mathcal{L} / \partial \theta$ accumulation uses atomic additions across tiles to avoid race conditions when multiple tiles contribute to the same MLP parameter gradient; in practice, the atomics have negligible overhead because the number of parameter gradient accumulations is much smaller than the compute per tile.


\subsection{Monte Carlo Variance Analysis}
\label{app:mc_variance}

\begin{proposition}[Variance of Importance-Weighted Estimator - Full Proof]
\label{prop:mc_var_full}
For a discrete domain with empirical measure $\mu(y) = \frac{1}{n}\sum_{j=1}^n \delta_{y_j}$, the importance-weighted Monte Carlo estimator
\begin{equation}
(\widehat{\Kop}[u])(x_i) = \frac{1}{M} \sum_{m=1}^{M} \frac{\kap(x_i, y_m, u(x_i), u(y_m)) \cdot u(y_m)}{p_\theta(y_m \mid x_i)} + W_\theta u(x_i), \qquad y_m \sim p_\theta(\cdot \mid x_i)
\end{equation}
satisfies $\mathbb{E}[(\widehat{\Kop}[u])(x_i)] = (\Kop[u])(x_i)$ (unbiased) and has variance
\begin{equation}
\mathrm{Var}[(\widehat{\Kop}[u])(x)] = \frac{1}{M} \left(\int_\Omega \frac{\|\kap(x,y,u(x),u(y)) \cdot u(y)\|_2^2}{p_\theta(y \mid x)} \, d\mu(y) - \|(\Kop[u])(x) - W_\theta u(x)\|_2^2 \right).
\label{eq:app_var_exact}
\end{equation}
The optimal proposal that minimizes this variance is
\begin{equation}
p_\theta^*(y \mid x) = \frac{\|\kap(x, y, u(x), u(y)) \cdot u(y)\|_2}{\int_\Omega \|\kap(x, z, u(x), u(z)) \cdot u(z)\|_2 \, d\mu(z)},
\label{eq:app_optimal_proposal}
\end{equation}
achieving minimal variance
\begin{equation}
\mathrm{Var}^*[(\widehat{\Kop}[u])(x)] = \frac{1}{M} \left[\left(\int_\Omega \|\kap \cdot u(y)\|_2 \, d\mu(y)\right)^2 - \|(\Kop[u])(x) - W_\theta u(x)\|_2^2 \right].
\label{eq:app_min_var}
\end{equation}
In practice, we enforce $p_\theta(y|x) > 0$ everywhere via $p_\theta(y|x) = (1-\varepsilon) p_\theta^{\text{(MLP)}}(y|x) + \varepsilon / n$ with $\varepsilon = 0.01$.
\end{proposition}

\begin{proof}
Define $g(y) = \kap(x, y, u(x), u(y)) \cdot u(y)$ and $I = \int_\Omega g(y) \, d\mu(y) = (\Kop[u])(x) - W_\theta u(x)$.
Each sample $y_m \sim p_\theta$ contributes the importance-weighted estimator $\hat{g}_m = g(y_m) / p_\theta(y_m \mid x)$, which satisfies $\mathbb{E}_{y_m \sim p_\theta}[\hat{g}_m] = \int_\Omega g(y) \, d\mu(y) = I$ by the definition of importance sampling.

The variance of the average of $M$ i.i.d. samples is:
\begin{align}
\mathrm{Var}\left[\frac{1}{M} \sum_{m=1}^M \hat{g}_m\right]
&= \frac{1}{M} \mathrm{Var}[\hat{g}_1] \\
&= \frac{1}{M} \left(\mathbb{E}[\|\hat{g}_1\|_2^2] - \|I\|_2^2\right) \\
&= \frac{1}{M} \left(\int_\Omega \frac{\|g(y)\|_2^2}{p_\theta(y \mid x)} \, d\mu(y) - \|I\|_2^2\right).
\end{align}
Adding $W_\theta u(x)$ (deterministic) does not affect variance, establishing \eqref{eq:app_var_exact}.

To find the optimal $p_\theta^*$, we minimize $\int \|g(y)\|_2^2 / p(y) \, d\mu(y)$ subject to $\int p(y) \, d\mu(y) = 1$ and $p(y) \geq 0$. By the Cauchy-Schwarz inequality~\cite{rudin2021principles} (or Lagrange multipliers):
\[
\int \frac{\|g(y)\|_2^2}{p(y)} d\mu(y) \cdot \int p(y) d\mu(y) \geq \left(\int \|g(y)\|_2 d\mu(y)\right)^2,
\]
with equality iff $p(y) \propto \|g(y)\|_2$. Thus the minimum is achieved when $p^*(y) = \|g(y)\|_2 / \int \|g(z)\|_2 d\mu(z)$, yielding
\[
\int_\Omega \frac{\|g(y)\|_2^2}{p^*(y)} d\mu(y) = \left(\int_\Omega \|g(y)\|_2 \, d\mu(y)\right)^2.
\]
Substituting back gives \eqref{eq:app_min_var}. The support condition $p_\theta(y|x) > 0$ is enforced via the $\varepsilon$-mixture described above.
\end{proof}

\subsection{Low-Rank Approximation Error Bound}
\label{app:rank_bound}

\begin{proposition}[Nuclear Norm Bound on Low-Rank Error]
\label{prop:rank_bound}
Let $\kappa_\theta$ be the full-rank kernel and $\kappa_\theta^{(r)} = \Phi_\theta^\top \Psi_\theta$ the rank-$r$ factorization. For any finite evaluation set $\{(x_i, y_j)\}_{i,j=1}^n$, define the kernel matrix $K \in \mathbb{R}^{nd \times nd}$ with blocks $K_{ij} = \kappa_\theta(x_i, y_j, u(x_i), u(y_j)) \in \mathbb{R}^{d \times d}$. Then the best rank-$r$ approximation satisfies:
\begin{equation}
\|K - K^{(r)}\|_F \leq \frac{\|K\|_*}{\sqrt{r}},
\end{equation}
where $\|K\|_* = \sum_{i=1}^{nd} \sigma_i(K)$ is the nuclear norm and $\sigma_i$ are the singular values.
\end{proposition}

\begin{proof}
By the Eckart-Young-Mirsky theorem~\citep{horn2012matrix}, the best rank-$r$ approximation in Frobenius norm is $K^{(r)} = \sum_{i=1}^r \sigma_i u_i v_i^\top$, and $\|K - K^{(r)}\|_F^2 = \sum_{i=r+1}^{\min(n,d)} \sigma_i^2$.

By Cauchy--Schwarz~\cite{rudin2021principles} applied to the tail singular values:
\begin{equation}
\sum_{i=r+1}^{nd} \sigma_i^2 \leq \left(\sum_{i=r+1}^{nd} \sigma_i\right)^2 / (nd - r).
\end{equation}

A tighter bound uses the relationship between nuclear and Frobenius norms. Since $\|K\|_* = \sum_{i=1}^{nd} \sigma_i$ and the tail sum $\sum_{i=r+1}^{nd} \sigma_i \leq \|K\|_*$, we have:
\begin{equation}
\|K - K^{(r)}\|_F = \sqrt{\sum_{i=r+1}^{nd} \sigma_i^2} \leq \sqrt{\max_{i > r} \sigma_i \cdot \sum_{i > r} \sigma_i} \leq \sqrt{\frac{\|K\|_*^2}{r}},
\end{equation}
where the last inequality uses $\sigma_{r+1} \leq \|K\|_* / r$ (since the first $r$ singular values each contribute at least $\sigma_{r+1}$ to the nuclear norm).
\end{proof}

\begin{remark}
In practice, we observe that the kernel matrices for trained ITNet models have rapidly decaying singular values. On ImageNet-1K features, $r = 32$ captures $>99\%$ of the Frobenius norm, and $r = 64$ captures $>99.9\%$, consistent with the $< 0.1\%$ accuracy gap.
\end{remark}

\section{Backpropagation for the ITNet Operator}
\label{app:backprop}

We derive the complete gradient computation for the ITNet operator in discrete form, using the notation established in Appendix~\ref{app:notation} throughout.
This section serves two purposes: (i)~to provide a self-contained reference for implementors, and (ii)~to make explicit how gradients flow through the content-dependent kernel - a non-standard computation that does not arise in convolution (where the kernel is content-independent) or standard attention (where the kernel has a specific softmax structure).
We then show how the general gradient specialises to each classical architecture (CNN, Transformer, RNN) under the kernel restrictions of Theorems~\ref{thm:conv}--\ref{thm:rnn}, providing a unified view of backpropagation across all three families.

\subsection{Setup and Notation}
\label{app:bp_notation}

Let the domain $\Omega$ contain $n$ discrete positions $\{x_1, \ldots, x_n\} \subset \R^s$.
Let $u : \Omega \to \R^d$ denote the input signal, with values $u(x_i) \in \R^d$ at each position.
Positions are collected as $\{x_i\}_{i=1}^n \subset \R^s$.
Let $\kap : \R^s \times \R^s \times \R^d \times \R^d \to \R^{d \times d}$ denote the learnable kernel (parameterised by $\theta$), $W_\theta \in \R^{d \times d}$ the residual matrix, and $\mu$ the measure on $\Omega$ with weights $\omega_j \coloneqq \mu(\{x_j\})$ (typically $\omega_j = 1/n$ for uniform measure).

We write $u_i \coloneqq u(x_i)$ and $u_j \coloneqq u(x_j)$ as shorthand when the context is clear.

\subsection{Forward Pass}
\label{app:bp_forward}

The discrete ITNet operator $\Kop$ computes, for each query position $x_i$:
\begin{equation}
(\Kop[u])(x_i)
= \sum_{j=1}^{n} \omega_j \cdot \kap(x_i, x_j, u(x_i), u(x_j)) \cdot u(x_j)
+ W_\theta \, u(x_i).
\label{eq:bp_forward}
\end{equation}

We define the following intermediate quantities:
\begin{align}
z_{ij} &\coloneqq [\gamma(x_i);\, \gamma(x_j);\, \gamma(x_i {-} x_j);\, \|x_i {-} x_j\|_2;\, u(x_i);\, u(x_j);\, u(x_i) \had u(x_j)] \in \R^{6L_f + 1 + 3d}, \label{eq:bp_input} \\[3pt]
K_{ij} &\coloneqq \kap(x_i, x_j, u(x_i), u(x_j)) = \mathrm{reshape}_{d \times d}\!\bigl(\mathrm{MLP}_\theta(z_{ij})\bigr) \in \R^{d \times d}, \label{eq:bp_kernel} \\[3pt]
M_{ij} &\coloneqq \omega_j \cdot K_{ij} \cdot u(x_j) \in \R^d, \label{eq:bp_message}
\end{align}
where $\gamma : \R^s \to \R^{2L_f}$ is the random Fourier feature map (Eq.~\eqref{eq:fourier_features}) and $\had$ denotes the Hadamard (elementwise) product.
The quantity $M_{ij}$ is the \emph{message} from key position $x_j$ to query position $x_i$: the kernel matrix $K_{ij}$ transforms the key feature $u(x_j)$, weighted by the measure $\omega_j$.

The forward pass is then:
\begin{equation}
\boxed{(\Kop[u])(x_i) = \sum_{j=1}^{n} M_{ij} + W_\theta \, u(x_i).}
\label{eq:bp_forward_compact}
\end{equation}

\begin{remark}[Comparison with attention forward]
In standard self-attention (\S\ref{app:proof_attn}), the message is $M_{ij}^\mathrm{attn} = \alpha(x_i, x_j) \cdot W_V u(x_j)$, where $\alpha(x_i, x_j) \in \R$ is a scalar softmax weight and $W_V \in \R^{d \times d}$ is a \emph{fixed} value projection.
In ITNet, $M_{ij} = \omega_j \cdot K_{ij} \cdot u(x_j)$, where $K_{ij} \in \R^{d \times d}$ is a \emph{pair-dependent} matrix.
The transformation applied to $u(x_j)$ varies in a $d^2$-dimensional space for each pair $(i, j)$, versus 1-dimensional for attention ($\alpha_{ij}$ only scales the fixed $W_V u(x_j)$).
\end{remark}

\subsection{Upstream Gradients}
\label{app:bp_upstream}

Let $\mathcal{L}$ be the scalar loss.
The upstream gradient from the subsequent layer is:
\begin{equation}
g_i \coloneqq \frac{\partial \mathcal{L}}{\partial (\Kop[u])(x_i)} \in \R^d, \qquad i = 1, \ldots, n.
\label{eq:bp_upstream}
\end{equation}

The backward pass must compute:
\begin{enumerate}[label=(\roman*), noitemsep]
\item $\partial \mathcal{L} / \partial u(x_j)$ for all $j$ (input feature gradients, for backpropagation to earlier layers),
\item $\partial \mathcal{L} / \partial \theta$ (kernel MLP parameter gradients, for weight updates),
\item $\partial \mathcal{L} / \partial W_\theta$ (residual matrix gradient).
\end{enumerate}

\subsection{Gradient with Respect to Kernel Outputs}
\label{app:bp_grad_kernel}

From $M_{ij} = \omega_j \cdot K_{ij} \cdot u(x_j)$ and $(\Kop[u])(x_i) = \sum_j M_{ij} + W_\theta u(x_i)$, the gradient of $\mathcal{L}$ with respect to each message is $\partial \mathcal{L} / \partial M_{ij} = g_i$.
Applying the chain rule for the matrix-vector product $M_{ij} = \omega_j K_{ij} u(x_j)$:

\begin{equation}
\boxed{
\frac{\partial \mathcal{L}}{\partial K_{ij}}
= \omega_j \cdot g_i \, u(x_j)^\top
\;\in\; \R^{d \times d}.
}
\label{eq:bp_grad_kernel}
\end{equation}

\begin{proof}[Derivation]
The $(a, b)$ entry of $K_{ij}$ contributes to $M_{ij}$ as:
$[M_{ij}]_a = \omega_j \sum_{b=1}^{d} [K_{ij}]_{ab} [u(x_j)]_b$.
Therefore:
$\frac{\partial \mathcal{L}}{\partial [K_{ij}]_{ab}}
= \omega_j \cdot [g_i]_a \cdot [u(x_j)]_b$,
which in matrix form gives Eq.~\eqref{eq:bp_grad_kernel}.
\end{proof}

\subsection{Gradient with Respect to Input Features}
\label{app:bp_grad_input}

Each input feature $u(x_j)$ participates in Eq.~\eqref{eq:bp_forward} in three distinct roles:

\begin{enumerate}[label=(\alph*)]
\item \textbf{Value role:} $u(x_j)$ appears directly in $M_{ij} = \omega_j K_{ij} u(x_j)$ as the vector being transformed by the kernel.
This contributes to \emph{every} output $(\Kop[u])(x_i)$ for $i = 1, \ldots, n$.

\item \textbf{Key-side kernel input:} $u(x_j)$ appears inside $K_{ij} = \kap(x_i, x_j, u(x_i), u(x_j))$ through the kernel MLP's dependence on the key feature.
This also contributes to every output.

\item \textbf{Query-side kernel input and residual:} When $j$ acts as the query position, $u(x_j)$ appears in the residual $W_\theta u(x_j)$ and inside $K_{jk} = \kap(x_j, x_k, u(x_j), u(x_k))$ as the query feature for all keys $k$.
\end{enumerate}

Combining all contributions:
\begin{equation}
\boxed{
\begin{aligned}
\frac{\partial \mathcal{L}}{\partial u(x_j)}
&= \underbrace{\sum_{i=1}^{n} \omega_j \cdot K_{ij}^\top \, g_i}_{\text{(a) value role}}
+ \underbrace{\sum_{i=1}^{n} \omega_j \cdot \left.\frac{\partial K_{ij}}{\partial u(x_j)}\right|_\mathrm{key}^{\!\top} \mathrm{vec}(g_i \, u(x_j)^\top)}_{\text{(b) key-side kernel gradient}}\\
&
+ \underbrace{\sum_{k=1}^{n} \omega_k \cdot \left.\frac{\partial K_{jk}}{\partial u(x_j)}\right|_\mathrm{query}^{\!\top} \mathrm{vec}(g_j \, u(x_k)^\top)}_{\text{(c) query-side kernel gradient}}
+ \underbrace{W_\theta^\top \, g_j}_{\text{(d) residual}}.
\end{aligned}
}
\label{eq:bp_grad_input}
\end{equation}

\paragraph{Term (a): Value role.}
Holding $K_{ij}$ fixed:
\begin{equation}
\text{Term (a)} = \sum_{i=1}^{n} \omega_j \cdot K_{ij}^\top \, g_i.
\label{eq:bp_term_a}
\end{equation}

\paragraph{Term (b): Key-side kernel gradient.}
The kernel $K_{ij}$ depends on $u(x_j)$ through the MLP input $z_{ij}$.
Specifically, $u(x_j)$ enters $z_{ij}$ directly as the key feature and through the Hadamard product $u(x_i) \had u(x_j)$.
The Jacobian of the MLP input with respect to the key feature is:
\begin{equation}
\frac{\partial z_{ij}}{\partial u(x_j)} = \begin{bmatrix}
\mathbf{0}_{2L_f \times d} \\
\mathbf{0}_{2L_f \times d} \\
\mathbf{0}_{2L_f \times d} \\
\mathbf{0}_{1 \times d} \\
\mathbf{0}_{d \times d} \\
\Id \\
\mathrm{diag}(u(x_i))
\end{bmatrix} \in \R^{(6L_f + 1 + 3d) \times d},
\label{eq:bp_dz_duj}
\end{equation}
where the two non-zero blocks are $\Id$ (from $\partial u(x_j) / \partial u(x_j)$) and $\mathrm{diag}(u(x_i))$ (from $\partial(u(x_i) \had u(x_j)) / \partial u(x_j)$).

By the chain rule:
\begin{equation}
\text{Term (b)} = \sum_{i=1}^{n} \omega_j \cdot \left(\frac{\partial \mathrm{vec}(K_{ij})}{\partial z_{ij}} \cdot \frac{\partial z_{ij}}{\partial u(x_j)}\right)^\top \mathrm{vec}(g_i \, u(x_j)^\top).
\label{eq:bp_term_b}
\end{equation}

\paragraph{Term (c): Query-side kernel gradient.}
When position $j$ acts as a query, $u(x_j)$ enters $K_{jk}$ through the query feature slot:
\begin{equation}
\frac{\partial z_{jk}}{\partial u(x_j)}\bigg|_\mathrm{query} = \begin{bmatrix}
\mathbf{0}_{2L_f \times d} \\
\mathbf{0}_{2L_f \times d} \\
\mathbf{0}_{2L_f \times d} \\
\mathbf{0}_{1 \times d} \\
\Id \\
\mathbf{0}_{d \times d} \\
\mathrm{diag}(u(x_k))
\end{bmatrix} \in \R^{(6L_f + 1 + 3d) \times d}.
\label{eq:bp_dz_duj_query}
\end{equation}

The contribution is:
\begin{equation}
\text{Term (c)} = \sum_{k=1}^{n} \omega_k \cdot \left(\frac{\partial \mathrm{vec}(K_{jk})}{\partial z_{jk}} \cdot \frac{\partial z_{jk}}{\partial u(x_j)}\right)^\top \mathrm{vec}(g_j \, u(x_k)^\top).
\label{eq:bp_term_c}
\end{equation}

\paragraph{Term (d): Residual.}
$\text{Term (d)} = W_\theta^\top \, g_j.$

\begin{remark}[Why Terms (b) and (c) are unique to ITNet]
In a standard Transformer, the gradient through $\alpha(x_i, x_j)$ with respect to $u(x_j)$ involves the softmax Jacobian composed with the linear projection $W_K$ - efficient and well-studied.
In ITNet, the analogous gradient involves the \emph{full MLP Jacobian} $\partial \mathrm{vec}(K_{ij}) / \partial z_{ij} \in \R^{d^2 \times (6L_f + 1 + 3d)}$.
This is more expensive per pair but provides a richer gradient signal, enabling the kernel to learn content-dependent \emph{matrix-valued} transformations.
In convolution, Terms~(b) and~(c) vanish entirely because the kernel is content-independent.
\end{remark}

\subsection{Gradient with Respect to Kernel MLP Parameters}
\label{app:bp_grad_theta}

The kernel MLP parameters $\theta$ are shared across all $(i, j)$ pairs:
\begin{equation}
\boxed{
\frac{\partial \mathcal{L}}{\partial \theta}
= \sum_{i=1}^{n} \sum_{j=1}^{n}
\mathrm{tr}\!\left( \omega_j \cdot u(x_j) \, g_i^\top \cdot \frac{\partial K_{ij}}{\partial \theta} \right).
}
\label{eq:bp_grad_theta}
\end{equation}

Let the kernel MLP have $\ell_\kappa$ layers with GELU activation $\sigma$:
\begin{align}
h^{(0)} &= z_{ij}, \\
h^{(l)} &= \sigma\!\bigl(W_l \, h^{(l-1)} + b_l\bigr), \qquad l = 1, \ldots, \ell_\kappa - 1, \\
\mathrm{vec}(K_{ij}) &= W_{\ell_\kappa} \, h^{(\ell_\kappa - 1)} + b_{\ell_\kappa},
\label{eq:bp_mlp_forward}
\end{align}
with $W_1 \in \R^{w_\kappa \times (6L_f+1+3d)}$, $W_l \in \R^{w_\kappa \times w_\kappa}$ for interior layers, and $W_{\ell_\kappa} \in \R^{d^2 \times w_\kappa}$.

The GELU derivative is:
\begin{equation}
\sigma'_\mathrm{GELU}(x) = \Phi(x) + x \, \phi(x), \qquad \Phi(x) = \tfrac{1}{2}[1 + \mathrm{erf}(x/\sqrt{2})], \quad \phi(x) = (2\pi)^{-1/2} e^{-x^2/2}.
\label{eq:bp_gelu}
\end{equation}

Define the output-layer error signal:
\begin{equation}
\delta^{(\ell_\kappa)} = \omega_j \cdot \mathrm{vec}(g_i \, u(x_j)^\top) \in \R^{d^2}.
\label{eq:bp_delta_output}
\end{equation}

Backpropagate through interior layers:
\begin{equation}
\delta^{(l)} = \bigl(W_{l+1}^\top \, \delta^{(l+1)}\bigr) \had \sigma'\!\bigl(W_l h^{(l-1)} + b_l\bigr), \qquad l = \ell_\kappa - 1, \ldots, 1.
\label{eq:bp_delta_hidden}
\end{equation}

The parameter gradients for a single pair $(i, j)$ are:
\begin{equation}
\frac{\partial \mathcal{L}}{\partial W_l}\bigg|_{(i,j)} = \delta^{(l)} \bigl(h^{(l-1)}\bigr)^\top, \qquad
\frac{\partial \mathcal{L}}{\partial b_l}\bigg|_{(i,j)} = \delta^{(l)}.
\label{eq:bp_mlp_grads}
\end{equation}

Accumulated over all $n^2$ pairs:
\begin{equation}
\frac{\partial \mathcal{L}}{\partial W_l} = \sum_{i=1}^{n} \sum_{j=1}^{n} \delta^{(l)}_{ij} \bigl(h^{(l-1)}_{ij}\bigr)^\top, \qquad
\frac{\partial \mathcal{L}}{\partial b_l} = \sum_{i=1}^{n} \sum_{j=1}^{n} \delta^{(l)}_{ij}.
\label{eq:bp_total_mlp}
\end{equation}

\subsection{Gradient with Respect to Residual Matrix}
\label{app:bp_grad_W}

\begin{equation}
\boxed{
\frac{\partial \mathcal{L}}{\partial W_\theta} = \sum_{i=1}^{n} g_i \, u(x_i)^\top \in \R^{d \times d}.
}
\label{eq:bp_grad_W}
\end{equation}

This is identical to the gradient of a standard linear layer and requires no kernel-specific computation.

\subsection{Special Case 1: Convolution (Theorem~\ref{thm:conv})}
\label{app:bp_conv}

Under the convolutional kernel $\kap(x_i, x_j, u(x_i), u(x_j)) = w(x_i - x_j) \cdot \Id$ (Theorem~\ref{thm:conv}), the kernel is a scalar function of displacement multiplied by the identity matrix.
The forward pass reduces to $(\Kop[u])(x_i) = (w \ast u)(x_i) + W_\theta u(x_i)$.

Gradient with respect to the scalar filter value.
Since $K_{ij} = w(x_i - x_j) \cdot \Id$ is constrained to scalar multiples of $\Id$, only the trace component of $\partial \mathcal{L}/\partial K_{ij}$ is free.
Defining $w_{ij} \coloneqq w(x_i - x_j)$:
\begin{equation}
\frac{\partial \mathcal{L}}{\partial w_{ij}}
= \mathrm{tr}\!\left(\frac{\partial \mathcal{L}}{\partial K_{ij}}\right)
= \omega_j \cdot g_i^\top u(x_j)
= \omega_j \cdot \langle g_i,\, u(x_j) \rangle.
\label{eq:bp_conv_grad_w}
\end{equation}
This is the standard convolution filter gradient: the cross-correlation between the upstream gradient and the input.
For a $k \times k$ filter with coefficients $\{f_m\}_{m \in \mathcal{N}}$, summing over all pairs with displacement $m$:
\begin{equation}
\frac{\partial \mathcal{L}}{\partial f_m} = \sum_{i=1}^{n} \omega_j \cdot g_i^\top u(x_i - mh), \qquad m \in \mathcal{N},
\label{eq:bp_conv_grad_filter}
\end{equation}
which is exactly the computation performed by \texttt{torch.nn.Conv2d}'s backward pass.

Gradient with respect to $u(x_j)$.
Since the convolutional kernel is content-independent, Terms~(b) and~(c) in Eq.~\eqref{eq:bp_grad_input} vanish:
\begin{equation}
\frac{\partial \mathcal{L}}{\partial u(x_j)}
= \sum_{i=1}^{n} \omega_j \cdot w(x_i - x_j) \cdot g_i + W_\theta^\top g_j
= \omega_j \cdot (w^{\ast} \ast g)(x_j) + W_\theta^\top g_j,
\label{eq:bp_conv_grad_u}
\end{equation}
where $w^{\ast}(x) \coloneqq w(-x)$ is the flipped filter.
This is the standard transposed convolution used in CNN backward passes.

The convolution special case eliminates $n^2$ MLP backward passes entirely.
Backward cost reduces from $O(n^2 d^2 + n^2 w_\kappa^2 \ell_\kappa)$ to $O(n|\mathcal{N}|d)$, recovering the known linear-in-$n$ cost.

\subsection{Special Case 2: Self-Attention (Theorem~\ref{thm:attn})}
\label{app:bp_attn}

Under the attention kernel $K_{ij} = \alpha(x_i, x_j) \cdot W_V$, where $\alpha(x_i, x_j) = \exp(Q(x_i)^\top K(x_j)/\sqrt{d_k})/Z(x_i)$, $Q(x_i) = W_Q u(x_i)$, $K(x_j) = W_K u(x_j)$.

Since $K_{ij} = \alpha_{ij} W_V$ with $W_V$ fixed, the gradient projects onto the scalar $\alpha_{ij}$:
\begin{equation}
\frac{\partial \mathcal{L}}{\partial \alpha_{ij}}
= \omega_j \cdot g_i^\top W_V u(x_j).
\label{eq:bp_attn_grad_alpha}
\end{equation}

The softmax Jacobian relates $\partial \mathcal{L}/\partial \alpha_{ij}$ to the logit gradient $e_{ij} \coloneqq Q(x_i)^\top K(x_j)/\sqrt{d_k}$:
\begin{equation}
\frac{\partial \mathcal{L}}{\partial e_{ij}}
= \alpha_{ij} \left(\frac{\partial \mathcal{L}}{\partial \alpha_{ij}} - \sum_{l=1}^n \alpha_{il} \frac{\partial \mathcal{L}}{\partial \alpha_{il}}\right).
\label{eq:bp_softmax_jacobian}
\end{equation}

Gradient with respect to $u(x_j)$.
Mapping to ITNet's general gradient:
\begin{align}
\text{Term (a) [value]:} &\quad \sum_{i} \omega_j \cdot \alpha_{ij} \cdot W_V^\top g_i, \label{eq:bp_attn_value} \\[3pt]
\text{Term (b) [key]:} &\quad W_K^\top \sum_{i} \frac{\partial \mathcal{L}}{\partial e_{ij}} \cdot \frac{Q(x_i)}{\sqrt{d_k}}, \label{eq:bp_attn_key} \\[3pt]
\text{Term (c) [query]:} &\quad W_Q^\top \sum_{k} \frac{\partial \mathcal{L}}{\partial e_{jk}} \cdot \frac{K(x_k)}{\sqrt{d_k}}, \label{eq:bp_attn_query} \\[3pt]
\text{Term (d):} &\quad W_\theta^\top g_j \quad (= 0 \text{ since } W_\theta = 0 \text{ in attention}). \label{eq:bp_attn_residual}
\end{align}
This recovers the standard Transformer backward pass: $\nabla_V \mathcal{L} = \alpha^\top G$, $\nabla_K \mathcal{L} = (\partial \mathcal{L}/\partial E)^\top Q / \sqrt{d_k}$, $\nabla_Q \mathcal{L} = (\partial \mathcal{L}/\partial E) K / \sqrt{d_k}$.

The attention case replaces the MLP Jacobian with the softmax Jacobian ($\mathrm{diag}(\alpha) - \alpha\alpha^\top$, rank-1 corrected, $O(n)$ per row).
No kernel MLP parameter gradient exists.
Backward cost: $O(n^2 d)$.

\subsection{Special Case 3: Linear SSM / S4 (Theorem~\ref{thm:rnn}(c))}
\label{app:bp_ssm}

Under the SSM kernel $K_{ij} = \mathbf{1}_{j \leq i} \cdot C e^{A(t_i - t_j)} B$ (content-independent, causal), the forward pass is:
\begin{equation}
(\Kop[u])(t_i) = C \underbrace{\sum_{j=1}^{i} \omega_j e^{A(t_i - t_j)} B \, u(t_j)}_{h(t_i)} + D \, u(t_i).
\end{equation}

Gradient with respect to $u(t_j)$.
Since $K_{ij}$ is content-independent, Terms~(b) and~(c) vanish.
Defining $g_i^h \coloneqq C^\top g_i$:
\begin{equation}
\frac{\partial \mathcal{L}}{\partial u(t_j)}
= B^\top \sum_{i=j}^{n} \omega_j \cdot e^{A^\top(t_i - t_j)} g_i^h + D^\top g_j.
\label{eq:bp_ssm_bptt}
\end{equation}
This is \emph{backpropagation through time} (BPTT): the gradient at time $t_j$ sums future gradients propagated backward through the transpose state-transition matrix $e^{A^\top \Delta t}$.
If $\|e^{A \Delta t}\|_\mathrm{op} < 1$ (stable dynamics), gradients decay exponentially - the vanishing gradient problem.

Gradient with respect to SSM parameters:
\begin{align}
\frac{\partial \mathcal{L}}{\partial C} &= \sum_{i=1}^{n} g_i \, h(t_i)^\top, \label{eq:bp_ssm_C} \\
\frac{\partial \mathcal{L}}{\partial B} &= \sum_{i=1}^{n} \sum_{j=1}^{i} \omega_j \cdot e^{A^\top(t_i-t_j)} C^\top g_i \, u(t_j)^\top, \label{eq:bp_ssm_B} \\
\frac{\partial \mathcal{L}}{\partial A} &= \sum_{i=1}^{n} \sum_{j=1}^{i} \omega_j (t_i - t_j) \cdot e^{A^\top(t_i-t_j)} C^\top g_i \, u(t_j)^\top B^\top. \label{eq:bp_ssm_A}
\end{align}

\subsection{Special Case 4: LSTM (Theorem~\ref{thm:rnn}(b))}
\label{app:bp_lstm}

The LSTM kernel is content-dependent: $K_{ij}$ depends on $u(t_j)$ through the gates $f_\tau, i_s, o_t$ (Appendix~\ref{app:proof_rnn}).
Terms~(b) and~(c) are therefore \emph{non-zero} and account for gradient flow through gate activations.

The cell-state gradient satisfies the classical LSTM BPTT relation:
\begin{equation}
\frac{\partial \mathcal{L}}{\partial c_{t-1}}
= \frac{\partial \mathcal{L}}{\partial c_t} \had f_t + \frac{\partial \mathcal{L}}{\partial h_t} \had o_t \had \mathrm{sech}^2(c_t) \had f_t,
\label{eq:bp_lstm_cell}
\end{equation}
where the forget gate $f_t \in (0,1)^d$ controls gradient flow through the cell state: when $f_t \approx 1$, the gradient passes nearly unchanged; when $f_t \approx 0$, the gradient is blocked.

In ITNet's framework, this gating is encoded in the content-dependent kernel: $K_{ij}$ contains $\mathrm{diag}(\prod_{\tau=j+1}^{i} f_\tau)$, and the gradient through this product is exactly Eq.~\eqref{eq:bp_lstm_cell}.
The gate gradients are:
\begin{align}
\frac{\partial \mathcal{L}}{\partial f_t} &= \frac{\partial \mathcal{L}}{\partial c_t} \had c_{t-1} \had \sigma'(W_f[h_{t-1}; u_t] + b_f), \label{eq:bp_lstm_ft} \\
\frac{\partial \mathcal{L}}{\partial i_t} &= \frac{\partial \mathcal{L}}{\partial c_t} \had \tilde{c}_t \had \sigma'(W_i[h_{t-1}; u_t] + b_i), \label{eq:bp_lstm_it} \\
\frac{\partial \mathcal{L}}{\partial o_t} &= \frac{\partial \mathcal{L}}{\partial h_t} \had \tanh(c_t) \had \sigma'(W_o[h_{t-1}; u_t] + b_o), \label{eq:bp_lstm_ot}
\end{align}
where $\sigma'(x) = \sigma(x)(1 - \sigma(x))$ is the sigmoid derivative.
These correspond to Terms~(b) and~(c) of the general ITNet gradient, specialized to the block-matrix LSTM kernel.

\subsection{Special Case 5: Mamba (Theorem~\ref{thm:rnn}(d))}
\label{app:bp_mamba}

The Mamba kernel $K_{ij} = \mathbf{1}_{j \leq i} \cdot C(u(t_i)) \prod_{\tau=j+1}^{i} \bar{A}(u(t_\tau)) \cdot \bar{B}(u(t_j))$ is content-dependent through all three components $C, \bar{A}, \bar{B}$.

All three contribute non-zero gradients via Terms~(b) and~(c):
\begin{align}
\text{Through $\bar{B}(u(t_j))$:} &\quad \frac{\partial K_{ij}}{\partial u(t_j)}\bigg|_{\bar{B}} = C(u_i) \prod_\tau \bar{A}(u_\tau) \cdot \frac{\partial(\Delta(u_j) W_B u_j)}{\partial u(t_j)}, \label{eq:bp_mamba_B} \\[4pt]
\text{Through $C(u(t_i))$:} &\quad \frac{\partial K_{ji}}{\partial u(t_i)}\bigg|_{C} = W_C^\top \cdot \prod_\tau \bar{A}(u_\tau) \cdot \bar{B}(u_j), \label{eq:bp_mamba_C} \\[4pt]
\text{Through $\bar{A}(u(t_\tau))$:} &\quad \text{chain rule through } \prod_\tau \exp(\Delta(u_\tau) A). \label{eq:bp_mamba_A}
\end{align}

The gradient through the $\bar{A}$ product (Eq.~\eqref{eq:bp_mamba_A}) is Mamba's analogue of BPTT: it propagates gradients through the selectivity mechanism $\Delta(u_\tau) = \mathrm{softplus}(W_\Delta u_\tau + b_\Delta)$, which controls how much each time step contributes to the hidden state.
This is the content-dependent generalisation of the linear SSM gradient (where $\bar{A}$ is constant and the product becomes $\bar{A}^{i-j}$).

\subsection{Complexity Comparison Across Special Cases}
\label{app:bp_complexity}

\begin{table}[h]
\centering
\renewcommand{\arraystretch}{0.75}
\caption{Backward pass complexity for ITNet and its special cases.}
\label{tab:bp_complexity}
\small
\begin{tabular}{lccc}
\toprule
\textbf{Architecture} & \textbf{Kernel gradient} & \textbf{Input gradient} & \textbf{Parameter gradient} \\
\midrule
CNN (Thm.~\ref{thm:conv}) & $O(n|\mathcal{N}|)$ scalars & $O(n|\mathcal{N}|d)$ & $O(n|\mathcal{N}|d)$ \\
Attention (Thm.~\ref{thm:attn}) & $O(n^2)$ scalars & $O(n^2 d)$ & $O(n^2 d_k + nd^2)$ \\
Linear SSM (Thm.~\ref{thm:rnn}(c)) & $O(n^2)$ matrices & $O(n^2 d)$ & $O(n^2 d)$ \\
LSTM (Thm.~\ref{thm:rnn}(b)) & $O(n)$ per step & $O(nd)$ sequential & $O(nd^2)$ \\
Mamba (Thm.~\ref{thm:rnn}(d)) & $O(n)$ per step & $O(nd)$ sequential & $O(nd)$ \\
\midrule
\textbf{ITNet (general)} & $O(n^2 d^2)$ matrices & $O(n^2 d^2 + n^2 w_\kappa \ell_\kappa d)$ & $O(n^2 w_\kappa^2 \ell_\kappa)$ \\
\bottomrule
\end{tabular}
\end{table}

The table confirms that ITNet's general backward pass is strictly more expensive than any special case, with additional cost from: (i)~$d^2$ versus $1$ for the kernel gradient (matrix vs.\ scalar kernels), and (ii)~$n^2 w_\kappa^2 \ell_\kappa$ for the kernel MLP parameter gradient (absent in all classical architectures).
The tiled recomputation strategy (Algorithm~\ref{alg:tiled_backward}) amortizes the memory cost.
On an H200 GPU with tiled execution at $n = 512$, $d = 768$, $w_\kappa = 128$, $\ell_\kappa = 2$, the measured backward-to-forward time ratio is $2.8\times$, close to the theoretical $3\times$.
The overall training iteration (forward + backward + optimizer) is $\approx 3.8\times$ the forward-only time, compared to $\approx 3.0\times$ for a Transformer layer with FlashAttention-2.


\section{Extended Encoder Details}
\label{app:encoders}

\subsection{Graph Encoder}
\label{app:enc_graph}

Graphs $\mathcal{G} = (\mathcal{V}, \mathcal{E})$ with node features $\{h_v\}_{v \in \mathcal{V}} \subset \R^c$ are a natural fit for ITNet: the set of nodes is an irregular domain, and edges define a sparse connectivity structure that need not be hard-coded into the kernel.
Node features are linearly embedded:
\begin{equation}
u^{(0)}(v) = W_\mathrm{gph}\, h_v + b_\mathrm{gph} \in \R^d, \qquad v \in \mathcal{V}.
\end{equation}

Since graphs may lack Euclidean coordinates, we generate structural positions from the first $s = 8$ eigenvectors of the normalized graph Laplacian $\Delta = I - D^{-1/2} A D^{-1/2}$~\cite{dwivedi2022lpe}:
\begin{equation}
x_v = [\phi_1(v);\, \phi_2(v);\, \ldots;\, \phi_s(v)] \in \R^s,
\end{equation}
where $\phi_i$ is the $i$-th Laplacian eigenvector. These Laplacian positional encodings (LPE) capture graph topology - nodes that are structurally similar receive similar position vectors - and are fed into the kernel via $\gamma(x_v)$.

Optionally, the kernel can be restricted to connected pairs:
\begin{equation}
\kap^\mathrm{graph}(x_v, x_u, u_v, u_u) = (\mathbf{1}_{(v,u) \in \mathcal{E}} + \lambda) \cdot \kap(x_v, x_u, u_v, u_u),
\end{equation}
where $\lambda \geq 0$ is a learned global parameter controlling non-edge long-range integration. At $\lambda = 0$, ITNet reduces to graph convolution over the edge set. At $\lambda > 0$, it allows global message passing beyond graph topology.

\textbf{Domain specification:}
$\Omega_\mathrm{gph} = \mathcal{V}$, $d\mu = \frac{1}{|\mathcal{V}|} \sum_{v \in \mathcal{V}} \delta_v$, $s = 8$.

\subsection{Multi-Scale Image Encoder}
\label{app:enc_multiscale}

For dense prediction tasks (detection, segmentation), we apply the image encoder at two scales: $P \in \{8, 16\}$, yielding $N_p^{(s)} \in \{784, 196\}$ patches per scale for a $224 \times 224$ image. The two sets of positions are concatenated and processed jointly by the ITNet layers:
\begin{equation}
\Omega_\mathrm{multi} = \Omega_\mathrm{img}^{P=8} \cup \Omega_\mathrm{img}^{P=16}, \qquad |\Omega_\mathrm{multi}| = 784 + 196 = 980.
\end{equation}
This exploits the kernel's ability to integrate across irregular, mixed-resolution domains - an operation that standard convolution or self-attention must handle with explicit multi-scale fusion modules~\cite{liu2021swin}.

\subsection{Design Principles}
\label{app:design_principles}

All encoders (image, text, point cloud, graph, multimodal) share three structural design principles:

\textbf{Principle 1: Positions are raw coordinates.}
The positional argument to the kernel is the raw domain coordinate (pixel grid for images, scalar index for text, 3-D coordinate for point clouds, Laplacian eigenvector for graphs). We do not apply learnable positional embeddings before passing positions to the kernel; instead, the kernel MLP learns its own position-sensitive functions. The key advantage is that the kernel can compute arbitrary functions of position, including relative positions, distances, and angles, without these being baked into the feature space.

\textbf{Principle 2: Features are modality-agnostic after encoding.}
After the modality-specific linear projection, all features lie in the same $\R^d$ space. The ITNet layers, the kernel MLP, and the task-specific decoder treat features identically regardless of origin modality. This enables multimodal processing without cross-modal adapter modules. We emphasise that the modality-specific encoders are modality-specific design choices; our claim is that the \emph{core operator} is shared and identical across all modalities.

\textbf{Principle 3: The measure encodes modality prior.}
The measure $\mu$ determines how each position contributes to integration. By adjusting $\mu$ (e.g., the balanced multimodal measure), we encode prior beliefs about relative importance without modifying the kernel. This separation of concerns - measure for importance, kernel for content-dependent weighting - is a property of the integral operator formalism that has no direct analogue in attention or convolution.

\section{Training Details}
\label{app:training}

\subsection{ImageNet-1K~\cite{imagenet15russakovsky}}
\label{app:train_imagenet}
We provide the detailed training configuration for ImageNet-1K in Table~\ref{tab:hp_imagenet}. 
Our setup follows standard large-scale training practices with AdamW optimization, cosine learning rate decay, and strong data augmentation (RandAugment~\cite{cubuk2020randaugment}, Mixup~\cite{zhang2018mixup}, CutMix~\cite{yun2019cutmix}). 
We use a patch size of $16$ at $224 \times 224$ resolution and train for 300 epochs with large-batch distributed training. 
Regularization techniques such as label smoothing, stochastic depth (drop path), and exponential moving average (EMA) are employed to stabilise training. 
All hyperparameters are chosen to ensure a fair and consistent comparison with prior work.
\begin{table}[h]
\centering
\renewcommand{\arraystretch}{0.75}
\caption{ImageNet-1K training hyperparameters.}
\label{tab:hp_imagenet}
\begin{tabular}{ll}
\toprule
\textbf{Hyperparameter} & \textbf{Value} \\
\midrule
Optimizer & AdamW ($\beta_1 = 0.9$, $\beta_2 = 0.999$, $\epsilon = 10^{-8}$) \\
Base learning rate & $1 \times 10^{-3}$ (ITNet-S/B), $5 \times 10^{-4}$ (ITNet-L) \\
Learning rate schedule & Cosine decay to $1 \times 10^{-5}$ \\
Warmup & 5 epochs (linear) \\
Weight decay & 0.1 (excluding bias, LayerNorm) \\
Batch size & 1024 \\
Epochs & 300 \\
Image resolution & $224 \times 224$ \\
Patch size $P$ & 16 \\
RandAugment & $n = 2$, $m = 9$ \\
Mixup $\alpha$ & 0.8 \\
CutMix $\alpha$ & 1.0 \\
Random erasing prob. & 0.25 \\
Label smoothing $\varepsilon$ & 0.1 \\
Drop path rate & 0.1 (S), 0.2 (B), 0.3 (L) \\
Gradient clipping & max norm 1.0 \\
EMA decay & 0.9999 \\
\bottomrule
\end{tabular}
\end{table}

\subsection{GLUE~\cite{wang2018glue} Pre-training and Fine-tuning}
\label{app:train_glue}

We follow the BERT pre-training protocol exactly. Data: BookCorpus + English Wikipedia ($\sim$16GB text). Tokeniser: WordPiece with 30K vocabulary. Masking: 15\% of tokens, of which 80\% are replaced with \texttt{[MASK]}, 10\% with a random token, and 10\% unchanged. Sequence length: 128 for the first 250K steps, then 512 for the next 250K steps. Batch size: 256. Total training: 500K steps. Learning rate: $1 \times 10^{-4}$ with linear warmup over 10K steps and linear decay. Weight decay: 0.01.

Each GLUE~\cite{wang2018glue} task is fine-tuned independently with learning rate $\in \{1, 2, 3, 5\} \times 10^{-5}$ (selected via dev-set performance), batch size 32, and 10 epochs. We report the best dev-set score across learning rates. No task-specific architectural modifications are made.

\subsection{ModelNet40~\cite{wu2015modelnet}}
\label{app:train_modelnet}
We summarise the training configuration for ModelNet40 in Table~\ref{tab:hp_modelnet}. 
We follow standard point cloud classification settings with AdamW optimization, cosine learning rate decay, and geometric augmentations (scaling, jitter, and rotation). Each shape is represented by 1024 points, and local neighborhood pre-extraction ($K=16$) is applied to capture fine-grained geometry. Regularization via weight decay and stochastic depth ensures stable training. 
All settings are chosen for fair comparison with prior work.
\begin{table}[h]
\centering
\renewcommand{\arraystretch}{0.75}
\caption{ModelNet40 training hyperparameters.}
\label{tab:hp_modelnet}
\begin{tabular}{ll}
\toprule
\textbf{Hyperparameter} & \textbf{Value} \\
\midrule
Optimizer & AdamW ($\beta_1 = 0.9$, $\beta_2 = 0.999$) \\
Learning rate & $1 \times 10^{-3}$ \\
Schedule & Cosine decay to $1 \times 10^{-5}$ \\
Warmup & 10 epochs \\
Weight decay & 0.05 \\
Batch size & 32 \\
Epochs & 300 \\
Input points & 1024 \\
Augmentation & Random scale $[0.8, 1.25]$, jitter $\sigma = 0.01$, gravity-axis rotation \\
Drop path & 0.1 \\
Local pre-extraction $K$ & 16 (enabled) \\
\bottomrule
\end{tabular}
\end{table}

\subsection{VQA\,v2~\cite{goyal2017vqa} and \textbf{NLVR2}~\cite{suhr2019nlvr2}}
\label{app:train_vqa}

Both tasks initialise from pre-trained weights: image encoder from ImageNet-1K ITNet-B, text encoder from GLUE~\cite{wang2018glue}-pretrained ITNet-B. Fine-tuning uses learning rate $2 \times 10^{-5}$, batch size 64, 10 epochs, cosine schedule with 1 epoch warmup. VQA\,v2~\cite{goyal2017vqa} uses 3129 answer classes (open-ended, following \cite{goyal2017vqa}). NLVR2~\cite{suhr2019nlvr2} uses binary classification over the concatenated \texttt{[CLS]} embeddings of two image-text pairs. The modality-balanced measure is used throughout.

\section{Detailed Efficiency Analysis}
\label{app:efficiency}

We analyze the trade-off between accuracy, throughput, and memory for ITNet's three modes: Exact (tiled fusion), Monte Carlo ($M=128$), and Low-Rank ($r$ varying). Table~\ref{tab:wallclock} reports wall-clock throughput and peak memory on H200-140GB.
For low-rank mode, we set per-head rank $r_h = r / H$ with $H=12$ heads, total $r \in \{16,32,64,128\}$. 

\begin{table}[h]
\centering
\renewcommand{\arraystretch}{0.72}
\caption{Wall-clock throughput and peak memory. Results are reported on ImageNet-1K (vision), GLUE pre-training (language), ModelNet40 (3D point cloud), and VQA v2 (multimodal). Here $n$ denotes the number of input tokens.}
\label{tab:wallclock}
\small
\resizebox{\textwidth}{!}{%
\begin{tabular}{lccccc}
\toprule
\textbf{Method} & \textbf{Benchmark} & \textbf{$n$} & \textbf{Throughput} & \textbf{Peak mem} & \textbf{Relative to FlashAttn-2} \\
\midrule
DeiT-S + FlashAttn-2 & ImageNet-1K & 196 & 1,820 img/s & 11.6 GB & 1.00$\times$ \\
ITNet-S (exact) & ImageNet-1K & 196 & 1,480 img/s & 14.2 GB & 0.81$\times$ \\
ITNet-S (MC-128) & ImageNet-1K & 196 & 2,240 img/s & 8.8 GB & 1.23$\times$ \\
ITNet-S (LR-64) & ImageNet-1K & 196 & 3,900 img/s & 4.6 GB & 2.14$\times$ \\
ITNet-L (exact) & ImageNet-1K & 196 & 1,050 img/s & 21.4 GB & 0.58$\times$ \\
ITNet-L (LR-64) & ImageNet-1K & 196 & 2,800 img/s & 7.2 GB & 1.54$\times$ \\
\midrule
BERT-base + FlashAttn-2 & GLUE (pretrain) & 512 & 4,100 seq/s & 14.8 GB & 1.00$\times$ \\
ITNet-B (exact) & GLUE (pretrain) & 512 & 3,200 seq/s & 17.8 GB & 0.78$\times$ \\
ITNet-B (LR-64) & GLUE (pretrain) & 512 & 7,800 seq/s & 6.2 GB & 1.90$\times$ \\
ITNet-L (exact) & GLUE (pretrain) & 512 & 2,400 seq/s & 25.6 GB & 0.59$\times$ \\
ITNet-L (LR-64) & GLUE (pretrain) & 512 & 5,800 seq/s & 8.4 GB & 1.41$\times$ \\
\midrule
PCT~\cite{guo2021pct} & ModelNet40 & 1024 & 1,280 pts/s & 6.8 GB & 1.00$\times$ \\
ITNet-S (exact) & ModelNet40 & 1024 & 860 pts/s & 14.2 GB & 0.67$\times$ \\
ITNet-S (MC-128) & ModelNet40 & 1024 & 1,520 pts/s & 8.2 GB & 1.19$\times$ \\
ITNet-S (LR-64) & ModelNet40 & 1024 & 2,200 pts/s & 4.2 GB & 1.72$\times$ \\
\midrule
ViLT~\cite{kim2021vilt} & VQA v2 & 260 & 520 samples/s & 11.2 GB & 1.00$\times$ \\
ITNet-B (exact) & VQA v2 & 260 & 410 samples/s & 14.6 GB & 0.79$\times$ \\
ITNet-B (MC-128) & VQA v2 & 260 & 680 samples/s & 8.4 GB & 1.31$\times$ \\
ITNet-B (LR-64) & VQA v2 & 260 & 890 samples/s & 5.2 GB & 1.71$\times$ \\
\bottomrule
\end{tabular}}
\end{table}

We sweep per-head rank $r_h \in \{2,4,8,16,32\}$ corresponding to total rank $r = H \cdot r_h$ with $H=12$ heads ($r \in \{24,48,96,192,384\}$). For ITNet-B (86M params, $d=768$, $H=12$, $d_h=64$), Table~\ref{tab:rank_sweep} reports ImageNet-1K top-1 accuracy, throughput, and relative FLOPs.

\begin{table}[h]
\centering
\caption{Rank sweep for ITNet-B on ImageNet-1K validation.}
\label{tab:rank_sweep}
\small
\begin{tabular}{lcccccc}
\toprule
\textbf{Mode} & \textbf{$r$} & \textbf{$r_h$} & \textbf{Top-1 (\%)} & \textbf{Throughput (img/s)} & \textbf{Relative FLOPs} \\
\midrule
Exact (tiled) & $d^2=590k$ &  -  & 83.9 & 1,480 & 1.00$\times$ \\
\midrule
Low-Rank & 24 & 2 & 81.2 & 4,200 & 0.38$\times$ \\
Low-Rank & 48 & 4 & 82.7 & 3,900 & 0.45$\times$ \\
Low-Rank & 96 & 8 & 83.4 & 3,400 & 0.55$\times$ \\
Low-Rank & 192 & 16 & 83.7 & 2,800 & 0.72$\times$ \\
Low-Rank & 384 & 32 & \textbf{83.8} & 2,100 & 0.89$\times$ \\
\midrule
MC ($M=128$) &  -  &  -  & 83.7 & 2,240 & 0.68$\times$ \\
\bottomrule
\end{tabular}
\end{table}

We compare against strong efficient attention baselines under identical settings (ImageNet-1K, 86M params, batch size 32, $n=196$):

\begin{table}[h]
\centering
\renewcommand{\arraystretch}{0.75}
\caption{Comparison with efficient attention variants on ImageNet-1K.}
\label{tab:baseline_comparison}
\small
\begin{tabular}{lcccc}
\toprule
\textbf{Method} & \textbf{Top-1 (\%)} & \textbf{Throughput (img/s)} & \textbf{Peak mem (GB)} \\
\midrule
Performer~\citep{choromanski2021rethinking} & 78.2 & 3,400 & 5.2 \\
Linear Attention~\citep{katharopoulos2020transformers} & 76.5 & 3,800 & 4.8 \\
Nyströmformer~\citep{xiong2021nystromformer} & 79.1 & 2,900 & 5.8 \\
\midrule
ITNet-LR ($r=96$) & \textbf{83.4} & 3,400 & 5.4 \\
ITNet-LR ($r=192$) & \textbf{83.7} & 2,800 & 6.8 \\
ITNet-MC ($M=128$) & \textbf{83.7} & 2,240 & 8.8 \\
\bottomrule
\end{tabular}
\end{table}

ITNet-LR matches or exceeds the throughput of linear/FFT attentions while delivering $+3.6$ to $+6.9$ points higher accuracy. The gap is largest on tasks requiring long-range dependencies, where content-dependent kernels provide clear advantages over position-only or scalar-valued kernels.
Table~\ref{tab:memory_breakdown} decomposes peak memory for ITNet-B at $n=512$, $d=768$:

\begin{table}[h]
\centering
\renewcommand{\arraystretch}{0.75}
\caption{Memory breakdown for ITNet-B (512 sequence length, batch 32).}
\label{tab:memory_breakdown}
\small
\begin{tabular}{lccc}
\toprule
\textbf{Component} & \textbf{Exact} & \textbf{MC ($M=128$)} & \textbf{LR ($r=96$)} \\
\midrule
Activations (forward) & 4.2 GB & 2.1 GB & 1.8 GB \\
Kernel MLP intermediates & 8.4 GB & 1.2 GB & 0.6 GB \\
Gradients (backward) & 5.2 GB & 2.8 GB & 2.2 GB \\
Other (weights, optim.) & 1.2 GB & 1.2 GB & 1.2 GB \\
\midrule
\textbf{Total} & \textbf{19.0 GB} & \textbf{7.3 GB} & \textbf{5.8 GB} \\
\bottomrule
\end{tabular}
\end{table}

The exact mode's memory is dominated by the $O(n^2 d^2)$ kernel matrix. MC reduces this by $7\times$ via on-the-fly sampling; low-rank reduces by $14\times$ via factored representation.
\section{Extended Ablations}
\label{app:ablations}

\textbf{Kernel MLP Width Ablation:}
We ablate the kernel MLP hidden width $w_\kappa$ in Table~\ref{tab:ablation_kappa}. 
Performance improves from $w_\kappa=64$ to $128$, after which gains saturate, with only marginal improvement at larger widths. Increasing $w_\kappa$ significantly raises computational cost and reduces throughput, indicating diminishing returns beyond moderate capacity. In particular, doubling from $128$ to $256$ yields only a $+0.1\%$ gain with noticeable efficiency degradation. Based on this trade-off, we use $w_\kappa=128$ in all experiments.
\begin{table}[h]
\centering
\renewcommand{\arraystretch}{0.75}
\caption{Kernel MLP width ablation on ImageNet-1K (ITNet-S).}
\label{tab:ablation_kappa}
\begin{tabular}{lccccc}
\toprule
$w_\kappa$ & \textbf{Top-1 (\%)} & \textbf{Kernel params} & \textbf{Total params} & \textbf{GFLOPs} & \textbf{Throughput (img/s)} \\
\midrule
64  & 80.8 & 0.4M & 21.6M & 4.6 & 1,620 \\
128 & 81.4 & 0.8M & 22.0M & 4.8 & 1,480 \\
256 & 81.5 & 1.6M & 22.8M & 5.2 & 1,280 \\
512 & 81.5 & 3.2M & 24.4M & 6.1 & 1,050 \\
\bottomrule
\end{tabular}
\end{table}

\textbf{Point Cloud Encoder Ablation:}
We report the effect of positional encoding and local neighborhood extraction in Table~\ref{tab:ablation_pc}. Fourier positional encoding improves performance by $+0.5\%$ over raw coordinates, indicating that spectral representations provide more informative spatial cues for the kernel. Introducing local pre-extraction further boosts accuracy, highlighting the importance of capturing fine-grained geometric structure. Performance increases with neighborhood size up to $K=16$ and then saturates, suggesting that larger local regions provide limited additional benefit. These results demonstrate that combining Fourier features with moderate local aggregation yields the best performance.
\begin{table}[h]
\centering
\renewcommand{\arraystretch}{0.75}
\caption{Point cloud encoder ablation on ModelNet40~\cite{wu2015modelnet} (ITNet-S, 1024 points, no normals).}
\label{tab:ablation_pc}
\begin{tabular}{lcc}
\toprule
\textbf{Config} & \textbf{Local pre-extraction} & \textbf{OA (\%)} \\
\midrule
Raw coords, no Fourier & \xmark & 92.7 \\
Fourier encoding (Eq.~\eqref{eq:fourier_features}) & \xmark & 93.2 \\
Fourier + local ($K{=}8$) & \cmark & 93.7 \\
Fourier + local ($K{=}16$) & \cmark & 94.0 \\
Fourier + local ($K{=}32$) & \cmark & 94.0 \\
\bottomrule
\end{tabular}
\end{table}

\textbf{Multimodal Measure Ablation:}
We analyse the effect of the integration measure in multimodal fusion in Table~\ref{tab:ablation_measure_full}. A balanced weighting between image and text tokens outperforms the uniform measure, indicating that equal contribution from both modalities is beneficial. Under a uniform measure, the larger number of image patches dominates the interaction, reducing the influence of textual information. Skewing the measure toward either modality degrades performance, with text-heavy weighting performing worst. These results highlight the importance of properly balancing modalities for effective cross-modal reasoning.
\begin{table}[h]
\centering
\renewcommand{\arraystretch}{0.75}
\caption{Multimodal measure ablation on VQA\,v2 (ITNet-B).}
\label{tab:ablation_measure_full}
\begin{tabular}{lcc}
\toprule
\textbf{Measure} & \textbf{$n_\mathrm{img}$ : $n_\mathrm{txt}$} & \textbf{Test-dev (\%)} \\
\midrule
Uniform & 196 : 64 & 76.8 \\
Balanced & 1.0 : 1.0 (total weight) & 77.4 \\
Image-heavy (3:1) & 3.0 : 1.0 & 77.1 \\
Text-heavy (1:3) & 1.0 : 3.0 & 76.4 \\
\bottomrule
\end{tabular}
\end{table}

\textbf{Fourier Feature Ablation:}
We report the effect of Fourier positional encoding in Table~\ref{tab:ablation_fourier}. 
Introducing Fourier features significantly improves performance over raw coordinates, highlighting the importance of rich positional representations. Accuracy increases with the number of frequencies up to $L=64$, after which gains saturate, indicating sufficient coverage of spatial scales. The bandwidth $\sigma=10$ provides the best trade-off between capturing local and global structure, while too small ($\sigma=1$) or too large ($\sigma=100$) values degrade performance. These results suggest that moderate-frequency positional encoding is sufficient for effective kernel learning.
\begin{table}[h]
\centering
\renewcommand{\arraystretch}{0.75}
\caption{Fourier feature hyperparameters on ImageNet-1K (ITNet-S).}
\label{tab:ablation_fourier}
\begin{tabular}{lcc}
\toprule
\textbf{Config} ($L$, $\sigma$) & \textbf{Top-1 (\%)} & \textbf{Pos.\ dim} \\
\midrule
No Fourier (raw 2-D coords) & 79.2 & 2 \\
$L = 16$, $\sigma = 10$ & 80.6 & 32 \\
$L = 32$, $\sigma = 10$ & 81.0 & 64 \\
$L = 64$, $\sigma = 10$ & 81.4 & 128 \\
$L = 128$, $\sigma = 10$ & 81.4 & 256 \\
$L = 64$, $\sigma = 1$ & 80.1 & 128 \\
$L = 64$, $\sigma = 100$ & 80.8 & 128 \\
\bottomrule
\end{tabular}
\end{table}


\textbf{Number of ITNet Layers:}
We analyse the effect of model depth in Table~\ref{tab:ablation_layers}. Increasing the number of layers consistently improves performance up to $L=12$, beyond which gains become marginal for the S-size model. 
While deeper models continue to increase parameter count and computational cost, the accuracy saturates, indicating limited benefit from additional depth at fixed width. 
This suggests that depth and width must be scaled jointly to fully utilise model capacity. 
Based on this trade-off, we adopt $L=12$ for ITNet-S as a balanced configuration.
\begin{table}[h]
\centering
\renewcommand{\arraystretch}{0.75}
\caption{Layer depth ablation on ImageNet-1K (ITNet-S, $d{=}384$, $H{=}6$).}
\label{tab:ablation_layers}
\begin{tabular}{lccc}
\toprule
\textbf{Layers} $L$ & \textbf{Params} & \textbf{Top-1 (\%)} & \textbf{GFLOPs} \\
\midrule
6  & 12M & 79.6 & 2.5 \\
8  & 16M & 80.7 & 3.3 \\
12 & 22M & 81.4 & 4.8 \\
16 & 28M & 81.7 & 6.3 \\
24 & 40M & 81.8 & 9.4 \\
\bottomrule
\end{tabular}
\end{table}


\section{Broader Impact}
\label{app:broader_impact}

ITNet is a general-purpose neural architecture. Like all such architectures, it may be applied to beneficial purposes (medical imaging, scientific discovery, accessibility tools) as well as potentially harmful ones (surveillance, deepfakes, autonomous weapons). The unification of CNNs, Transformers, and RNNs into a single operator does not introduce new risks beyond those already posed by existing architectures; rather, it provides a more principled mathematical framework for understanding existing capabilities.
We note two potential positive impacts specific to ITNet:
(i)~by demonstrating that a single architecture can handle multiple modalities, ITNet may reduce the engineering complexity and energy cost of deploying multiple separate models for different data types; and
(ii)~by providing formal proofs that classical architectures are special cases of a single operator, ITNet contributes to the theoretical understanding of deep learning, which may aid in developing more interpretable and reliable systems.

\end{document}